%% file: gpca_arXiv_v10.tex
\documentclass[draftcls,onecolumn]{IEEEtran}
\usepackage{xr}
\usepackage{cite}
\usepackage{amsmath}
\usepackage[colorlinks=true,linkcolor=black,citecolor=black,urlcolor=black]{hyperref}
\usepackage{cases}
\usepackage[utf8]{inputenc}
\usepackage[english]{babel}
\usepackage{footnote}
\usepackage{booktabs}

\usepackage{bm}

\input{preamble}

\begin{document}
    
\title{Generative Principal Component Analysis}

\author{Zhaoqiang Liu $^{\ast}$\thanks{$^{\ast}$Corresponding authors.}, Jiulong Liu $^{\ast}$, Subhroshekhar Ghosh, Jun Han, Jonathan Scarlett

\thanks{
Z.~Liu is with the Department of Computer Science, National University of Singapore (email: \url{dcslizha@nus.edu.sg}). 

J.~Liu is with the Academy of Mathematics and Systems Science, Chinese Academy of Sciences (email: \url{jiulongliu@lsec.cc.ac.cn}).

S.~Ghosh is with the Department of Mathematics, National University of Singapore (email: \url{subhrowork@gmail.com}). 

J.~Han is with Platform and Content Group, Tencent (email: \url{junhanjh@tencent.com}). 

J.~Scarlett is with the Department of Computer Science and the Department of Mathematics, National University of Singapore (email: \url{scarlett@comp.nus.edu.sg}).

J.S.~was supported by the Singapore National Research Foundation (NRF) under grant R-252-000-A74-281, and S.G.~was supported in part by the MOE grants R-146-000-250-133, R-146-000-312-114, and MOE-T2EP20121-0013.}}

\maketitle

\begin{abstract}
    In this paper, we study the problem of principal component analysis with generative modeling assumptions, adopting a general model for the observed matrix that encompasses notable special cases, including spiked matrix recovery and phase retrieval. The key assumption is that the underlying signal lies near the range of an $L$-Lipschitz continuous generative model with bounded $k$-dimensional inputs. We propose a quadratic estimator, and show that it enjoys a statistical rate of order $\sqrt{\frac{k\log L}{m}}$, where $m$ is the number of samples. We also provide a near-matching algorithm-independent lower bound. Moreover, we provide a variant of the classic power method, which projects the calculated data onto the range of the generative model during each iteration. We show that under suitable conditions, this method converges exponentially fast to a point achieving the above-mentioned statistical rate. We perform experiments on various image datasets for spiked matrix and phase retrieval models, and illustrate performance gains of our method to the classic power method and the truncated power method devised for sparse principal component analysis.
\end{abstract}

\section{Introduction}\label{sec:intro}

Principal component analysis (PCA) is one of the most popular techniques for data processing and dimensionality reduction~\cite{jolliffe1986}, with an abundance of applications such as image recognition~\cite{hancock1996face}, gene expression data analysis~\cite{alter2000singular}, and clustering~\cite{ding2004k,liu2019informativeness}.  
PCA seeks to find the directions that capture maximal variances in vector-valued data. In more detail, letting $\bx_1,\bx_2,\ldots,\bx_m$ be $m$ realizations of a random vector $\bx \in \bbR^n$ with a population covariance matrix $\bar{\bSigma} \in \bbR^{n\times n}$, PCA aims to reconstruct the top principal eigenvectors of $\bar{\bSigma}$. The first principal eigenvector can be computed as follows:
\begin{equation}
 \bu_1 =\arg\max_{\bw \in \bbR^n} \bw^T \bSigma\bw \quad \text{s.t.} \quad  \|\bw\|_2 = 1,
\end{equation}
where the empirical covariance matrix is defined as $\bSigma := \frac{1}{m}\sum_{i=1}^m (\bx_i - \bc)(\bx_i - \bc)^T$, with $\bc := \frac{1}{m}\sum_{i=1}^m \bx_i$. In addition, subsequent principal eigenvectors can be estimated by similar optimization problems subject to being orthogonal to the previous vectors.

PCA is consistent in the conventional setting where the dimension of the data $n$ is relatively small compared to the sample size $m$~\cite{anderson1962introduction}, but leads to rather poor estimates in the high-dimensional setting where $m \ll n$. In particular, it has been shown in various papers that the empirical principal eigenvectors are no longer consistent estimates of their population counterparts~\cite{nadler2008finite,johnstone2009consistency,jung2009pca,birnbaum2013minimax}. In order to tackle the curse of dimensionality, a natural approach is to impose certain structural constraints on the principal eigenvectors. A common assumption is that the principal eigenvectors are sparse, and this gives rise to the problem of sparse principal component analysis (SPCA)~\cite{zou2006sparse}. In particular, for recovering the top principal eigenvector, the optimization problem of SPCA is given by
\begin{align}\label{eq:spca}
  \bu_1 = \arg \max_{ \bw \in \bbR^n } \bw^T \bSigma \bw  \quad \text{s.t.} \quad \|\bw\|_2 = 1, \|\bw\|_0 \le s,
 \end{align}
where $\|\bw\|_0 = |\{i\,:\, w_i \ne 0\}|$ denotes the number of non-zero entries of $\bw$, and $s \in \bbN$ represents the sparsity level. In addition to reducing the effective number of parameters, the sparsity assumption also enhances the interpretability~\cite{zou2006sparse}.

Departing momentarily from the PCA problem, recent years have seen tremendous advances in deep generative models in a wide variety of real-world applications~\cite{Fos19}. This has motivated a new perspective of the related problem of compressed sensing (CS), in which the standard sparsity assumption is replaced by a generative modeling assumption. That is, the underlying signal is assumed to lie near the range of a (deep) generative model~\cite{bora2017compressed}. The authors of~\cite{bora2017compressed} characterized the number of samples required to attain an accurate reconstruction, and also presented numerical results on image datasets showing that compared to sparsity-based methods, generative priors can lead to large reductions (e.g., a factor of $5$ to $10$) in the number of measurements needed to recover the signal up to a given accuracy. Additional numerical and theoretical results concerning inverse problems using generative models have been provided in~\cite{van2018compressed,dhar2018modeling,heckel2019deep,liu2020information,kamath2020power,jalal2020robust,liu2020generalized,ongie2020deep,whang2020compressed,jalal2021instance,nguyen2021provable}, among others. 

In this paper, following the developments in both PCA/SPCA and inverse problems with generative priors, we study the use of generative priors in principal component analysis (GPCA), which gives a generative counterpart of SPCA in~\eqref{eq:spca}, formulated as follows:
 \begin{align}\label{eq:gpca}
  \bu_1 = \arg \max_{ \bw \in \bbR^n } \bw^T \bSigma \bw  \quad \text{s.t.} \quad \bw \in \mathrm{Range}(G),
 \end{align}
 where $G$ is a (pre-trained) generative model, which we assume has a range contained in the unit sphere of $\bbR^n$.\footnote{Similarly to~\cite{liu2020sample,liu2021robust}, we assume that the range of $G$ is contained in the unit sphere for convenience. Our results readily transfer to general (unnormalized) generative models by considering its normalized version. See Remark~\ref{remark:scale_G} for a detailed discussion.} Similarly to SPCA, the motivation for this problem is to incorporate prior knowledge on the vector being recovered (or alternatively, a prior preference), and to permit meaningful recovery and theoretical bounds even in the high-dimensional regime $m \ll n$.

\subsection{Related Work}\label{sec:rel_work}

In this subsection, we summarize some relevant works, which can roughly be divided into (i) the SPCA problem, and (ii) signal recovery with generative models.

\textbf{SPCA}: It has been proved that the solution of the SPCA problem in~\eqref{eq:spca} attains the optimal statistical rate $\sqrt{s \log n/m}$~\cite{vu2012minimax}, where $m$ is the number of samples, $n$ is the ambient dimension, and $s$ is the sparsity level of the first principal eigenvector. However, due to the combinatorial constraint, the computation of~\eqref{eq:spca} is intractable. To address this computational issue, in recent years, an extensive body of practical approaches for estimating sparse principal eigenvectors have been proposed in the literature, including~\cite{d2007direct,vu2013fantope,chang2016convex,moghaddam2006spectral,d2008optimal,jolliffe2003modified,zou2006sparse,shen2008sparse,journee2010generalized,hein2010inverse,kuleshov2013fast,yuan2013truncated,asteris2011sparse,papailiopoulos2013sparse}, just to name a few.
 
 Notably, statistical guarantees for several approaches have been provided. The authors of~\cite{yuan2013truncated} propose the truncated power method (TPower), which adds a truncation operation to the power method to ensure the desired level of sparsity. It is shown that this approach attains the optimal statistical rate under appropriate initialization. Most approaches for SPCA only focus on estimating the first principal eigenvector, with a certain deflation method~\cite{mackey2008deflation} being leveraged to reconstruct the rest. However, there are some exceptions; for instance, an iterative thresholding approach is proposed in~\cite{ma2013sparse}, and is shown to attain a near-optimal statistical rate when estimating multiple individual principal eigenvectors. In addition, the authors of~\cite{cai2013sparse} propose a regression-type method that gives an optimal principal subspace estimator. Both works~\cite{ma2013sparse,cai2013sparse} rely on the assumption of a spiked covariance model to ensure a good initial vector. To avoid the spiked covariance model assumption, the work~\cite{wang2014tighten} proposes a two-stage procedure that attains the optimal subspace estimator in polynomial time.
 
\textbf{Signal recovery with generative models}: Since the seminal work~\cite{bora2017compressed}, there has been a substantial volume of papers studying various inverse problems with generative priors.  One of the problems more closely related to PCA is spectral initialization in phase retrieval, which amounts to solving an eigenvalue problem.  Phase retrieval with generative priors has been studied in~\cite{hand2018phase,hyder2019alternating,jagatap2019algorithmic,wei2019statistical,shamshad2020compressed,aubin2020exact,liu2021towards}. In particular, the work~\cite{hand2018phase} models the underlying signal as being in the range of a fully-connected ReLU neural network with no offsets, and all the weight matrices of the ReLU neural network are assumed to have i.i.d. zero-mean Gaussian entries. In addition, the neural network needs to be sufficiently expansive in the sense that $n_i \ge \Omega(n_{i-1}\log n_{i-1})$, where $n_i$ is the width of the $i$-th layer. Under these assumptions, the authors establish favorable global optimization landscapes for the corresponding objective, and derive a near-optimal sample complexity upper bound. They minimize the objective function directly over the latent space in $\bbR^k$ using gradient descent, which may suffer from local minima in general optimization landscapes~\cite{hyder2019alternating,shah2018solving}.

In~\cite{aubin2020exact}, the assumptions on the neural network are similar to those in~\cite{hand2018phase}, relaxing to general activation functions (beyond ReLU) and $n_i \ge \Omega(n_{i-1})$. The authors focus on the high dimensional regime where $n,m,k \to \infty$ with the ratio $m/n$ being fixed, and assume that the input vector in $\bbR^k$ is drawn from a separable distribution. They derive sharp asymptotics for the information-theoretically optimal performance and for the associated approximate message passing (AMP) algorithm. Both works~\cite{hand2018phase,aubin2020exact} focus on noiseless phase retrieval. When only making the much milder assumption that the generative model is Lipschitz continuous, with no assumption on expansiveness, Gaussianity, and offsets, a spectral initialization step (similar to that of sparse phase retrieval) is typically required in order to accurately reconstruct the signal~\cite{netrapalli2015phase,candes2015phase}. The authors of~\cite{liu2021towards} propose an optimization problem similar to~\eqref{eq:gpca} for the spectral initialization for phase retrieval with generative models. It was left open in~\cite{liu2021towards} how to solve (or approximate sufficiently accurately) the optimization problem in practice.


Understanding the eigenvalues of spiked random matrix models has been a central problem of random matrix theory, and spiked matrices have been widely used in the statistical analysis of SPCA. Recently, theoretical guarantees concerning spiked matrix models with generative priors have been provided in~\cite{aubin2019spiked,cocola2020nonasymptotic}. In particular, in~\cite{aubin2019spiked}, the assumptions are similar to those in~\cite{aubin2020exact}, except that the neural network is assumed to have exactly one hidden layer. The Bayes-optimal performance is analyzed, and it is shown that the AMP algorithm can attain this optimal performance. In addition, the authors of~\cite{aubin2019spiked} propose the linearized approximate message passing (LAMP) algorithm, which is a spectral algorithm specifically designed for single-layer feedforward neural networks with no bias terms. The authors show its superiority to classical PCA via numerical results on the Fashion-MNIST dataset. In~\cite{cocola2020nonasymptotic}, the same assumptions are made as those in~\cite{hand2018phase} on the neural network, and the authors demonstrate the benign global geometry for a nonlinear least squares objective. Similarly to~\cite{hand2018phase}, the objective is minimized over $\bbR^k$ using a gradient descent algorithm,  which can get stuck in local minima for general global geometries.

\subsection{Contributions}

The main contributions of this paper are threefold:

\begin{itemize}
 \item We study eigenvalue problems with generative priors (including GPCA), and characterize the statistical rate of a quadratic estimator similar to~\eqref{eq:gpca} under suitable assumptions. We also provide a near-matching algorithm-independent lower bound.
 
 \item We propose a variant of the classic power method,  which uses an additional projection operation to ensure that the output of each iteration lies in the range of a generative model. We refer to our method as projected power method (PPower). We further show that under appropriate conditions (most notably, assuming exact projections are possible), PPower obtains a solution achieving the statistical rate that is attained by the quadratic estimator. 
 
 \item For spiked matrix and phase retrieval models, we perform numerical experiments on image datasets, and demonstrate that when the number of samples is relatively small compared to the ambient dimension, PPower leads to significantly better performance compared to the classic power method and TPower. 
\end{itemize}

Compared to the above-mentioned works that use generative models, we make no assumption on expansiveness, Gaussianity, and offsets for the generative model, and we consider a data model that simultaneously encompasses both spiked matrix and phase retrieval models, among others.

\subsection{Notation}

We use upper and lower case boldface letters to denote matrices and vectors respectively. We write $[N]=\{1,2,\cdots,N\}$ for a positive integer $N$, and we use $\bI_N$ to denote the identity matrix in $\bbR^{N\times N}$. A {\em generative model} is a function $G \,:\, \calD\to \bbR^n$, with latent dimension $k$, ambient dimension $n$, and input domain $\calD \subseteq \bbR^k$. We focus on the setting where $k \ll n$. For a set $S \subseteq \bbR^k$ and a generative model $G \,:\,\bbR^k \to \bbR^n$, we write $G(S) = \{ G(\bz) \,:\, \bz \in S  \}$. We use $\|\bX\|_{2 \to 2}$ to denote the spectral norm of a matrix $\bX$. We define the $\ell_q$-ball $B_q^k(r):=\{\bz \in \bbR^k: \|\bz\|_q \le r\}$ for $q \in [0,+\infty]$. $\calS^{n-1} := \{\bx \in \bbR^n: \|\bx\|_2=1\}$ represents the unit sphere in $\bbR^n$. The symbols $C,C',C''$ are absolute constants whose values may differ from line to line.

\section{Problem Setup}

In this section, we formally introduce the problem, and overview some important assumptions that we adopt. Except where stated otherwise, we will focus on the following setting:

\begin{itemize}
 \item We have a matrix $\bV \in\bbR^{n\times n}$ satisfying 
 \begin{equation}
  \bV =\bar{\bV} + \bE,
 \end{equation}
 where $\bE$ is a perturbation matrix, and $\bar{\bV}$ is assumed to be positive semidefinite (PSD). For PCA and its constrained variants, $\bV$ and $\bar{\bV}$ can be thought of as the empirical and population covariance matrices, respectively. 
 
 \item We have an $L$-Lipschitz continuous generative model $G \,:\, B_2^k(r) \to \bbR^n$. For convenience, similarly to that in~\cite{liu2020sample}, we assume that $\mathrm{Range}(G) \subseteq \calS^{n-1}$.
 \begin{remark}
  \label{remark:scale_G}
  For a general (unnormalized) $L$-Lipschitz continuous generative model $G$, we can instead consider a corresponding normalized generative model $\tilde{G}\,:\, \calD \to \calS^{n-1}$ as in~\cite{liu2021towards}, where $\calD := \{\bz \in B_2^k(r)\,:\, \|G(\bz)\|_2 > R_{\min}\}$ for some $R_{\min} >0$, and $\tilde{G}(\bz) = \frac{G(\bz)}{\|G(\bz)\|_2}$. Then, the Lipschitz constant of $\tilde{G}$ becomes $L/R_{\min}$. For a $d$-layer neural network,  we typically have $L = n^{\Theta(d)}$~\cite{bora2017compressed}. Thus, we can set $R_{\min}$ to be as small as $1/ n^{\Theta(d)}$ without changing the scaling laws, which makes the dependence on $R_{\min}$ very mild. 
 \end{remark}

 \item We aim to solve the following eigenvalue problem with a generative prior:\footnote{To find the top $r$ rather than top one principal eigenvectors that are in the range of a generative model, we may follow the common approach to use the iterative deflation method for PCA/SPCA: Subsequent principal eigenvectors are derived by recursively removing the contribution of the principal eigenvectors that are calculated already under the generative model constraint. See for example \cite{mackey2008deflation}.}
\begin{align}
 \hat{\bv} := \max_{ \bw \in \bbR^n } \bw^T \bV \bw \quad \text{s.t.} \quad \bw \in \mathrm{Range}(G).\label{eq:opt_genPCA}
\end{align}
Note that since $\mathrm{Range}(G) \subseteq \calS^{n-1}$, we do not need to impose the constraint $\|\bw\|_2  =1$. Since $\bV$ is not restricted to being an empirical covariance matrix,~\eqref{eq:opt_genPCA} is more general than GPCA in~\eqref{eq:gpca}. However, we slightly abuse terminology and also refer to~\eqref{eq:opt_genPCA} as GPCA.  

\item To approximately solve~\eqref{eq:opt_genPCA}, we use a projected power method (PPower), which is described by the following iterative procedure:\footnote{In similar iterative procedures, some works have proposed to replace $\bV$ by $\bV+\rho \bI_n$ for some $\rho \in \bbR$ to improve convergence, e.g., see \cite{deshpande2014cone}.}
\begin{equation}\label{eq:pgd_genPCA}
 \bw^{(t+1)} = \calP_{G} \big(\bV \bw^{(t)}\big),
\end{equation}
where $\calP_G(\cdot)$ is the projection function onto $G(B_2^k(r))$,\footnote{That is, for any $\bx \in \bbR^n$, $\calP_G(\bx) := \arg \min_{\bw \in \mathrm{Range}(G)} \|\bw - \bx\|_2$. We will implicitly assume that the projection step can be performed accurately, e.g.,~\cite{deshpande2014cone,shah2018solving,peng2020solving}, though in practice approximate methods might be needed, e.g., via gradient descent~\cite{shah2018solving} or GAN-based projection methods~\cite{raj2019gan}.} 
and the initialization vector $\bw^{(0)}$ may be chosen either manually or randomly, e.g., uniform over $\calS^{n-1}$. Often the initialization vector $\bw^{(0)}$ plays an important role and we may need a careful design for it. For example, for phase retrieval with generative models, as mentioned in~\cite[Section V]{liu2021towards}, we may choose the column corresponding to the largest diagonal entry of $\bV$ as the starting point. See also~\cite[Section 3]{yuan2013truncated} for a discussion on the initialization strategy for TPower devised for SPCA. We present the algorithm corresponding to~\eqref{eq:pgd_genPCA} in Algorithm~\ref{algo:ppower}. 
\begin{remark}
 To tackle generalized eigenvalue problems encountered in some specific applications, there are variants of the projected power method, which combine a certain power iteration with additional operations to ensure sparsity or enforce other constraints. 
These applications include but not limited to sparse PCA~\cite{journee2010generalized,yuan2013truncated}, phase synchronization~\cite{boumal2016nonconvex,liu2017estimation}, the hidden clique problem~\cite{deshpande2015finding}, the joint alignment problem~\cite{chen2018projected}, and cone-constrained PCA~\cite{deshpande2014cone,yi2020non}. For example, under the simple spiked Wigner model~\cite{perry2018optimality} for the observed data matrix $\bV$ with the underlying signal being assumed to lie in a convex cone, the authors of~\cite{deshpande2014cone} show that cone-constrained PCA can be computed efficiently via a generalized projected power method. In general, the range of a Lipschitz-continuous generative model is non-convex and not a cone. In addition, we consider a matrix model that is more general than the spiked Wigner model.  
\end{remark}

\begin{algorithm}[t]
\caption{A projected power method for GPCA (\texttt{PPower})}
\label{algo:ppower}
{\bf Input}: $\bV$, number of iterations $T$, pre-trained generative model $G$, initial vector $\bw^{(0)}$ \\
{\bf Procedure:} Iterate $\bw^{(t+1)} = \calP_{G}\big(\bV\bw^{(t)}\big)$ for $t \in \{0,1,\dotsc,T-1\}$, and return $\bw^{(T)}$
\end{algorithm}
\end{itemize}

Although it is not needed for our main results, we first state a lemma (proved in Appendix~\ref{app:bVPSD}) that establishes a monotonicity property with minimal assumptions, only requiring that $\bV$ is PSD; see also Proposition 3 of \cite{yuan2013truncated} for an analog in sparse PCA.  By comparison, our main results in Section \ref{sec:main} will make more assumptions, but will also provide stronger guarantees.  Note that the PSD assumption holds, for example, when $\bE = \mathbf{0}$, or when $\bV$ is a sample covariance matrix. 

\begin{lemma}\label{lem:bVPSD}
For any $\bx \in \bbR^n$, let $Q(\bx) = \bx^T \bV\bx$. Then, if $\bV$ is PSD, the sequence $\{Q(\bw^{(t)})\}_{t > 0}$ for $\bw^{(t)}$ in~\eqref{eq:pgd_genPCA} is monotonically non-decreasing.
\end{lemma}

\section{Specialized Data Models and Examples} \label{sec:examples}

In this section, we make more specific assumptions on $\bV = \bar{\bV} +\bE$, starting with the following.
\begin{assumption}[Assumption on $\bar{\bV}$]\label{assump:barbV}
 Assume that $\bar{\bV}$ is PSD with eigenvalues $\bar{\lambda}_1 > \bar{\lambda}_2 \ge  \ldots \ge \bar{\lambda}_n \ge 0$. We use $\bar{\bx}$ (a unit vector) to represent the eigenvector of $\bar{\bV}$ that corresponds to $\bar{\lambda}_1$. 
\end{assumption}

In the following, it is useful to think of $\bar{\bx}$ is being close to the range of the generative model $G$. In the special case of~\eqref{eq:gpca}, letting $m$ be the number of samples, it is natural to derive that the upper bound of $\|\bE\|_{2 \to 2}$ grows linearly in $(n/m)^b$ for some positive constant $b$ such as $\frac{1}{2}$ or $1$ (with high probability; see, e.g.,~\cite[Corollary~5.35]{vershynin2010introduction}). In the following, we consider general scenarios with $\bV$ depending on $m$ samples (see below for specific examples).   Similarly to~\cite{yuan2013truncated}, we may consider a restricted version of $\|\bE\|_{2\to 2}$, leading to the following.

\begin{assumption}[Assumption on $\bE$]\label{assump:E_cond}
 Let $S_1,S_2$ be two (arbitrary) finite sets in $\bbR^n$ satisfying $m = \Omega(\log(|S_1|\cdot |S_2|))$. Then, we have for all $\bs_1 \in S_1$ and $\bs_2 \in S_2$ that
 \begin{equation}\label{eq:bE_cond}
  \left|\bs_1^T \bE \bs_2\right| \le C \sqrt{\frac{\log (|S_1| \cdot |S_2|)}{m}} \cdot \|\bs_1\|_2 \cdot \|\bs_2\|_2, 
 \end{equation}
 where $C$ is an absolute constant. In addition, we have $\|\bE\|_{2\to 2} = O(n/m)$.\footnote{For the spectral norm of $\bE$, one often expects an even tighter bound $O(\sqrt{n/m})$, but we use $O(n/m)$ to simplify the analysis of our examples. Moreover, at least under the typical scaling where $L$ is polynomial in $n$~\cite{bora2017compressed}, the upper bound for $\|\bE\|_{2\to 2}$ can be easily relaxed to $O\big((n/m)^b\big)$ for any positive constant $b$, without affecting the scaling of our derived statistical rate.}
\end{assumption}

The following examples show that when the number of measurements is sufficiently large, the data matrices corresponding to certain spiked matrix and phase retrieval models satisfy the above assumptions with high probability. Short proofs are given in Appendix~\ref{app:proof_examples} for completeness. 

\begin{example}[Spiked covariance model]\label{exam:spikedCov}
 In the spiked covariance model~\cite{johnstone2009consistency,deshpande2016sparse}, the observed vectors $\bx_1,\bx_2,\ldots,\bx_m \in \bbR^n$ are of the form
\begin{equation}\label{eq:spikedCov}
 \bx_i = \sum_{q =1}^r \sqrt{\beta_q} u_{q,i} \bs_q + \bz_i,
\end{equation}
where $\bs_1,\ldots,\bs_r \in \bbR^n$ are orthonormal vectors that we want to estimate, while $\bz_i \sim \calN(\mathbf{0},\bI_n)$ and $u_{q,i} \sim \calN(0,1)$ are independent and identically distributed. In addition, $\beta_1,\ldots,\beta_r$ are positive constants that dictate the signal-to-noise ratio (SNR). To simplify the exposition, we focus on the rank-one case and drop the subscript $q \in [r]$. 
Let 
\begin{equation}\label{eq:V_spikedCov}
 \bV = \frac{1}{m}\sum_{i=1}^m (\bx_i\bx_i^T -\bI_n),
\end{equation}
and $\bar{\bV} = \bbE[\bV] = \beta \bs \bs^T$.\footnote{To avoid non-essential complications, $\beta$ is typically assumed to be known~\cite{johnstone2009consistency}. } Then, $\bar{\bV}$ satisfies Assumption~\ref{assump:barbV} with $\bar{\lambda}_1 = \beta >0$, $\bar{\lambda}_2 =\ldots =\bar{\lambda}_n =0$, and $\bar{\bx} = \bs$. In addition, letting $\bE = \bV -\bar{\bV}$, the Bernstein-type inequality~\cite[Proposition~5.10]{vershynin2010introduction} for the sum of sub-exponential random variables yields that for any finite sets $S_1,S_2 \subset \bbR^n$, when $m = \Omega\big(\log (|S_1|\cdot|S_2|)\big)$, with probability $1- e^{-\Omega(\log (|S_1|\cdot|S_2|))}$, $\bE$ satisfies~\eqref{eq:bE_cond} in Assumption~\ref{assump:E_cond}. Moreover, standard concentration arguments give $\|\bE\|_{2\to 2} = O(n/m)$ with probability $1-e^{-\Omega(n)}$.
\end{example}

\begin{remark}
We can also consider the simpler spiked Wigner model~\cite{perry2018optimality,chung2019weak} where $\bV = \beta \bs \bs^T + \frac{1}{\sqrt{n}}\bH$, with the signal $\bs$ being a unit vector, $\beta >0$ being an SNR parameter, and $\bH \in \bbR^{n \times n}$ being a symmetric matrix with entries drawn i.i.d.~(up to symmetry) from $\calN(0,1)$.  In this case, when $m = n$ is sufficiently large, with high probability, $\bar{\bV}: = \bbE[\bV] = \beta \bs\bs^T$ and $\bE := \bV -\bar{\bV}$ similarly satisfy Assumptions~\ref{assump:barbV} and~\ref{assump:E_cond} respectively.
\end{remark}

\begin{example}[Phase retrieval]\label{exam:noiseless_pr}
 Let $\bA \in \bbR^{m \times n}$ be a matrix having i.i.d.~$\calN(0,1)$ entries, and let $\ba_i^T$ be the $i$-th row of $\bA$. 
 For some unit vector $\bs$, suppose that the observed vector is $\by = |\bA \bs|$, where the absolute value is applied element-wise.\footnote{Without loss of generality, we assume that $\bs$ is a unit vector. For a general signal $\bs$, we may instead focus on estimating $\bar{\bs} = \bs/\|\bs\|_2$, and simply use $\frac{1}{m}\sum_{i=1}^m y_i$ to approximate $\|\bs\|_2$.} We construct the weighted empirical covariance matrix as follows~\cite{zhang2017nonconvex,liu2021towards}:
 \begin{equation}\label{eq:V_pr}
  \bV = \frac{1}{m} \sum_{i=1}^m \left(y_i \ba_i \ba_i^T  \mathbf{1}_{\{l < y_i < u\}} - \gamma \bI_n\right), 
 \end{equation}
where $u > l> 1$ are positive constants, and for $g \sim \calN(0,1)$, $\gamma := \bbE\big[|g|\mathbf{1}_{\{l < |g| < u\}}\big]$. Let $\bar{\bV} = \bbE[\bV] = \beta \bs\bs^T$, where $\beta := \bbE\big[\big(|g|^3 - |g|\big)\mathbf{1}_{\{l < |g| < u\}}\big]$. Then, $\bar{\bV}$ satisfies Assumption~\ref{assump:barbV} with $\bar{\lambda}_1 = \beta >0$, $\bar{\lambda}_2 =\ldots =\bar{\lambda}_n =0$, and $\bar{\bx} = \bs$. In addition, letting $\bE = \bV - \bar{\bV}$, we have similarly to Example~\ref{exam:spikedCov} that $\bE$ satisfies Assumption~\ref{assump:E_cond} with high probability. 
\end{example}

\section{Main Results} \label{sec:main}

The following theorem concerns globally optimal solutions of~\eqref{eq:opt_genPCA}. The proof is given in Appendix~\ref{app:thmGlobally}.  

\begin{theorem}\label{thm:barbx_bxG}
Let $\bV = \bar{\bV} + \bE$ with Assumptions~\ref{assump:barbV} and~\ref{assump:E_cond} being satisfied by $\bar{\bV}$ and $\bE$ respectively, and let $\bx_G := \calP_{G}(\bar{\bx}) =\arg\min_{\bw \in \mathrm{Range}(G)}\|\bw-\bar{\bx}\|_2$. Suppose that $\hat{\bv}$ is a globally optimal solution to~\eqref{eq:opt_genPCA}. Then, for any $\delta\in (0,1)$, we have 
\begin{equation}
 \|\hat{\bv}\hat{\bv}^T - \bar{\bx}\bar{\bx}^T\|_\rmF =  \frac{O\left(\sqrt{\frac{k \log\frac{Lr}{\delta}}{m}}\right)}{\bar{\lambda}_1 - \bar{\lambda}_2} + O\left(\sqrt{\frac{\delta n/m}{\bar{\lambda}_1 - \bar{\lambda}_2}}\right) + O\left(\sqrt{\frac{(\bar{\lambda}_1 + \epsilon_n) \|\bar{\bx}-\bx_G\|_2}{\bar{\lambda}_1 - \bar{\lambda}_2}}\right),
\end{equation}
where $\epsilon_n = O\big( \sqrt{\frac{k \log\frac{Lr}{\delta}}{m}}\big)$.
\end{theorem}

We have stated this result as an upper bound on $\|\hat{\bv}\hat{\bv}^T - \bar{\bx}\bar{\bx}^T\|_\rmF$, which intuitively measures the distance between the 1D subspaces spanned by $\hat{\bv}$ and $\bar{\bx}$.  Note, however, that for any two unit vectors $\bw_1,\bw_2$ with $\bw_1^T\bw_2 \ge 0$, the  distances $\|\bw_1-\bw_2\|_2$ and $\|\bw_1\bw_1^T-\bw_2\bw_2^T\|_\rmF$ are equivalent up to constant factors, whereas if $\bw_1^T\bw_2 \le 0$, then a similar statement holds for $\|\bw_1+\bw_2\|_2$.  See Appendix~\ref{sec:distances} for the precise statements.

In Theorem~\ref{thm:barbx_bxG}, the final term quantifies the effect of representation error.  When there is no such error (i.e., $\bar{\bx} \in \mathrm{Range}(G)$), under the scaling $\bar{\lambda}_1 - \bar{\lambda}_2 = \Theta(1)$, $L = n^{\Omega(1)}$, and $\delta = O(1/n)$,  Theorem~\ref{thm:barbx_bxG} simplifies to $\|\hat{\bv}\hat{\bv}^T - \bar{\bx}\bar{\bx}^T\|_\rmF = O(\sqrt{\frac{k\log L}{m}})$. This provides a natural counterpart to the statistical rate of order $\sqrt{\frac{s\log n}{m}}$ for SPCA mentioned in Section~\ref{sec:rel_work}. In Appendix~\ref{app:lb_gpca}, we provide an algorithm-independent lower bound showing that this scaling cannot be improved in general.


Before providing our main theorem for PPower described in~\eqref{eq:pgd_genPCA}, we present the following important lemma, whose proof is presented in Appendix~\ref{app:main_lem}. To simplify the statement of the lemma, we fix $\delta$ to be $O(1/n)$, though a more general form analogous to Theorem~\ref{thm:barbx_bxG} is also possible. 

\begin{lemma}\label{lem:main_lem}
Let $\bV = \bar{\bV} + \bE$ with Assumptions~\ref{assump:barbV} and~\ref{assump:E_cond} being satisfied by $\bar{\bV}$ and $\bE$ respectively, and further assume that $\bar{\bx} \in \mathrm{Range}(G)$. 
Let $\bar{\gamma} = \bar{\lambda}_2/\bar{\lambda}_1$ with $\bar{\lambda}_1 = \Theta(1)$. 
Then, for {\em all} $\bs \in \mathrm{Range}(G)$ satisfying $\bs^T\bar{\bx} > 0$, we have  
 \begin{equation}
  \|\calP_G(\bV\bs)-\bar{\bx}\|_2 \le \frac{2\bar{\gamma}\|\bs-\bar{\bx}\|_2}{\bs^T\bar{\bx}} + \frac{C}{\bs^T\bar{\bx}}\sqrt{\frac{k\log (nLr)}{m}},
 \end{equation}
 where $C$ is an absolute constant.
\end{lemma}
\begin{remark}
 The assumption $\bs^T\bar{\bx} > 0$ will be particularly satisfied when the range of $G$ only contains nonnegative vectors. As mentioned in various works studying nonnegative SPCA~\cite{zass2007nonnegative,sigg2008expectation,asteris2014nonnegative}, for several practical fields such as economics, bioinformatics, and computer vision, it is natural to assume that the underlying signal has no negative entries. More generally, the assumption $\bs^T\bar{\bx} > 0$ can be removed if we additionally have that $-\bar{\bx}$ is also contained in the range of $G$. For this case, when $\bs^T\bar{\bx} < 0$, we can instead derive an upper bound for $\|\hat{\bs}+\bar{\bx}\|_2$.
\end{remark}

Based on Lemma \ref{lem:main_lem}, we have the following theorem, whose proof is given in Appendix~\ref{app:thm_main}.
\begin{theorem}\label{thm:main}
Suppose that the assumptions on the data model $\bV = \bar{\bV} + \bE$ are the same as those in Lemma~\ref{lem:main_lem}, and assume that there exists $t_0 \in \bbN$ such that $\bar{\bx}^T\bw^{(t_0)} = 2\bar{\gamma} + \nu$ with $2\bar{\gamma} + \nu \le 1-\tau$, where $\bar{\gamma} = \bar{\lambda}_2/\bar{\lambda}_1 \in [0,1)$, and $\nu,\tau$ are both positive and scale as $\Theta(1)$. Let $\mu_0 = \frac{2\bar{\gamma}}{\bar{\bx}^T\bw^{(t_0)}} = \frac{2\bar{\gamma}}{2\bar{\gamma} + \nu} <1$, and in addition, suppose that $m \ge C_{\nu,\tau}\cdot k \log (nLr)$ with $C_{\nu,\tau}>0$ being large enough. Then, we have after $\Delta_0 = O\big(\log \big(\frac{m}{k \log(nLr)}\big)\big)$ iterations of PPower (beyond $t_0$) that  
 \begin{equation}\label{eq:either1}
 \|\bw^{(t)} - \bar{\bx}\|_2 \le \frac{C}{(1-\mu_0)\nu} \sqrt{\frac{k\log (nLr)}{m}},
\end{equation}
i.e., this equation holds for all $t \ge T_0 := t_0 + \Delta_0$. Moreover, if $\bar{\gamma} = 0$ then $\Delta_0 \le 1$, whereas if $\bar{\gamma} = \Theta(1)$, we have exponentially fast convergence via the following contraction property:  There exists a constant $\xi \in (0,1)$ such that for $t \in [t_0,T_0)$, it holds that
 \begin{equation}
  \|\bw^{(t+1)} - \bar{\bx}\|_2  \le (1-\xi) \|\bw^{(t)} - \bar{\bx}\|_2. \label{eq:recurse}
 \end{equation}
\end{theorem}

Regarding the assumption $\bar{\bx}^T \bw^{(t_0)} \ge 2\bar{\gamma}+\nu$, we note that when $t_0 = 0$, this condition can be viewed as having a good initialization. For both Examples~\ref{exam:spikedCov} and~\ref{exam:noiseless_pr}, we have $\bar{\gamma} = 0$. Thus, for the spiked covariance and phase retrieval models corresponding to these examples, the assumption $ \bar{\bx}^T \bw^{(t_0)} \ge 2\bar{\gamma}+\nu$ reduces to $\bar{\bx}^T \bw^{(t_0)} \ge \nu$ for a sufficiently small positive constant $\nu$, which results in a mild assumption. Such an assumption is also required for the projected power method devised for cone-constrained PCA under the simple spiked Wigner model, with the underlying signal being assumed to lie in a convex cone; see~\cite[Theorem~3]{deshpande2014cone}. Despite using a similar assumption on the initialization, our proof techniques are significantly different from \cite{deshpande2014cone}; see Appendix~\ref{app:diff_proof_techs} for discussion.

When $L$ is polynomial in $n$, Theorem~\ref{thm:main} reveals that we have established conditions under which PPower in~\eqref{eq:pgd_genPCA} converges exponentially fast to a point achieving the statistical rate of order $\sqrt{\frac{k\log L}{m}}$. From the algorithm-independent lower bound established in Appendix~\ref{app:lb_gpca}, we know that this statistical rate nearly matches the optimal rate for GPCA.   
We highlight that Theorem \ref{thm:main} partially addresses the computational-to-statistical gap (e.g., see~\cite{wang2016statistical,hand2018phase,aubin2019spiked,cocola2020nonasymptotic}) for spiked matrix recovery and phase retrieval under a generative prior, though closing it completely would require efficiently finding a good initialization and addressing the assumption of exact projections.

Perhaps the main caveat to Theorem \ref{thm:main} is that it assumes the projection step can be performed exactly.  However, this is a standard assumption in analyses of projected gradient methods, e.g., see~\cite{shah2018solving}, and both gradient-based projection and GAN-based projection have been shown to be highly effective in practice~\cite{shah2018solving,raj2019gan}.

 \section{Experiments}
 
 In this section, we experimentally study the performance of Algorithm~\ref{algo:ppower}~(\texttt{PPower}). We note that these experiments are intended as a simple proof of concept rather than seeking to be comprehensive, as our contributions are primarily theoretical. 
 We compare with the truncated power method (\texttt{TPower}) devised for SPCA proposed in~\cite[Algorithm~1]{yuan2013truncated} and the vanilla power method (\texttt{Power}) that performs the iterative procedure $\bw^{(t+1)} = (\bV\bw^{(t)})/\|\bV\bw^{(t)}\|_2$.
 For a fair comparison, for~\texttt{PPower},~\texttt{TPower}, and~\texttt{Power}, we use the same initial vector.  Specifically, as mentioned in~\cite[Section~V]{liu2021towards}, we choose the initialization vector $\bw^{(0)}$ as the column of $\bV$ that corresponds to its largest diagonal entry. For all three algorithms, the total number of iterations $T$ is set to be $30$. To compare the performance across algorithms, we use the scale-invariant Cosine Similarity metric (in absolute value) defined as $\mathrm{Cossim} \left(\bar{\bx},\bw^{(T)}\right) := \frac{|\left\langle\bar{\bx},\bw^{(T)} \right\rangle|}{\|\bar{\bx}\|_2\|\bw^{(T)}\|_2}$, where $\bar{\bx}$ is the ground-truth signal to estimate, and $\bw^{(T)}$ denotes the output vector of the algorithm. Note that we take the absolute value because any estimate $\hat{\bx}$ or its negative $-\hat{\bx}$ are considered equally good in all the baseline methods we compare with.

The experiments are performed on the MNIST~\cite{lecun1998gradient}, Fashion-MNIST~\cite{xiao2017fashion} and CelebA~\cite{liu2015deep} datasets, with the numerical results for the Fashion-MNIST and CelebA datasets being presented in Appendices~\ref{app:exp_fashionMNIST} and~\ref{app:exp_CelebA}. The MNIST dataset consists of $60,000$ images of handwritten digits. The size of each image is $28 \times 28$, and thus $n = 784$. To reduce the impact of local minima, we perform $10$ random restarts, and choose the best among these. The cosine similarity is averaged over the test images, and also over these $10$ random restarts. The generative model $G$ is set to be a pre-trained variational autoencoder (VAE) model with latent dimension $k = 20$. We use the VAE model trained by the authors of~\cite{bora2017compressed} directly, for which the encoder and decoder are both fully connected neural networks with two hidden layers, with the architecture being $20-500-500-784$. The VAE is trained by the Adam optimizer with a mini-batch size of $100$ and a learning rate of $0.001$.  The projection step $\calP_G(\cdot)$ is solved by the Adam optimizer with a learning rate of $0.03$ and $200$ steps. In each iteration of~\texttt{TPower}, the calculated entries are truncated to zero except for the largest $q$ entries, where $q \in \bbN$ is a tuning parameter. Since for~\texttt{TPower}, $q$ is usually selected as an integer larger than the true sparsity level, and since it is unlikely that the image of the MNIST dataset can be well approximated by a $k$-sparse vector with $k = 20$, we choose a relatively large $q$, namely $q = 150$. Similarly to~\cite{bora2017compressed} and other related works, we only report the results on a test set that is unseen by the pre-trained VAE model, i.e., the training of $G$ and the PPower computations do not use common data.\footnote{All experiments are run using Python 3.6 and Tensorflow 1.5.0, with a NVIDIA GeForce GTX 1080 Ti 11GB GPU. The corresponding code is available at~\url{https://github.com/liuzq09/GenerativePCA}.}
  
 \begin{enumerate}
 \item \textbf{Spiked covariance model (Example~\ref{exam:spikedCov})}: The numerical results are shown in Figures~\ref{fig:mnist_imgs_spikedCov} and~\ref{fig:quant_mnist_spikedCov}. We observe from Figure~\ref{fig:mnist_imgs_spikedCov} that~\texttt{Power} and~\texttt{TPower} attain poor reconstructions, and the generative prior based method~\texttt{PPower} attains significantly better reconstructions. To illustrate the effect of the sample size $m$, we fix the SNR parameter $\beta = 1$ and vary $m$ in $\{100,200,300,400,500\}$. In addition, to illustrate the effect of the SNR parameter $\beta$, we fix $m = 300$, and vary $\beta$ in $\{0.6,0.7,0.8,0.9,1,2,3,4\}$. From Figure~\ref{fig:quant_mnist_spikedCov}, we observe that for these settings of $m$ and $\beta$,~\texttt{PPower} always leads to a much higher cosine similarity compared to~\texttt{Power} and~\texttt{TPower}, which is natural given the more precise modeling assumptions used.
 
  \begin{figure} \vspace*{-2ex}
\begin{center}
\begin{tabular}{cc}
\includegraphics[height=0.195\textwidth]{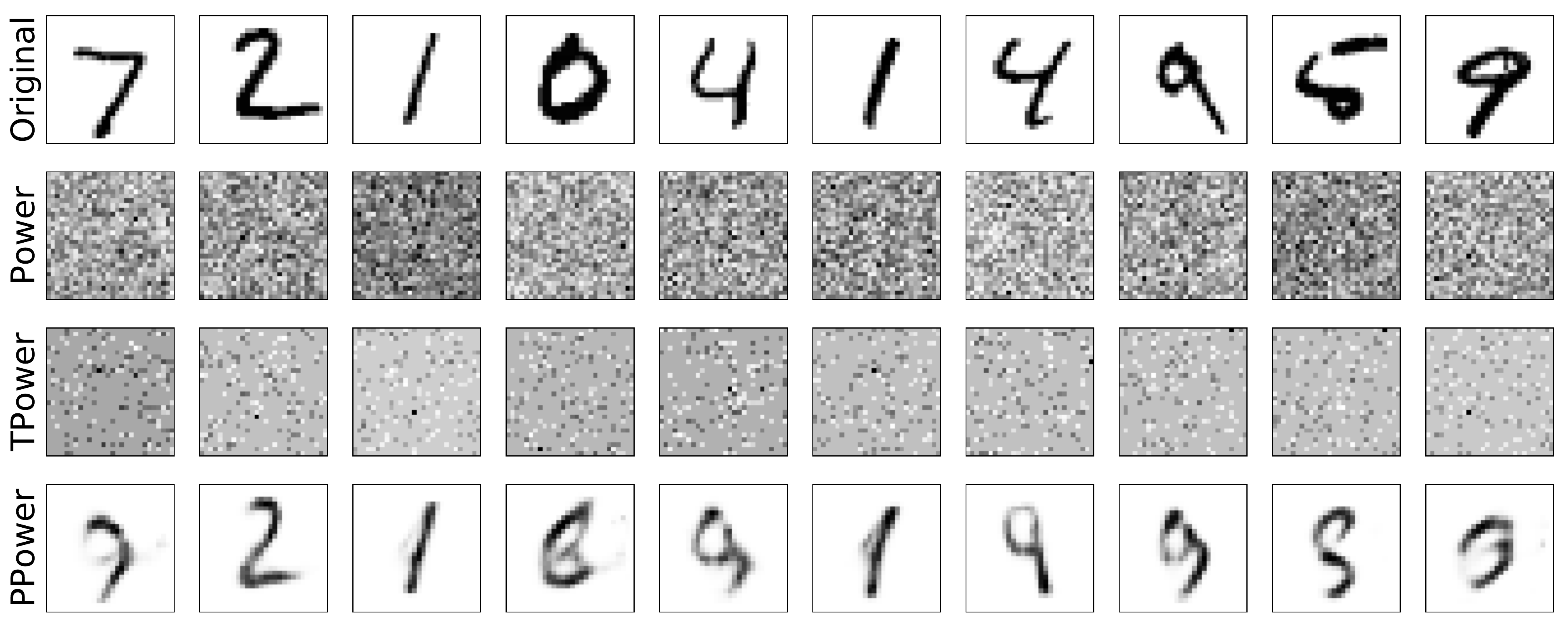} & \hspace{-0.5cm}
\includegraphics[height=0.195\textwidth]{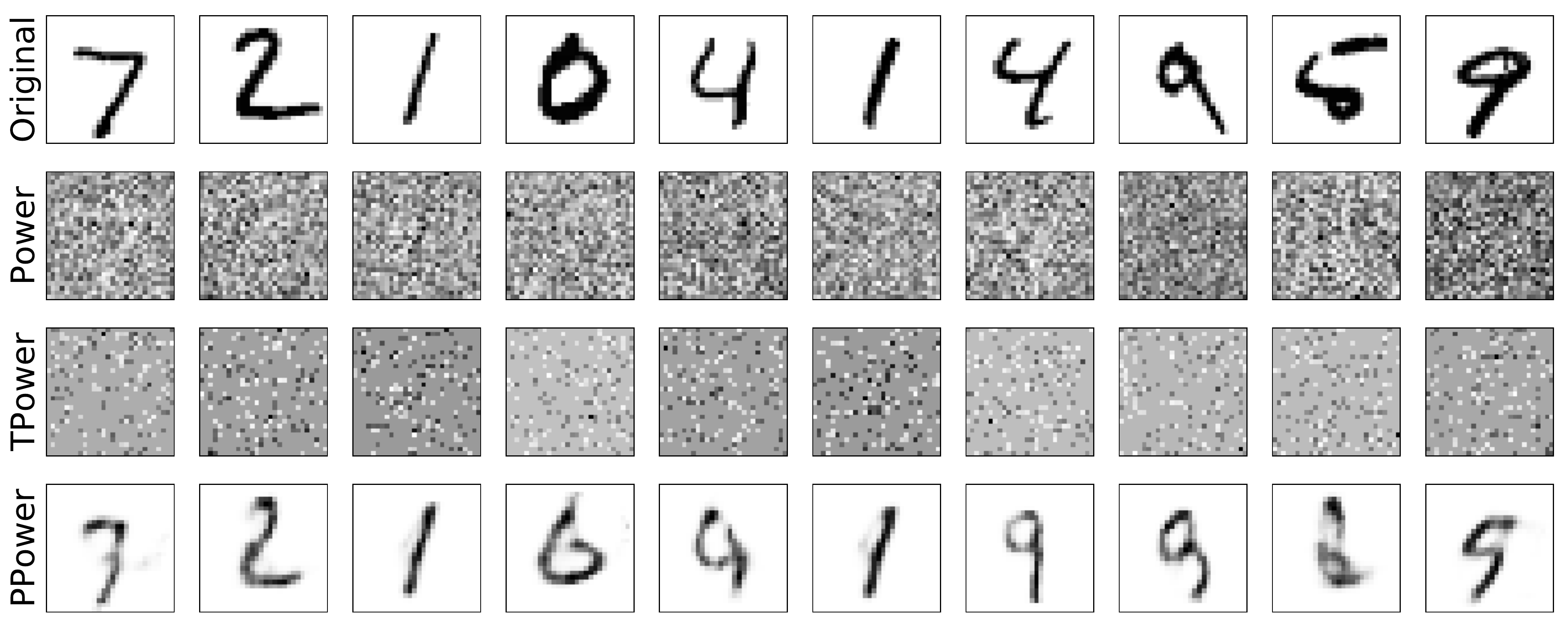} \\
{\small (a) $\beta = 1$ and $m = 200$} & {\small (b) $\beta = 2$ and $m = 100$}
\end{tabular}
\caption{Examples of reconstructed images of the MNIST dataset for the spiked covariance model.}
\label{fig:mnist_imgs_spikedCov}  \vspace*{-2ex}
\end{center}
\end{figure} 


 \begin{figure}
\begin{center}
\begin{tabular}{ccc}
\includegraphics[height=0.25\textwidth]{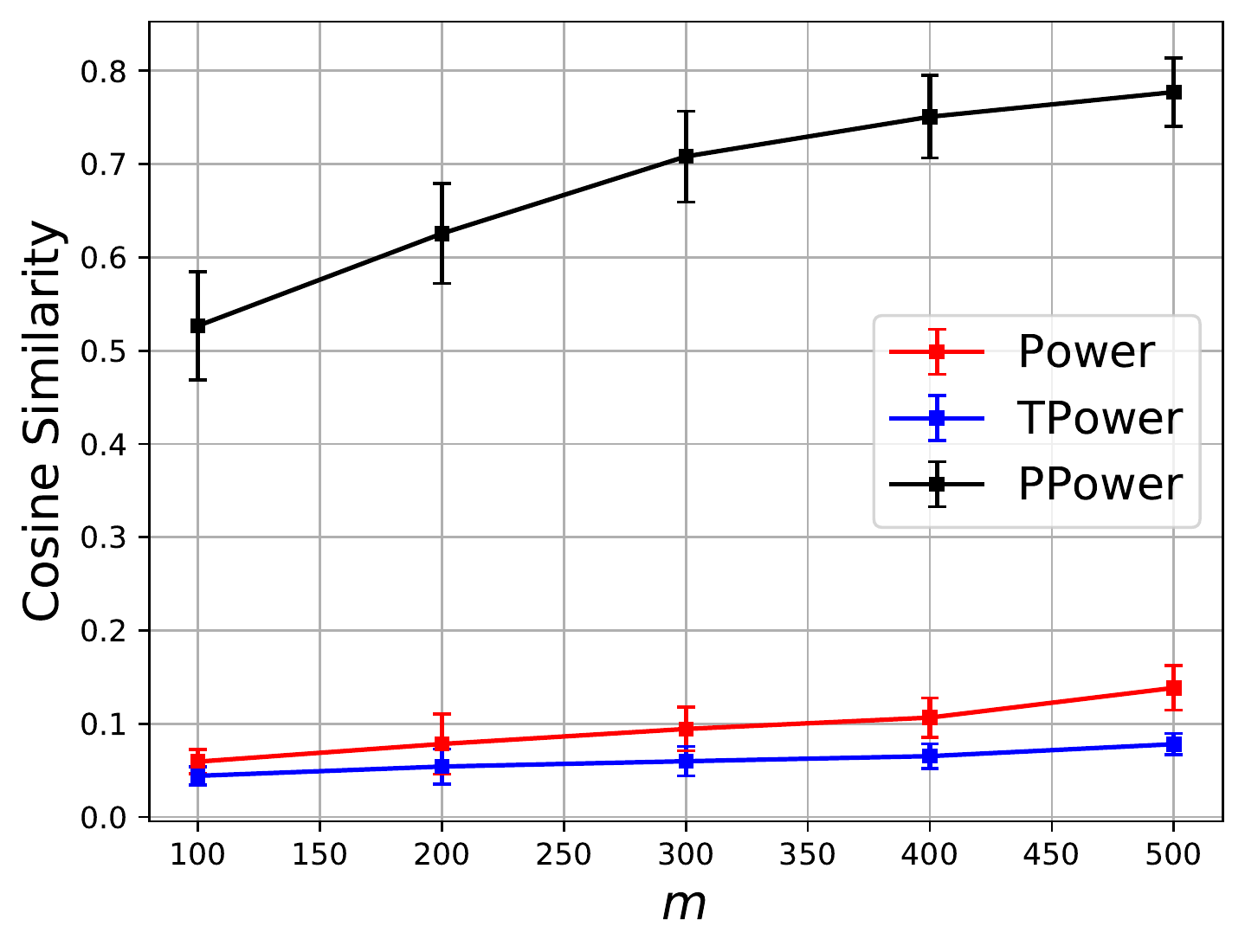} & \hspace{-0.5cm}
\includegraphics[height=0.25\textwidth]{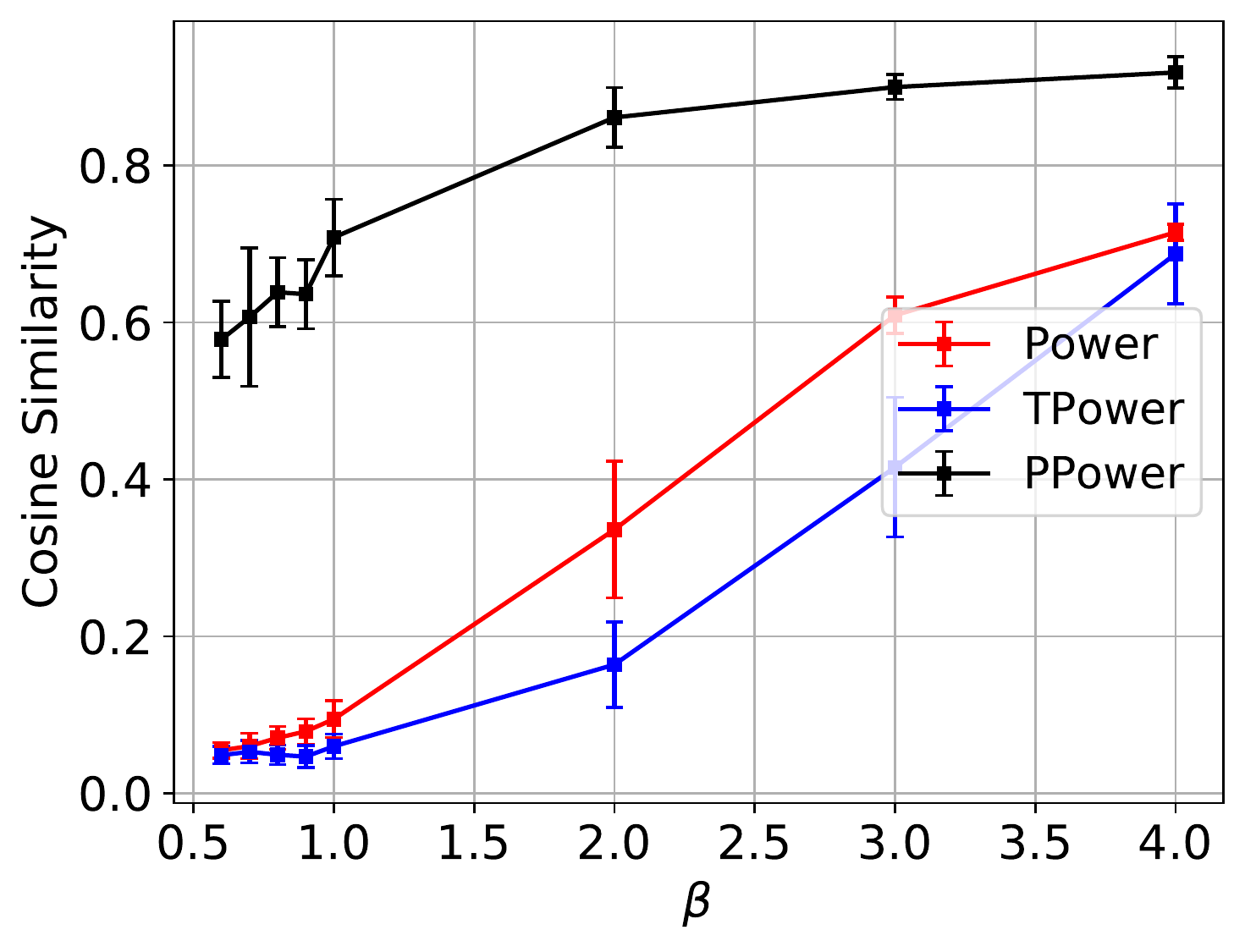} & \hspace{-0.5cm}
\includegraphics[height=0.25\columnwidth]{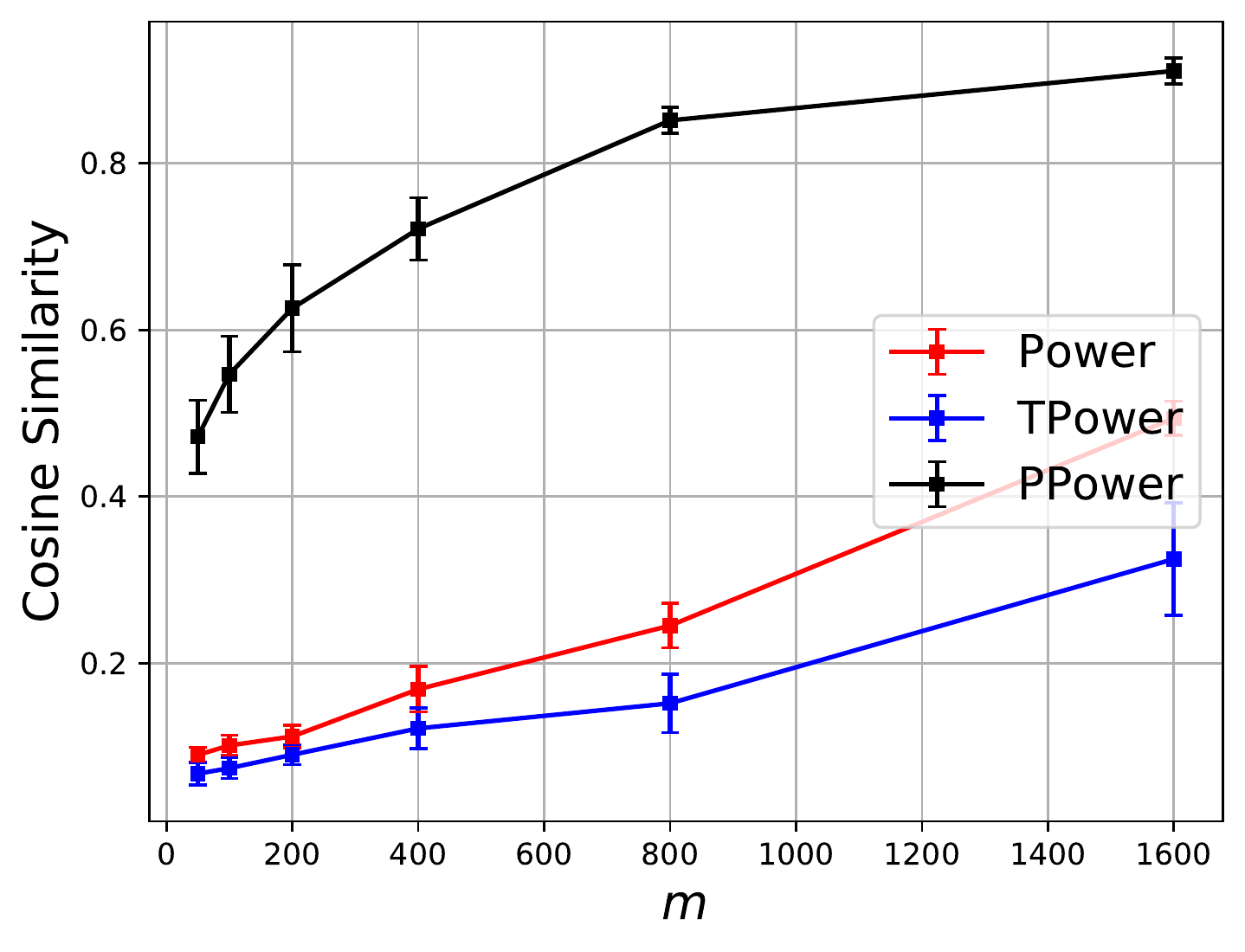} \\
{\small (a) Fixing $\beta = 1$ and varying $m$} & \hspace{-0.8cm} {\small (b) Fixing $m = 300$ and varying $\beta$} & \hspace{-0.7cm} {\small (c) Analog of (a) for phase ret.}
\end{tabular}
\caption{Quantitative comparisons of the performance of~\texttt{Power},~\texttt{TPower} and~\texttt{PPower} according to the Cosine Similarity for the MNIST dataset, under both the spiked covariance model (Left/Middle) and phase retrieval model (Right).} \label{fig:quant_mnist_spikedCov} \vspace*{-2ex}
\end{center}
\end{figure} 
 
  \item  \textbf{Phase retrieval (Example~\ref{exam:noiseless_pr})}: The results are shown in Figure~\ref{fig:quant_mnist_spikedCov} (Right) and Figure~\ref{fig:mnist_imgs_pr}. Again, we can observe that~\texttt{PPower} significantly outperforms~\texttt{Power} and~\texttt{TPower}. In particular, for sparse phase retrieval, when performing experiments on image datasets, even for the noiseless setting, solving an eigenvalue problem similar to~\eqref{eq:opt_genPCA} can typically only serve as a spectral initialization step, with a subsequent iterative algorithm being required to refine the initial guess. In view of this, it is notable that for phase retrieval with generative priors,~\texttt{PPower} can return meaningful reconstructed images for $m = 200$, which is small compared to $n = 784$. 
   \end{enumerate}
   
   \section{Conclusion}

We have proposed a quadratic estimator for eigenvalue problems with generative models, and we showed that this estimator attains a statistical rate of order $\sqrt{\frac{k\log L}{m}}$. We also showed that this statistical rate is almost tight by providing a near-matching algorithm-independent lower bound. Furthermore, we provided a projected power method to efficiently solve (modulo the complexity of the projection step) the corresponding optimization problem, and showed that our method converges exponentially fast to a point achieving a statistical rate of order $\sqrt{\frac{k \log L}{m}}$ under suitable conditions.

  \begin{figure}
\begin{center}
\begin{tabular}{cc}
\includegraphics[height=0.19\textwidth]{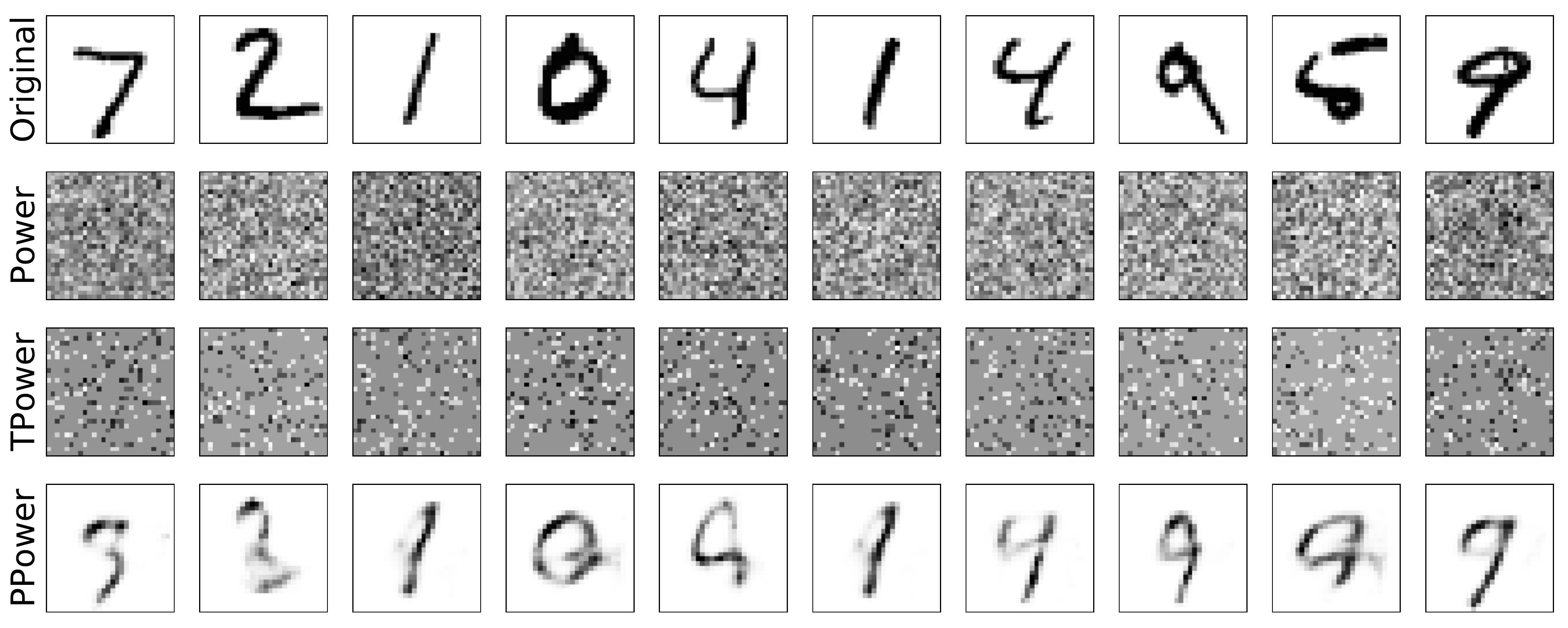} & \hspace{-0.5cm}
\includegraphics[height=0.19\textwidth]{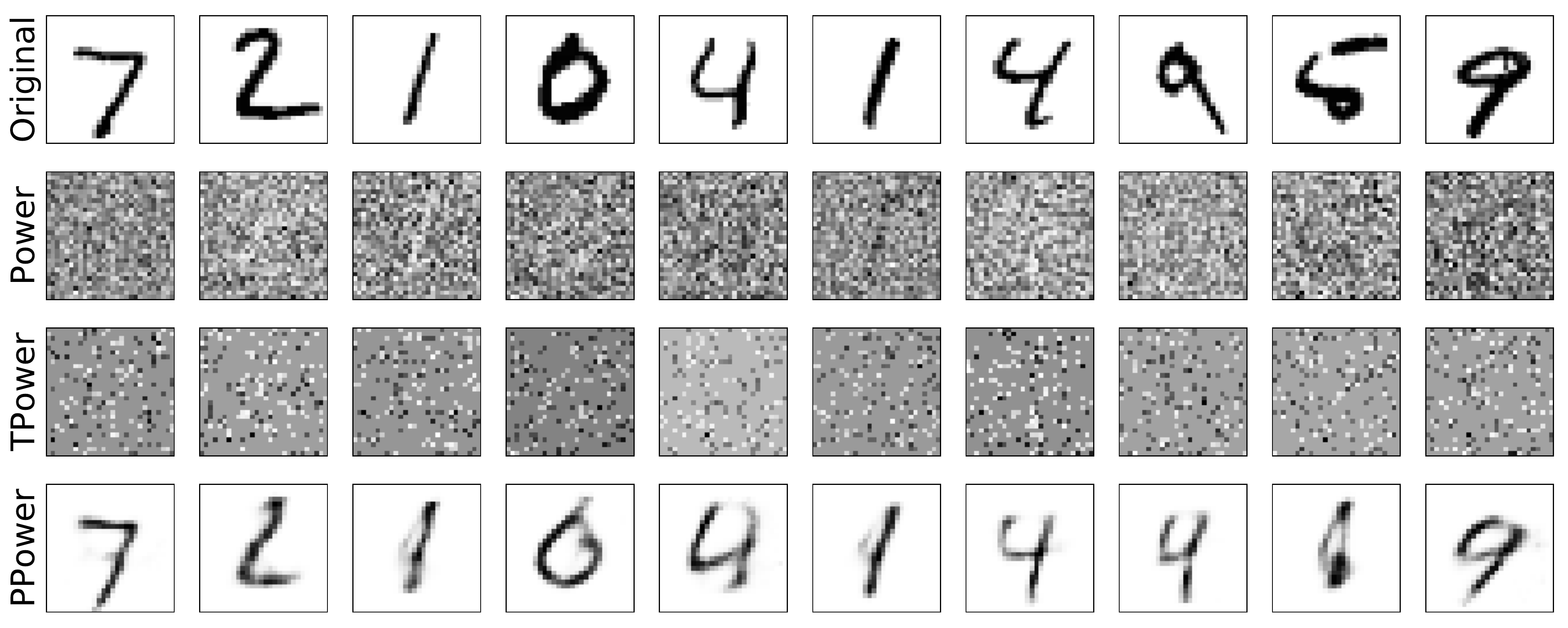} \\
{\small (a) $m = 200$} & {\small (b) $m = 400$}
\end{tabular}
\caption{Examples of reconstructed images of the MNIST dataset for phase retrieval.} \label{fig:mnist_imgs_pr}
\end{center}
\end{figure}


\appendices

\section{Proof of Lemma~\ref{lem:bVPSD} (Non-Decreasing Property of $Q$)}
\label{app:bVPSD}

 Since $\bw^{(t+1)} = \calP_{G} \left(\bV \bw^{(t)}\right)$ and $\bw^{(t)} \in \mathrm{Range}(G)$, we have 
 \begin{equation}
  \|\bV \bw^{(t)} - \bw^{(t+1)}\|_2 \le \|\bV \bw^{(t)} - \bw^{(t)}\|_2,
 \end{equation}
and since $\|\bw^{(t+1)}\|_2 = \|\bw^{(t)}\|_2 =1$, expanding the square gives 
\begin{equation}\label{eq:lem1_base1}
\left\langle \bV \bw^{(t)}, \bw^{(t+1)} \right\rangle \ge Q(\bw^{(t)}).
\end{equation}
Then, we obtain
\begin{align}
 Q(\bw^{(t+1)}) &= \left\langle \bV \bw^{(t+1)}, \bw^{(t+1)} \right\rangle \\
 & = Q(\bw^{(t+1)} - \bw^{(t)}) + 2 \left\langle \bV (\bw^{(t+1)} -\bw^{(t)}) , \bw^{(t)} \right\rangle + Q(\bw^{(t)}) \label{eq:lem1_exp0} \\
 & \ge 2 \left\langle \bV (\bw^{(t+1)} -\bw^{(t)}) , \bw^{(t)} \right\rangle + Q(\bw^{(t)})\label{eq:lem1_exp1} \\
 & \ge Q(\bw^{(t)})\label{eq:lem1_exp2},
\end{align}
where~\eqref{eq:lem1_exp0} follows by writing $\bw^{(t+1)} = \bw^{(t)} + (\bw^{(t+1)} - \bw^{(t)})$ and expanding, \eqref{eq:lem1_exp1} follows from the assumption that $\bV$ is PSD, and~\eqref{eq:lem1_exp2} follows from~\eqref{eq:lem1_base1}.

\section{Proofs for Spiked Matrix and Phase Retrieval Examples}\label{app:proof_examples}

Before proceeding, we present the following standard definitions.

\begin{definition} \label{def:subg}
 A random variable $X$ is said to be sub-Gaussian if there exists a positive constant $C$ such that $\left(\mathbb{E}\left[|X|^{p}\right]\right)^{1/p} \leq C  \sqrt{p}$ for all $p\geq 1$.  The sub-Gaussian norm of a sub-Gaussian random variable $X$ is defined as $\|X\|_{\psi_2}:=\sup_{p\ge 1} p^{-1/2}\left(\mathbb{E}\left[|X|^{p}\right]\right)^{1/p}$. 
\end{definition}

\begin{definition}
 A random variable $X$ is said to be sub-exponential if there exists a positive constant $C$ such that $\left(\bbE\left[|X|^p\right]\right)^{\frac{1}{p}} \le C p$ for all $p \ge 1$. The sub-exponential norm of $X$ is defined as $\|X\|_{\psi_1} := \sup_{p \ge 1} p^{-1} \left(\bbE\left[|X|^p\right]\right)^{\frac{1}{p}}$.
\end{definition}

The following lemma states that the product of two sub-Gaussian random variables is sub-exponential, regardless of the dependence between them.

\begin{lemma}{\em \hspace{1sp}\cite[Lemma~2.7.7]{vershynin2018high}}\label{lem:prod_subGs}
  Let $X$ and $Y$ be sub-Gaussian random variables (not necessarily independent). Then $XY$ is sub-exponential, and satisfies
  \begin{equation}
   \|XY\|_{\psi_1} \le \|X\|_{\psi_2}\|Y\|_{\psi_2}.
  \end{equation}
\end{lemma}

The following lemma provides a useful concentration inequality for the sum of independent sub-exponential random variables.
\begin{lemma}{\em \hspace{1sp}\cite[Proposition~5.16]{vershynin2010introduction}}\label{lem:large_dev}
Let $X_{1}, \ldots , X_{N}$ be independent zero-mean sub-exponential random variables, and $K = \max_{i} \|X_{i} \|_{\psi_{1}}$. Then for every $\balpha = [\alpha_1,\ldots,\alpha_N]^T \in \bbR^N$ and $\epsilon \geq 0$, it holds that
\begin{align}
 & \mathbb{P}\bigg( \Big|\sum_{i=1}^{N}\alpha_i X_{i}\Big|\ge \epsilon\bigg)  \leq 2  \exp \left(-c \cdot \mathrm{min}\Big(\frac{\epsilon^{2}}{K^{2}\|\balpha\|_2^2},\frac{\epsilon}{K\|\balpha\|_\infty}\Big)\right), \label{eq:subexp}
\end{align}
where $c > 0$ is an absolute constant.  In particular, with $\balpha = \big[ \frac{1}{N},\dotsc,\frac{1}{N} \big]^T$, we have
\begin{align}
    & \mathbb{P}\bigg( \Big| \frac{1}{N} \sum_{i=1}^{N} X_{i}\Big|\ge \epsilon\bigg)  \leq 2  \exp \left(-c \cdot \mathrm{min}\Big(\frac{N \epsilon^{2}}{K^{2}},\frac{N \epsilon}{K}\Big)\right). \label{eq:subexp2}
\end{align}
\end{lemma}

The widely-used notion of an $\epsilon$-net is introduced as follows.
\begin{definition}
Let $(\calX,d)$ be a metric space, and fix $\epsilon>0$. A subset $S \subseteq \calX$ is said be an {\em $\epsilon$-net} of $\calX$ if, for all $x \in \calX$, there exists some $s \in S$ such that $d(s,x) \le \epsilon$. The minimal cardinality of an $\epsilon$-net of $\calX$, if finite, is denoted $\calC(\calX,\epsilon)$ and is called the covering number of $\calX$ (at scale $\epsilon$).
\end{definition}

The following lemma provides a useful upper bound for the covering number of the unit sphere. 
\begin{lemma}{\em \hspace{1sp}\cite[Lemma~5.2]{vershynin2010introduction}}\label{lem:unit_sphere_cov}
 The unit Euclidean sphere $\calS^{N-1}$ equipped
with the Euclidean metric satisfies for every $\epsilon > 0$ that
\begin{equation}
 \calC(\calS^{N-1},\epsilon) \le \left(1+\frac{2}{\epsilon}\right)^{N}.
\end{equation}
\end{lemma}

The following lemma provides an upper bound for the spectral norm of a symmetric matrix.
\begin{lemma}{\em \hspace{1sp}\cite[Lemma~5.4]{vershynin2010introduction}}\label{lem:symm_spectral_norm}
 Let $\bX$ be a symmetric $N\times N$ matrix, and let $\calC_{\epsilon}$ be an $\epsilon$-net of $\calS^{N-1}$ for some $\epsilon \in [0,1/2)$. Then, 
 \begin{equation}
  \|\bX\|_{2\to 2} = \sup_{\br \in \calS^{N-1}} |\langle \bX\br,\br\rangle| \le (1-2\epsilon)^{-1} \sup_{\br\in \calC_{\epsilon}} |\langle \bX\br,\br\rangle|.
 \end{equation}
\end{lemma}

With the above auxiliary results in place, we provide the proofs of Assumption~\ref{assump:E_cond} holding for the two examples described in Section \ref{sec:examples}.

\subsection{Spiked Covariance Model (Example~\ref{exam:spikedCov})}

As per Assumption~\ref{assump:E_cond}, fix two finite signal sets $S_1$ and $S_2$.  For $r = 1$, we have $\bx_i = \sqrt{\beta}u_i \bs + \bz_i$ and a direct calculation gives $\bbE[\bV] = \bar{\bV} = \beta \bs\bs^T$.  Recall also that $\|\bs\|_2=1$, $u_i \sim \calN(0,1)$, and $\bz_i \sim \calN(\mathbf{0},\bI_n)$.  It follows that for any $\bs_1 \in S_1$, we have that $\bs_1^T\bx_i = \sqrt{\beta} u_i \bs^T\bs_1 + \bz_i^T\bs_1$ is sub-Gaussian, with the sub-Gaussian norm being upper bounded by $C(\sqrt{\beta}+1)\|\bs_1\|_2$. Similarly, we have for any $\bs_2 \in S_2$ that $\|\bs_2^T\bx_i\|_{\psi_2} \le C(\sqrt{\beta}+1)\|\bs_2\|_2$. Applying Lemma~\ref{lem:prod_subGs}, we deduce that $(\bs_1^T\bx_i)(\bs_2^T\bx_i)$ is sub-exponential, with the sub-exponential norm being upper bounded by $C^2 (\sqrt{\beta}+1)^2\|\bs_1\|_2 \|\bs_2\|_2$. In addition, from \eqref{eq:V_spikedCov} and $\bar{\bV} = \beta \bs\bs^T$, we have 
\begin{align}
 \bs_1^T\bE\bs_2 &= \bs_1^T(\bV -\bar{\bV})\bs_2 \\
 & = \frac{1}{m}\sum_{i=1}^m \left((\bx_i^T\bs_1)(\bx_i^T\bs_2) - \left((\bs_1^T\bs_2) + \beta (\bs^T\bs_1)(\bs^T\bs_2)\right)\right),
 \end{align}
and we observe that $\bbE[(\bx_i^T\bs_1)(\bx_i^T\bs_2)] = (\bs_1^T\bs_2) + \beta (\bs^T\bs_1)(\bs^T\bs_2)$. Then, from Lemma~\ref{lem:large_dev}, we obtain that for any $t>0$ satisfying $m = \Omega(t)$, the following holds with probability $1-e^{-\Omega(t)}$ (recall that $C$ may vary from line to line):
\begin{equation}\label{eq:main_bd_spikedCov}
 \left| \frac{1}{m}\sum_{i=1}^m \left((\bx_i^T\bs_1)(\bx_i^T\bs_2) - \left((\bs_1^T\bs_2) + \beta (\bs^T\bs_1)(\bs^T\bs_2)\right)\right)\right| \le C (\sqrt{\beta}+1)^2 \|\bs_1\|_2 \|\bs_2\|_2 \cdot  \frac{\sqrt{t}}{\sqrt{m}},
\end{equation}
where we note that the assumption $m = \Omega(t)$ ensures that the first term is dominant in the minimum in \eqref{eq:subexp2}.  Taking a union bound over all $\bs_1 \in S_1$ and $\bs_2 \in S_2$, and setting $t = \log (|S_1|\cdot|S_2|)$, we obtain with probability $1-e^{-\Omega(\log (|S_1| \cdot |S_2|))}$ that~\eqref{eq:bE_cond} holds (with $\beta$ being a fixed positive constant).  

Next, we bound $|\br^T\bE\br|$ for fixed $\br \in \calS^{n-1}$, but this time consider $t > 0$ (different from the above $t$) satisfying $t = \Omega(m)$.  In this case, we can follow the above analysis (with $\bs_1$ and $\bs_2$ both replaced by $\br$), but the assumption $t = \Omega(m)$ means that when applying Lemma~\ref{lem:large_dev}, the second term in the minimum in \eqref{eq:subexp2} is now the dominant one.  As a result, for any $t>0$ satisfying $t = \Omega(m)$, and any $\br \in \calS^{n-1}$, we have with probability $1-e^{-\Omega(t)}$ that
\begin{align}
  |\br^T\bE\br| &= \left| \frac{1}{m}\sum_{i=1}^m \left((\bx_i^T\br)^2 - \left(1 + \beta (\bs^T\br)^2\right)\right)\right| \\
  & \le C (\sqrt{\beta}+1)  \cdot  \frac{t}{m}.\label{eq:main_bd_spikedCov2}
\end{align}
From Lemma~\ref{lem:unit_sphere_cov}, there exists an $(1/4)$-net $\calC_{\frac{1}{4}}$ of $\calS^{n-1}$ satisfying $\log \big|\calC_{\frac{1}{4}}\big| \le n \log 9$. Taking a union bound over all $\br \in \calC_{\frac{1}{4}}$, and setting $t =C n$, we obtain with probability $1-e^{-\Omega(n)}$ that
\begin{equation}
 \sup_{\br \in \calC_{\frac{1}{4}}} \left|\br^T \bE \br\right| = O\left(\frac{n}{m}\right).
\end{equation}
Then, from Lemma~\ref{lem:symm_spectral_norm}, we have 
\begin{equation}\label{eq:ub_bE_spikedCov}
 \|\bE\|_{2\to 2} \le 2 \sup_{\br \in \calC_{\frac{1}{4}}} \left|\br^T \bE \br\right| = O\left(\frac{n}{m}\right).
\end{equation}

\subsection{Phase Retrieval (Example~\ref{exam:noiseless_pr})}

Let $\bW = \frac{1}{m} \sum_{i=1}^m y_i \ba_i\ba_i^T\mathbf{1}_{\{l < y_i < u\}}$ and $\bar{\bW}= \beta \bs\bs^T + \gamma \bI_n$. It is shown in~\cite[Lemma~8]{liu2021towards} that
\begin{equation}\label{eq:bWmean}
\bbE\left[\bW \right] = \bar{\bW},
\end{equation}
which implies
\begin{equation}
 \bbE[\bV] = \beta \bs\bs^T = \bar{\bV}.
\end{equation}
Then, for any $\bs_1 \in S_1$ and $\bs_2 \in S_2$, we have 
\begin{align}
 \bs_1^T\bE\bs_2 &= \bs_1^T(\bV -\bar{\bV})\bs_2 = \bs_1^T(\bW -\bar{\bW})\bs_2 \\
 & = \frac{1}{m}\sum_{i=1}^m \Big(y_i (\ba_i^T\bs_1)(\ba_i^T\bs_2)\mathbf{1}_{\{l < y_i < u\}} - \big(\beta (\bs^T\bs_1)(\bs^T\bs_2) + \gamma (\bs_1^T\bs_2)\big)\Big).
\end{align}
Since each $\ba_i$ has i.i.d.~$\calN(0,1)$ entries, we observe that $y_i (\ba_i^T\bs_1)(\ba_i^T\bs_2)\mathbf{1}_{\{l < y_i < u\}}$ is sub-exponential with the sub-exponential norm being upper bounded by $C u \|\bs_1\|_2 \|\bs_2\|_2$. In addition, from~\eqref{eq:bWmean}, we have $\bbE[y_i (\ba_i^T\bs_1)(\ba_i^T\bs_2)\mathbf{1}_{\{l < y_i < u\}}] = \beta (\bs^T\bs_1)(\bs^T\bs_2) + \gamma (\bs_1^T\bs_2)$. Then, from Lemma~\ref{lem:large_dev}, we obtain that for any $t>0$ satisfying $m = \Omega(t)$, with probability $1-e^{-\Omega(t)}$,
\begin{equation}\label{eq:bd_pr_mainTerm}
 \left|\frac{1}{m}\sum_{i=1}^m \left(y_i (\ba_i^T\bs_1)(\ba_i^T\bs_2)\mathbf{1}_{\{l < y_i < u\}} - (\beta (\bs^T\bs_1)(\bs^T\bs_2) + \gamma (\bs_1^T\bs_2))\right)\right| \le C u \|\bs_1\|_2 \|\bs_2\|_2 \cdot \frac{\sqrt{t}}{\sqrt{m}}.
\end{equation}
Taking a union bound over all $\bs_1 \in S_1$ and $\bs_2 \in S_2$, and setting $t = \log (|S_1|\cdot|S_2|)$, we obtain that with probability $1-e^{-\Omega(\log (|S_1| \cdot |S_2|))}$, \eqref{eq:bE_cond} holds as desired (with $u$ being a fixed positive constant). In addition, similarly to~\eqref{eq:ub_bE_spikedCov}, we have with probability $1-e^{-\Omega(n)}$ that $\|\bE\|_{2\to 2} = O\big(\frac{n}{m}\big)$.

\section{Equivalence of Distances} \label{sec:distances}

The following lemma gives a useful equivalence between two distances.

\begin{lemma}\label{lem:simple_dist_eq}
    For any pair of unit vectors $\bw_1, \bw_2$ with $\bw_1^T \bw_2 \ge 0$, we have 
    \begin{equation}
        \|\bw_1-\bw_2\|_2^2 \le \|\bw_1\bw_1^T-\bw_2\bw_2^T\|_\rmF^2 \le 2 \|\bw_1-\bw_2\|_2^2.
    \end{equation}
    Moreover, if $\bw_1^T \bw_2 < 0$, then the same holds with $\|\bw_1 - \bw_2\|_2$ replaced by $\|\bw_1 + \bw_2\|_2$.
\end{lemma}
\begin{proof}
    When $\bw_1^T \bw_2 \ge 0$, we have 
    \begin{align}
        \|\bw_1\bw_1^T-\bw_2\bw_2^T\|_\rmF^2& = \mathrm{tr}((\bw_1\bw_1^T-\bw_2\bw_2^T)^T(\bw_1\bw_1^T-\bw_2\bw_2^T)) \\
        &= 2\left(1-(\bw_1^T\bw_2)^2\right) \label{eq:intermediate} \\
        & \ge 2\left(1-\bw_1^T\bw_2\right) \label{eq:intermediat2}  \\
        &= \|\bw_1-\bw_2\|_2^2,
    \end{align}
    where \eqref{eq:intermediate} follows by expanding the product and writing $\mathrm{tr}( \bw_1\bw_1^T \bw_1\bw_1^T ) = \mathrm{tr}( \bw_1^T \bw_1\bw_1^T \bw_1 ) = (\bw_1^T \bw_1)^2 = 1$ and handling the other terms similarly, and \eqref{eq:intermediat2} follows since $\bw_1^T \bw_2 \in (0,1)$.  In addition, we have 
    \begin{align}
        \|\bw_1\bw_1^T-\bw_2\bw_2^T\|_\rmF^2 &= 2\left(1-(\bw_1^T\bw_2)^2\right) \\
        &= 2\left(1-\bw_1^T\bw_2\right) \left(1+ \bw_1^T\bw_2\right) \\
        & \le 4 \left(1-\bw_1^T\bw_2\right) \\
        &= 2\|\bw_1-\bw_2\|_2^2,
    \end{align}
    which gives the desired inequality.  The case $\bw_1^T \bw_2 < 0$ is handled similarly
\end{proof}

\section{Proof of Theorem~\ref{thm:barbx_bxG} (Guarantee on the Global Optimum)}
\label{app:thmGlobally}

Let the singular value decomposition (SVD) of $\bar{\bV}$ be 
\begin{equation}\label{eq:svd_barbv}
 \bar{\bV} = \bar{\bU}\bar{\bD}\bar{\bU}^T,
\end{equation}
where $\bar{\bD} = \mathrm{Diag}([\bar{\lambda}_1,\ldots,\bar{\lambda}_n])$, and $\bar{\bU} \in \bbR^n$ is an orthonormal matrix with the first column being $\bar{\bx}$. For $i>1$, let the $i$-th column of $\bar{\bU}$ be $\bar{\bu}_i$. Then, we have
\begin{align}
 \hat{\bv}^T \bar{\bV} \hat{\bv} & = \bar{\lambda}_1 \left(\bar{\bx}^T\hat{\bv}\right)^2 + \sum_{i>1} \bar{\lambda}_i \left(\bar{\bu}_i^T\hat{\bv}\right)^2 \\
 & \le \bar{\lambda}_1 \left(\bar{\bx}^T\hat{\bv}\right)^2 + \bar{\lambda}_2 \sum_{i>1} \left(\bar{\bu}_i^T\hat{\bv}\right)^2 \\
 & =  \bar{\lambda}_1 \left(\bar{\bx}^T\hat{\bv}\right)^2 + \bar{\lambda}_2  \left(1-\left(\bar{\bx}^T\hat{\bv}\right)^2\right),\label{eq:ub_impr_eq1}
\end{align}
where we use $\left(\bar{\bx}^T\hat{\bv}\right)^2 + \sum_{i >1} \left(\bar{\bu}_i^T\hat{\bv}\right)^2 = 1$ in~\eqref{eq:ub_impr_eq1}. 

In addition, for any $\bA \in \bbR^{n \times n}$ and any $\bs_1,\bs_2 \in \bbR^n$, we have
\begin{align}
 \bs_1^T \bA \bs_1 - \bs_2^T \bA \bs_2 &= \left(\frac{\bs_1 + \bs_2}{2} + \frac{\bs_1 - \bs_2}{2}\right)^T \bA  \left(\frac{\bs_1 + \bs_2}{2} + \frac{\bs_1 - \bs_2}{2}\right) \nonumber\\
 & \quad - \left(\frac{\bs_1 + \bs_2}{2} - \frac{\bs_1 - \bs_2}{2}\right)^T \bA  \left(\frac{\bs_1 + \bs_2}{2} - \frac{\bs_1 - \bs_2}{2}\right) \\
 & = 2 \left(\frac{\bs_1 + \bs_2}{2}\right)^T \bA \left(\frac{\bs_1 - \bs_2}{2}\right) + 2 \left(\frac{\bs_1 - \bs_2}{2}\right)^T \bA \left(\frac{\bs_1 + \bs_2}{2}\right).\label{eq:ub_impr_eq2}
\end{align}
In particular, when $\bA$ is symmetric, we obtain
\begin{equation}
 \bs_1^T \bA \bs_1 - \bs_2^T \bA \bs_2 = (\bs_1+\bs_2)^T \bA(\bs_1 -\bs_2).\label{eq:ub_impr_eq3}
\end{equation}
Let $M$ be a $(\delta/L)$-net of $B_2^k(r)$; from~\cite[Lemma~5.2]{vershynin2010introduction}, we know that there exists such a net with 
    \begin{equation}
        \log |M| \le k \log\frac{4 Lr}{\delta}. \label{eq:net_size}
    \end{equation}
    Since $G$ is $L$-Lipschitz continuous, we have that $G(M)$ is a $\delta$-net of $\mathrm{Range}(G) = G(B_2^k(r))$. We write 
    \begin{equation}\label{eq:hatbs_decomp}
     \hat{\bv} = (\hat{\bv} -\tilde{\bx}) + \tilde{\bx}, 
    \end{equation}
where $\tilde{\bx} \in G(M)$ satisfies $\|\hat{\bv} -\tilde{\bx}\|_2 \le \delta$. Suppose that $\bar{\bx}^T\hat{\bv} \ge 0$; if this is not the case, we can use analogous steps to obtain an upper bound for $\|\bar{\bx} + \hat{\bv}\|_2$ instead of $\|\bar{\bx} - \hat{\bv}\|_2$.  We have
\begin{align}
 &\frac{\bar{\lambda}_1 - \bar{\lambda}_2}{2} \cdot \|\bar{\bx} - \hat{\bv}\|_2^2 \\
 &= (\bar{\lambda}_1 - \bar{\lambda}_2) \left(1- \bar{\bx}^T\hat{\bv}\right) \label{eq:simple1} \\
 & \le (\bar{\lambda}_1 - \bar{\lambda}_2) \left(1- \left(\bar{\bx}^T\hat{\bv}\right)^2\right) \label{eq:simple2} \\
 &  = \bar{\lambda}_1 - \left(\bar{\lambda}_1 \left(\bar{\bx}^T\hat{\bv}\right)^2 + \bar{\lambda}_2 \left(1- \left(\bar{\bx}^T\hat{\bv}\right)^2\right)\right) \\
 & = \bar{\bx}^T\bar{\bV}\bar{\bx} - \left(\bar{\lambda}_1 \left(\bar{\bx}^T\hat{\bv}\right)^2 + \bar{\lambda}_2 \left(1- \left(\bar{\bx}^T\hat{\bv}\right)^2\right)\right) \label{eq:eigen} \\
 & \le \bar{\bx}^T\bar{\bV}\bar{\bx} - \hat{\bv}^T \bar{\bV} \hat{\bv}\label{eq:ub_impr_eq4} \\
 & = \bx_G^T\bar{\bV}\bx_G + (\bar{\bx}+\bx_G)^T\bar{\bV}(\bar{\bx}-\bx_G) - \hat{\bv}^T \bar{\bV} \hat{\bv}\label{eq:ub_impr_eq5} \\
 & \le \bx_G^T\bar{\bV}\bx_G + 2\bar{\lambda}_1 \|\bar{\bx}-\bx_G\|_2 - \hat{\bv}^T \bar{\bV} \hat{\bv} \label{eq:ub_impr_eq5d5}\\
 & = \bx_G^T(\bV-\bE)\bx_G + 2\bar{\lambda}_1 \|\bar{\bx}-\bx_G\|_2 - \hat{\bv}^T (\bV-\bE) \hat{\bv} \label{eq:sub_E} \\
 & \le \hat{\bv}^T\bE\hat{\bv} - \bx_G^T\bE\bx_G + 2\bar{\lambda}_1 \|\bar{\bx}-\bx_G\|_2 \label{eq:ub_impr_eq6} \\
 & = \tilde{\bx}^T\bE\tilde{\bx} + 2 \left(\frac{\hat{\bv} - \tilde{\bx}}{2}\right)^T\bE\left(\frac{\hat{\bv} + \tilde{\bx}}{2}\right) + 2 \left(\frac{\hat{\bv} + \tilde{\bx}}{2}\right)^T\bE\left(\frac{\hat{\bv} - \tilde{\bx}}{2}\right) - \bx_G^T\bE\bx_G + 2\bar{\lambda}_1 \|\bar{\bx}-\bx_G\|_2 \label{eq:ub_impr_eq7} \\
 & \le \tilde{\bx}^T\bE\tilde{\bx} + 2\delta \|\bE\|_{2\to 2}- \bx_G^T\bE\bx_G + 2\bar{\lambda}_1 \|\bar{\bx}-\bx_G\|_2 \label{eq:ub_impr_eq8}\\
 & = 2\left(\frac{\tilde{\bx}+\bx_G}{2}\right)^T \bE\left(\frac{\tilde{\bx}-\bx_G}{2}\right) + 2\left(\frac{\tilde{\bx}-\bx_G}{2}\right)^T \bE\left(\frac{\tilde{\bx}+ \bx_G}{2}\right) + 2\delta \|\bE\|_{2\to 2} + 2\bar{\lambda}_1 \|\bar{\bx}-\bx_G\|_2 \label{eq:ub_impr_eq9} \\
 & \le 2C\sqrt{\frac{k \log \frac{4Lr}{\delta}}{m}} \cdot \|\tilde{\bx}-\bx_G\|_2 + 2\delta \|\bE\|_{2\to 2} + 2\bar{\lambda}_1 \|\bar{\bx}-\bx_G\|_2 \label{eq:ub_impr_eq10}\\
 & \le 2C\sqrt{\frac{k \log \frac{4Lr}{\delta}}{m}} \cdot (\|\tilde{\bx}-\hat{\bv}\|_2 + \|\hat{\bv} -\bar{\bx}\|_2 + \|\bar{\bx}-\bx_G\|_2) + 2\delta \|\bE\|_{2\to 2} + 2\bar{\lambda}_1 \|\bar{\bx}-\bx_G\|_2 \label{eq:triangle}\\
 & \le 2C\sqrt{\frac{k \log \frac{4Lr}{\delta}}{m}} \cdot \|\hat{\bv} -\bar{\bx}\|_2 + O\left(\frac{\delta n}{m}\right) + O\big( (\bar{\lambda}_1 + \epsilon_n) \|\bar{\bx}-\bx_G\|_2\big)\label{eq:ub_impr_eq11}, 
 \end{align}
where:
\begin{itemize}
    \item \eqref{eq:simple1}--\eqref{eq:simple2} follow from $\|\bar{\bx}\|_2 = \|\hat{\bv}\|_2 = 1$ and hence $|\bar{\bx}^T\hat{\bv}| \le 1$;
    \item \eqref{eq:eigen} follows since $(\bar{\lambda}_1,\bar{\bx})$ are an eigenvalue-eigenvector pair for $\bar{\bV}$ with $\|\bar{\bx}\|_2 = 1$;
    \item \eqref{eq:ub_impr_eq4} follows from~\eqref{eq:ub_impr_eq1};
    \item \eqref{eq:ub_impr_eq5} follows from~\eqref{eq:ub_impr_eq3} with $\bar{\bV}$ being symmetric and setting $\bs_1 = \bar{\bx}$, $\bs_2 = \bx_G$;
    \item \eqref{eq:ub_impr_eq5d5} follows from $\|\bar{\bx} + \bx_G\|_2\le 2$ and $\|\bar{\bV}\|_{2\to 2} =\bar{\lambda}_1$;
    \item \eqref{eq:sub_E} follows since $\bar{\bV} = \bV - \bE$;
    \item \eqref{eq:ub_impr_eq6} follows since $\hat{\bv}$ is a globally optimal solution to~\eqref{eq:opt_genPCA} and $\bx_G \in \mathrm{Range}(G)$;
    \item \eqref{eq:ub_impr_eq7} follows from~\eqref{eq:ub_impr_eq2} with $\bs_1 = \hat{\bv}$ and $\bs_2 = \tilde{\bx}$;
    \item \eqref{eq:ub_impr_eq8} follows from~\eqref{eq:hatbs_decomp} along with $\|\hat{\bv} - \bar{\bx}\|_2 \le \delta$ and $\|\hat{\bv}+\tilde{\bx}\|_2\le 2$;
    \item \eqref{eq:ub_impr_eq9} follows from~\eqref{eq:ub_impr_eq2};
    \item \eqref{eq:ub_impr_eq10} follows from Assumption~\ref{assump:E_cond} (with $S_1 = S_2$ being $G(M)$ shifted by $\bx_G$) and~\eqref{eq:net_size};
    \item \eqref{eq:triangle} follows from the triangle inequality;
    \item \eqref{eq:ub_impr_eq11} follows by substituting $\|\hat{\bv} - \tilde{\bx}\|_2 \le \delta$, along with the assumptions $\|\bE\|_{2\to 2} = O(n/m)$, $m = \Omega\big(k \log\frac{Lr}{\delta}\big)$, and $\epsilon_n = O\big( \sqrt{\frac{k \log\frac{Lr}{\delta}}{m}} \big)$.
\end{itemize}
From~\eqref{eq:ub_impr_eq11}, we have the following when $\hat{\bv}^T\bar{\bx} \ge 0 $:
\begin{equation}
 \|\hat{\bv} -\bar{\bx}\|_2 = \frac{O\left(\sqrt{\frac{k \log\frac{Lr}{\delta}}{m}}\right)}{\bar{\lambda}_1 - \bar{\lambda}_2} + O\left(\sqrt{\frac{\delta n/m}{\bar{\lambda}_1 - \bar{\lambda}_2}}\right) + O\left(\sqrt{\frac{(\bar{\lambda}_1 + \epsilon_n) \|\bar{\bx}-\bx_G\|_2}{ \bar{\lambda}_1 - \bar{\lambda}_2}}\right).\label{eq:ub_impr_eq12}
\end{equation}
As mentioned earlier, if $\hat{\bv}^T\bar{\bx} <0$, we have the same upper bound as in~\eqref{eq:ub_impr_eq12} for $\|\hat{\bv} +\bar{\bx}\|_2$. Therefore, we obtain
\begin{align}
 \|\hat{\bv}\hat{\bv}^T - \bar{\bx}\bar{\bx}^T\|_\rmF &= \sqrt{2\left(1-\left(\bar{\bx}^T\hat{\bv}\right)^2\right)} \label{eq:sqrt_expr} \\
 & = \sqrt{2\left(1-\bar{\bx}^T\hat{\bv}\right)\left(1+\bar{\bx}^T\hat{\bv}\right)} \\
 & \le \sqrt{2} \min \{\|\hat{\bv}-\bar{\bx}\|_2,\|\hat{\bv}+\bar{\bx}\|_2\} \label{eq:finalM1}\\
 & = \frac{O\left(\sqrt{\frac{k \log\frac{Lr}{\delta}}{m}}\right)}{\bar{\lambda}_1 - \bar{\lambda}_2} + O\left(\sqrt{\frac{\delta n/m}{\bar{\lambda}_1 - \bar{\lambda}_2}}\right) + O\left(\sqrt{\frac{(\bar{\lambda}_1 + \epsilon_n) \|\bar{\bx}-\bx_G\|_2}{\bar{\lambda}_1 - \bar{\lambda}_2}}\right), \label{eq:final}
\end{align}
where \eqref{eq:sqrt_expr} follows from \eqref{eq:intermediate},~\eqref{eq:finalM1} follows from $\|\hat{\bv}\pm\bar{\bx}\|_2^2 = 2(1\pm\bar{\bx}^T\hat{\bv})$, and \eqref{eq:final} follows from \eqref{eq:ub_impr_eq12}.

\section{Proof of Lemma~\ref{lem:main_lem} (Auxiliary Result for PPower Analysis)}
\label{app:main_lem}

By the assumption $\mathrm{Range}(G)\subseteq \calS^{n-1}$, for any $\bx \in \bbR^n$ and $a>0$, we have 
 \begin{equation}\label{eq:simple_proj_property}
  \calP_G(\bx) = \calP_G(a\bx),
 \end{equation}
which is seen by noting that when comparing $\|\bx - \ba\|_2$ with $\|\bx - \bb\|_2$ (in accordance with projection mapping to the closest point), as long as $\|\ba\|_2 = \|\bb\|_2$, the comparison reduces to comparing $\langle \bx,\ba \rangle$ with $\langle \bx,\bb \rangle$, so is invariant to positive scaling of $\bx$.

 Let $\bar{\eta} = 1/\bar{\lambda}_1 >0$ and $\hat{\bs}=\calP_G(\bV\bs)$. Then, we have $\hat{\bs} = \calP_G(\bV\bs) = \calP_G(\bar{\eta}\bV\bs)$. Since $\bar{\bx} \in \mathrm{Range}(G)$, we have 
\begin{equation}
 \|\bar{\eta}\bV\bs - \hat{\bs}\|_2 \le \|\bar{\eta}\bV\bs - \bar{\bx}\|_2.
\end{equation}
This is equivalent to
\begin{equation}
 \|(\bar{\eta}\bV\bs -\bar{\bx}) + (\bar{\bx}- \hat{\bs})\|_2^2 \le \|\bar{\eta}\bV\bs - \bar{\bx}\|_2^2,
\end{equation}
and expanding the square gives
\begin{equation}\label{eq:imp_ineq_main}
 \|\bar{\bx}- \hat{\bs}\|_2^2 \le 2\langle \bar{\eta}\bV\bs -\bar{\bx}, \hat{\bs} - \bar{\bx} \rangle.
\end{equation}
Note also that from $\bar{\bV}\bar{\bx} = \bar{\lambda}_1\bar{\bx}$, we obtain $\bar{\bx} = \bar{\eta}\bar{\bV}\bar{\bx}$, which we will use throughout the proof. 

For $\delta >0$, let $M$ be a $(\delta/L)$-net of $B_2^k(r)$; from Lemma~\ref{lem:unit_sphere_cov}, there exists such a net with 
    \begin{equation}
        \log |M| \le k \log\frac{4 Lr}{\delta}. \label{eq:net_size_main_lem}
    \end{equation}
    By the $L$-Lipschitz continuity of $G$, we have that $G(M)$ is a $\delta$-net of $\mathrm{Range}(G) = G(B_2^k(r))$. We write 
    \begin{equation}\label{eq:bshatbs_decomp}
     \bs = (\bs -\bs_0) + \bs_0, \quad \hat{\bs} = (\hat{\bs} -\tilde{\bs}) + \tilde{\bs}, 
    \end{equation}
where $\tilde{\bs} \in G(M)$ satisfies $\|\hat{\bs} -\tilde{\bs}\|_2 \le \delta$, and $\bs_0 \in G(M)$ satisfies $\|\bs -\bs_0\|_2 \le \delta$. Then, we have 
\begin{align}
 &\langle \bar{\eta}\bV\bs -\bar{\bx}, \hat{\bs} - \bar{\bx} \rangle = \langle \bar{\eta} \bar{\bV} (\bs -\bar{\bx}), \hat{\bs} - \bar{\bx} \rangle  + \langle \bar{\eta}\bE\bs, \hat{\bs} - \bar{\bx} \rangle \label{eq:main_eq1},
\end{align}
which follows from $\bV = \bar{\bV} + \bE$ and $\bar{\bx} = \bar{\eta}\bar{\bV}\bar{\bx}$. 
In the following, we control the two terms in~\eqref{eq:main_eq1} separately. 
\begin{enumerate}
 \item The term $\langle \bar{\eta} \bar{\bV} (\bs -\bar{\bx}), \hat{\bs} - \bar{\bx} \rangle$: We decompose $\bs = \alpha \bar{\bx} + \beta \bt$ and $\hat{\bs} =  \hat{\alpha} \bar{\bx} + \hat{\beta} \hat{\bt}$, where $\|\bt\|_2 = \|\hat{\bt}\|_2 = 1$ and $\bt^T \bar{\bx} = \hat{\bt}^T \bar{\bx} = 0$. Since $\|\bs\|_2 = \|\hat{\bs}\|_2 =1$, we have $\alpha^2 +\beta^2 = \hat{\alpha}^2 + \hat{\beta}^2=1$. In addition, we have $\alpha = \bs^T\bar{\bx}$ and $\hat{\alpha} = \hat{\bs}^T\bar{\bx}$. Recall that in~\eqref{eq:svd_barbv}, we write the SVD of $\bar{\bV}$ as $\bar{\bV} = \bar{\bU}\bar{\bD}\bar{\bU}^T$. Since $\bt^T\bar{\bx} = 0$, we can write $\bt$ as $\bt = \sum_{i>1} h_i \bar{\bu}_i$. In addition, since $\|\bt\|_2 = 1$, we have $\sum_{i>1} h_i^2 = 1$. Hence, by the Cauchy-Schwarz inequality, we have
 \begin{align}
  |\langle \bar{\bV} \bt, \hat{\bt}\rangle| & \le \|\bar{\bV} \bt\|_2 \\
  & = \left\|\sum_{i>1} \bar{\lambda}_i h_i \bar{\bu}_i\right\|_2 \\
  & = \sqrt{\sum_{i>1} \bar{\lambda}_i^2 h_i^2} \\
  & \le \sqrt{\bar{\lambda}_2^2 \sum_{i>1}  h_i^2} \\ &= \bar{\lambda}_2.\label{eq:bt_barbV_bt}
 \end{align}
 Therefore, we obtain
 \begin{align}
  |\langle \bar{\eta} \bar{\bV} (\bs -\bar{\bx}), \hat{\bs} - \bar{\bx} \rangle| & = |\langle (\alpha-1)\bar{\bx} + \bar{\eta} \beta \bar{\bV} \bt, (\hat{\alpha}-1)\bar{\bx} + \hat{\beta} \hat{\bt} \rangle|  \label{eq:innerprod1} \\
  & = |(\alpha-1)(\hat{\alpha}-1) +  \bar{\eta} \beta\hat{\beta} \langle \bar{\bV} \bt, \hat{\bt}\rangle|  \label{eq:innerprod2}  \\
  & \le (1-\alpha)(1-\hat{\alpha}) + \bar{\eta} |\beta \hat{\beta}| \bar{\lambda}_2 \\
  & = (1-\alpha)(1-\hat{\alpha}) + \bar{\gamma} \sqrt{1-\alpha^2}\sqrt{1-\hat{\alpha}^2}, \label{eq:main_lem_fristTerm_bd}
 \end{align}
where \eqref{eq:innerprod1} uses $\eta \bar{\bV}\bar{\bx} = \bar{\bx}$, and \eqref{eq:innerprod2} uses $\|\bar{\bx}\|_2 = 1$ and $\langle \bar{\bx},\bt \rangle = 0$.
 
 \item The term $\langle \bar{\eta}\bE\bs, \hat{\bs} - \bar{\bx} \rangle$: We have 
 \begin{align}
  |\langle \bar{\eta}\bE\bs, \hat{\bs} - \bar{\bx} \rangle| &= \langle \bar{\eta}\bE((\bs-\bs_0) +\bs_0), \hat{\bs} - \bar{\bx} \rangle \\
  & = \langle \bar{\eta}\bE(\bs-\bs_0) , \hat{\bs} - \bar{\bx} \rangle + \langle \bar{\eta}\bE\bs_0, (\hat{\bs} -\tilde{\bs}) +(\tilde{\bs} - \bar{\bx}) \rangle \\
  & \le \bar{\eta} \|\bE\|_{2\to 2} \delta \|\hat{\bs} - \bar{\bx}\|_2 + \bar{\eta} \|\bE\|_{2\to 2} \delta + O\left(\sqrt{\frac{k\log\frac{Lr}{\delta}}{m}}\right) \cdot \|\tilde{\bs} - \bar{\bx}\|_2 \label{eq:main_lem_E_bd}\\
  & \le O\left(\delta \|\bE\|_{2\to 2}\right) +O\left(\sqrt{\frac{k\log\frac{Lr}{\delta}}{m}}\right) \cdot \|\hat{\bs} - \bar{\bx}\|_2,\label{eq:main_lem_compare_bd}
 \end{align}
where~\eqref{eq:main_lem_E_bd} follows from Assumption~\ref{assump:E_cond} and~\eqref{eq:net_size_main_lem}, and~\eqref{eq:main_lem_compare_bd} follows from $\bar{\eta} = 1/\bar{\lambda}_1$, along with the fact that we assumed $\bar{\lambda}_1 = \Theta(1)$. 
\end{enumerate}
Note that $\|\bar{\bx} - \hat{\bs}\|_2^2 = 2(1-\hat{\bs}^T\bar{\bx}) = 2(1-\hat{\alpha})$. Hence, and using~\eqref{eq:imp_ineq_main},~\eqref{eq:main_eq1},~\eqref{eq:main_lem_fristTerm_bd}, and~\eqref{eq:main_lem_compare_bd}, we obtain
\begin{align}
 2(1-\hat{\alpha})& \le 2\left((1-\alpha)(1-\hat{\alpha}) + \bar{\gamma} \sqrt{1-\alpha^2}\sqrt{1-\hat{\alpha}^2}\right) \nonumber \\
 & \indent + O\left(\delta \|\bE\|_{2\to 2}\right) + O\left(\sqrt{\frac{k\log\frac{Lr}{\delta}}{m}} \right)\cdot \sqrt{2(1-\hat{\alpha})}.\label{eq:combining_main_ineq}
\end{align}
 Using $2(1-\hat{\alpha}) - 2(1-\alpha)(1-\hat{\alpha}) = 2\alpha(1-\hat{\alpha})$, $\sqrt{1-\alpha^2} = \sqrt{1-\alpha}\sqrt{1+\alpha} \le \sqrt{2(1-\alpha)}$, and similarly $\sqrt{1-\hat{\alpha}^2} \le \sqrt{2(1-\hat{\alpha})}$, we obtain from~\eqref{eq:combining_main_ineq} that
\begin{equation}
 2\alpha(1-\hat{\alpha}) \le 2\bar{\gamma} \sqrt{2(1-\alpha)}\sqrt{2(1-\hat{\alpha})}  + O\left( \sqrt{\frac{k\log\frac{Lr}{\delta}}{m}} \right)\cdot \sqrt{2(1-\hat{\alpha})} + O\left(\delta \|\bE\|_{2\to 2}\right).\label{eq:combining_main_mod}
\end{equation}
Since $\|\hat{\bs} - \bar{\bx}\|_2^2  = 2(1-\hat{\alpha})$ and $\|\bs - \bar{\bx}\|_2^2  = 2(1-\alpha)$, this is equivalent to
\begin{align}
 \alpha \|\hat{\bs} -\bar{\bx}\|_2^2 &\le
 \left(2\bar{\gamma} \|\bs - \bar{\bx}\|_2 + O\left( \sqrt{\frac{k\log\frac{Lr}{\delta}}{m}} \right)\right) \|\hat{\bs} -\bar{\bx}\|_2 +  O\left(\delta \|\bE\|_{2\to 2}\right).\label{eq:imp_final-1_bd}
\end{align}
This equation is of the form $az^2 \le bz + c$ (where $z = \|\hat{\bs} -\bar{\bx}\|_2$ and $a = \alpha = \bs^T\bar{\bx} >0$), and using a simple application of the quadratic formula,\footnote{Since the leading coefficient $a = \alpha$ of the quadratic is positive, $z$ must lie in between the two associated roots.  This yields $z \le \frac{b + \sqrt{b^2 + 4ac}}{2a}$, from which the inequality $\sqrt{a+b} \le \sqrt{a} + \sqrt{b}$ gives \eqref{eq:imp_final_thm_bd}.} we obtain
\begin{align}
 \|\hat{\bs}-\bar{\bx}\|_2 &\le \frac{2\bar{\gamma}\|\bs-\bar{\bx}\|_2}{\bs^T\bar{\bx}} + O\left(\frac{1}{\bs^T\bar{\bx}}\left(\sqrt{\frac{k\log\frac{Lr}{\delta}}{m}}+ \sqrt{\bs^T\bar{\bx} \cdot \delta \|\bE\|_{2\to 2}}\right)\right) \label{eq:imp_final_thm_bd}\\
 & \le \frac{2\bar{\gamma}\|\bs-\bar{\bx}\|_2}{\bs^T\bar{\bx}} + O\left(\frac{1}{\bs^T\bar{\bx}}\sqrt{\frac{k\log (nLr)}{m}}\right), \label{eq:imp_finalP1_thm_bd}
\end{align}
where we use the assumption $\|\bE\|_{2\to 2} = O(n/m)$ and set $\delta = 1/n$ in~\eqref{eq:imp_finalP1_thm_bd}.

\section{Proof of Theorem~\ref{thm:main} (Main Theorem for PPower)}
\label{app:thm_main}

Suppose for the time being that \eqref{eq:either1} holds for at least one index $t \ge t_0$ (we will later verify that this is the case), and let $T_0 \ge t_0$ be the smallest such index.  Thus, we have
\begin{equation}
 \|\bw^{(T_0)} - \bar{\bx}\|_2  \le \frac{C}{(1-\mu_0)\nu} \sqrt{\frac{k\log (nLr)}{m}}.  \label{eq:T0_init}
\end{equation}
Note that according to the theorem statement, $1 - \mu_0$ is bounded away from zero.  Using $\|\bw^{(T_0)}\|_2 = \|\bar{\bx}\|_2 = 1$ and the assumption that $m \ge C_{\nu,\tau} \cdot k \log (nLr)$ with $C_{\nu,\tau}>0$ being large enough, we deduce from \eqref{eq:T0_init} that $\|\bw^{(T_0)} - \bar{\bx}\|_2$ is sufficiently small such that
\begin{equation}
 \bar{\bx}^T\bw^{(T_0)} \ge 1-\tau. \label{eq:T0_tau}
\end{equation}
Next, using the assumption $2\bar{\gamma} + \nu \le 1-\tau$, we write
\begin{equation}
 \frac{2\bar{\gamma}}{(1-\mu_0)\nu(1-\tau)}  + \frac{1}{1-\tau}  = \frac{2\bar{\gamma} + (1-\mu_0)\nu}{(1-\mu_0)\nu(1-\tau)} \le \frac{2\bar{\gamma} + \nu}{(1-\mu_0)\nu(1-\tau)} \le \frac{1}{(1-\mu_0)\nu}.\label{eq:T0_final0}
\end{equation}
Then, from Lemma~\ref{lem:main_lem}, we obtain
\begin{align}
 \|\bw^{(T_0+1)} - \bar{\bx}\|_2 &\le \frac{2\bar{\gamma}}{\bar{\bx}^T\bw^{(T_0)}} \cdot \|\bw^{(T_0)} - \bar{\bx}\|_2 + \frac{C}{\bar{\bx}^T\bw^{(T_0)}} \sqrt{\frac{k\log (nLr)}{m}}\\
 & \le \frac{2\bar{\gamma}}{1-\tau} \cdot \frac{C}{(1-\mu_0)\nu} \sqrt{\frac{k\log (nLr)}{m}} + \frac{C}{1-\tau} \sqrt{\frac{k\log (nLr)}{m}} \label{eq:T0_pre_final} \\
 & \le \frac{C}{(1-\mu_0)\nu} \sqrt{\frac{k\log (nLr)}{m}},\label{eq:T0_final}
\end{align}
where~\eqref{eq:T0_pre_final} follows from~\eqref{eq:T0_init}--\eqref{eq:T0_tau}, and~\eqref{eq:T0_final} follows from~\eqref{eq:T0_final0}.  Thus, we have transferred \eqref{eq:T0_init} from $T_0$ to $T_0 + 1$, and proceeding by induction, we obtain
\begin{equation}
 \|\bw^{(t)} - \bar{\bx}\|_2  \le \frac{C}{(1-\mu_0)\nu} \sqrt{\frac{k\log (nLr)}{m}}
\end{equation} 
for al $t \ge T_0$.  

Next, we consider $t \in [t_0,T_0)$. Again using Lemma~\ref{lem:main_lem} (with $\hat{\bs} = \bw^{(t_0+1)} = \calP_G(\bV\bw^{(t_0)})$), we have 
 \begin{equation}\label{eq:thm_imp_ineq}
  \|\bw^{(t_0+1)} - \bar{\bx}\|_2 \le \mu_0 \|\bw^{(t_0)} - \bar{\bx}\|_2 + \frac{C}{2\bar{\gamma}+\nu} \cdot \sqrt{\frac{k\log (nLr)}{m}},
 \end{equation}
where we recall that $\mu_0 = \frac{2\bar{\gamma}}{\bar{\bx}^T\bw^{(t_0)}} = \frac{2\bar{\gamma}}{2\bar{\gamma} + \nu} <1$, and note that the denominator in the second term of \eqref{eq:thm_imp_ineq} follows since $\bar{\bx}^T\bw^{(t_0)} = 2\bar{\gamma} + \nu $.  Supposing that $t_0 < T_0$ (otherwise, the above analysis for $t \ge T_0$ alone is sufficient), we have that \eqref{eq:either1} is reversed at $t=t_0$:
\begin{equation}
 \|\bw^{(t_0)} - \bar{\bx}\|_2  > \frac{C}{(1-\mu_0)\nu} \sqrt{\frac{k\log (nLr)}{m}}. \label{eq:hold}
\end{equation}
This means that we can upper bound the second term in~\eqref{eq:thm_imp_ineq} by $\frac{C}{\nu} \cdot \sqrt{\frac{k\log (nLr)}{m}} < (1-\mu_0) \|\bw^{(t_0)} - \bar{\bx}\|_2$, which gives
\begin{equation}
    \|\bw^{(t_0+1)} - \bar{\bx}\|_2 < \|\bw^{(t_0)} - \bar{\bx}\|_2. \label{eq:established}
\end{equation}
Squaring both sides, expanding, and canceling the norms (which all equal one), we obtain
\begin{equation}
 \bar{\bx}^T\bw^{(t_0 +1)} > \bar{\bx}^T\bw^{(t_0)}, \label{eq:est_consequence}
\end{equation}
and by induction, we obtain that $\{\bar{\bx}^T\bw^{(t)}\}_{t \in [t_0,T_0)}$ is monotonically increasing. 

Recall that we assume that $\bar{\lambda}_1 = \Theta(1)$, and that $T_0 = t_0 + \Delta_0$ is the smallest integer such that~\eqref{eq:either1} holds. To verify that $T_0$ is finite and upper bound $\Delta_0$, we consider the following three cases:
\begin{itemize} 
 \item \underline{$\mu_0 = 0$ (or equivalently, $\bar{\gamma} =\bar{\lambda}_2 =0$)}: In this case,~\eqref{eq:thm_imp_ineq} gives $T_0 = t_0 + 1$ (or $T_0 = t_0$, which we already addressed above).  Thus, we have $\Delta_0 \le 1$, as stated in the theorem.
 
 \item \underline{$\mu_0 = o(1)$ (or equivalently, $\bar{\gamma} = o(1)$ and $\bar{\lambda}_2 = o(1)$)}: Since $\{\bar{\bx}^T\bw^{(t)}\}_{t \in [t_0,T_0)}$ is monotonically increasing, for any positive integer $\Delta$ with $t_0 + \Delta \le T_0$, by applying Lemma~\ref{lem:main_lem} (or~\eqref{eq:thm_imp_ineq}) multiple times, we obtain\footnote{In simpler notation, if $z_{t+1} \le a z_t + b$, then we get $z_{t+2} \le a^2 z_t + (1+a)b$, then $z_{t+3} \le a^3 z_t + (1+a+a^2)b$, and so on, and then we can apply $1+a+\dotsc+a^{i-1} = \frac{1-a^i}{1-a}$ for $a \ne 1$.}
 \begin{align}
  \|\bw^{(t_0+\Delta)} - \bar{\bx}\|_2 &\le \mu_0^{\Delta} \|\bw^{(t_0)} - \bar{\bx}\|_2 + \frac{1-\mu_0^{\Delta}}{1-\mu_0}\cdot\frac{C}{2\bar{\gamma}+\nu} \cdot \sqrt{\frac{k\log (nLr)}{m}} \label{eq:thm_imp_ineq_multi1} \\
  & \le \mu_0^{\Delta} \|\bw^{(t_0)} - \bar{\bx}\|_2 + \frac{1}{1-\mu_0}\cdot\frac{C}{2\bar{\gamma}+\nu} \cdot \sqrt{\frac{k\log (nLr)}{m}} \label{eq:thm_imp_ineq_multi2}.
 \end{align}
 Then, if we choose $\Delta_0 \in \bbN$ such that 
 \begin{equation}\label{eq:p0}
  \mu_0^{\Delta_0 -1} \le \frac{C}{2\nu}  \cdot \sqrt{\frac{k\log (nLr)}{m}},
 \end{equation}
we obtain from~\eqref{eq:thm_imp_ineq_multi2} that
\begin{align}
 \|\bw^{(t_0+\Delta_0)} - \bar{\bx}\|_2  & \le \mu_0^{\Delta_0} \|\bw^{(t_0)} - \bar{\bx}\|_2 + \frac{1}{1-\mu_0}\cdot\frac{C}{2\bar{\gamma}+\nu} \cdot \sqrt{\frac{k\log (nLr)}{m}} \\
 & \le 2\mu_0^{\Delta_0} + \frac{1}{1-\mu_0}\cdot\frac{C}{2\bar{\gamma}+\nu} \cdot \sqrt{\frac{k\log (nLr)}{m}} \label{eq:thm_imp_ineq_multi3}\\
 & < \frac{C\mu_0}{\nu(1-\mu_0)}\cdot \sqrt{\frac{k\log (nLr)}{m}} + \frac{1}{1-\mu_0}\cdot\frac{C}{2\bar{\gamma}+\nu} \cdot \sqrt{\frac{k\log (nLr)}{m}} \label{eq:thm_imp_ineq_multi4} \\
 & = \frac{C}{(1-\mu_0)\nu} \cdot \sqrt{\frac{k\log (nLr)}{m}}, \label{eq:thm_imp_ineq_multi5}
\end{align}
where~\eqref{eq:thm_imp_ineq_multi3} follows from $\|\bw^{(t_0)} - \bar{\bx}\|_2 \le 2$,~\eqref{eq:thm_imp_ineq_multi4} follows from~\eqref{eq:p0} and $1 < \frac{1}{1-\mu_0}$, and~\eqref{eq:thm_imp_ineq_multi5} follows from $\mu_0 = \frac{2\bar{\gamma}}{2\bar{\gamma} + \nu}$, which implies $\frac{\mu_0}{\nu} + \frac{1}{2\bar{\gamma}+\nu} = \frac{\mu_0 (2\bar{\gamma}+\nu) + \nu}{\nu(2\bar{\gamma}+\nu)} = \frac{2\bar{\gamma}+\nu}{\nu(2\bar{\gamma}+\nu)}= \frac{1}{\nu}$.  Observe that \eqref{eq:thm_imp_ineq_multi5} coincides with \eqref{eq:either1}, and since $\mu_0 = o(1)$, we obtain from \eqref{eq:p0} that $\Delta_0 = O\big(\log\big(\frac{m}{k\log(nLr)}\big)\big)$ as desired. 

\item \underline{$\mu_0 = \Theta(1)$ (or equivalently, $\bar{\gamma} = \Theta(1)$ and $\bar{\lambda}_2 = \Theta(1)$)}: Recall that we only need to focus on the case $T_0 >t_0$. This means that~\eqref{eq:hold} holds, implying that we can upper bound the second term in~\eqref{eq:thm_imp_ineq} by $\frac{(1-\mu_0)\nu}{2\bar{\gamma}+\nu} \cdot \|\bw^{(t_0)} - \bar{\bx}\|_2 $, yielding
\begin{align}
    \|\bw^{(t_0+1)} - \bar{\bx}\|_2 & <  \mu_0  \|\bw^{(t_0)} - \bar{\bx}\|_2 + \frac{(1-\mu_0)\nu}{2\bar{\gamma}+\nu} \cdot \|\bw^{(t_0)} - \bar{\bx}\|_2 \\
    & = \frac{2\bar{\gamma} + (1-\mu_0)\nu}{2\bar{\gamma} + \nu} \cdot \|\bw^{(t_0)} - \bar{\bx}\|_2 \\
    & = (1-\xi)\|\bw^{(t_0)} - \bar{\bx}\|_2, \label{eq:established2}
\end{align}
where $\xi = \frac{\mu_0 \nu}{2\bar{\gamma}+\nu} = \mu_0 (1-\mu_0) = \Theta(1)$.  With the distance to $\bar{\bx}$ shrinking by a constant factor in each iteration according to \eqref{eq:established2}, and the initial distance $\|\bw^{(t_0)} - \bar{\bx}\|_2$ being at most $2$ due to the vectors having unit norm, we deduce that $\Delta_0 = O\big(\log\big(\frac{m}{k\log(nLr)}\big)\big)$ iterations suffice to ensure that~\eqref{eq:either1} holds for $t= t_0 + \Delta_0$.
\end{itemize}   

\section{Comparison of Analysis to \cite{deshpande2014cone}}
\label{app:diff_proof_techs}

As mentioned in Section \ref{sec:main}, our analysis is significantly different from that of \cite{deshpande2014cone} despite using a similar assumption on the initialization.  We highlight the differences as follows:
\begin{enumerate} 
 \item Perhaps the most significant difference is that the proof of~\cite[Theorem~3]{deshpande2014cone} is highly dependent on the Moreau decomposition, which is only valid for a closed convex cone (see~\cite[Definition~1.2]{deshpande2014cone}). In particular, the Moreau decomposition needs to be used at the beginning of the proof of~\cite[Theorem~3]{deshpande2014cone}, such as Eqs.~(18) and (19) in the supplementary material therein. We do not see a way for the proof to proceed without the Moreau decomposition, and our $\mathrm{Range}(G)$ may be very different from a convex cone.
 
 \item We highlight that one key observation in our proof of Lemma~\ref{lem:main_lem} (and thus Theorem~\ref{thm:main}) is that for a generative model $G$ with ${\rm Range}(G) \subseteq \calS^{n-1}$, and any $\bx \in \bbR^n$ and $a >0$, we have $\calP_G(a \bx) = \calP_G(\bx)$ (Eq.~\eqref{eq:simple_proj_property}). This enables us to derive the important equation $\hat{\bs} = \calP_G(\bV \bs) = \calP_G(\bar{\eta} \bV \bs)$. We are not aware of a similar idea being used in the proof of~\cite[Theorem~3]{deshpande2014cone}. 
 
 \item In the PPower method in~\cite{deshpande2014cone}, the authors need to add $\rho \bI_n$ with $\rho >0$ to the observed data matrix $\bV$ to improve the convergence. In particular, they mention in the paragraph before the statement of Theorem 3 that ``the memory term $\rho \bv^{t}$ is necessary for our proof technique to go through''. In contrast, our proof of Theorem~\ref{thm:main} does not require adding such terms, even when our data model is restricted to the spiked Wigner model considered in~\cite{deshpande2014cone}.
 
 \item We consider a matrix model that is significantly more general than the spiked Wigner model studied in~\cite{deshpande2014cone}.
\end{enumerate}

\section{Numerical Results for the Fashion-MNIST Dataset}
\label{app:exp_fashionMNIST}

The Fashion-MNIST dataset consists of Zalando's article images with a training set of $60,000$ examples and a test set of $10,000$ examples. The size of each image in the Fashion-MNIST dataset is also $28 \times 28$, and thus $n = 784$. 

The generative model $G$ is set to be a variational autoencoder (VAE). The VAE architecture is summarized as follows.\footnote{Further details of the architecture can be found at~\url{https://github.com/hwalsuklee/tensorflow-generative-model-collections}.} The generator has latent dimension $k = 62$ and four layers. The first two are fully connected layers with the architecture $62-1024-6272$, and with ReLU activation functions. The output of the second layer, reshaped to $128 \times 7 \times 7$, is forwarded to a deconvolution layer with kernel size $4$ and stride $2$. The third layer uses ReLU activations and has output size $64 \times 14\times 14$, where $64$ is the number of channels. The fourth layer is a deconvolution layer with kernel size $4$ and strides $2$, and it uses ReLU activations and has output size $1 \times 28 \times 28$, where the number of channels is $1$.  

The VAE is trained with a mini-batch size of $256$, a learning rate of $0.0002$, and $100$ epochs.  The remaining parameters are the same as those for the MNIST dataset. 

For most images in the Fashion-MNIST dataset, the corresponding vectors are not sparse in the natural basis. Hence, to make a fairer comparison to the sparsity-based method~\texttt{TPower}, we convert the original images to the wavelet basis, and perform~\texttt{TPower} on these converted images. The obtained results of~\texttt{TPower} are then converted back to the vectors in the natural basis. The corresponding method is denoted by~\texttt{TPowerW}, with ``W'' referring to the conversion to images in the wavelet basis.

We perform two sets of experiments, considering the spiked covariance and phase retrieval models separately. The corresponding results are reported in Figures~\ref{fig:fashionMNIST_imgs_spikedCov},~\ref{fig:quant_fashionMNIST_spikedCov},~\ref{fig:fashionMNIST_imgs_pr}, and~\ref{fig:quant_fashionMNIST_pr}. From these figures, again, we can observe clear superiority of~\texttt{PPower} compared to the baselines.  As mentioned above,most of the images are not sparse in the natural basis, but we observe that (i) even for the sparsest images in Figures~\ref{fig:fashionMNIST_imgs_spikedCov} and~\ref{fig:fashionMNIST_imgs_pr} (namely, those containing sandals), \texttt{PPower} significantly outperforms \texttt{TPower}, and (ii) generally, moving from \texttt{TPower} to \texttt{TPowerW} appears to provide at most marginal benefit.

 \begin{figure}
\begin{center}
\begin{tabular}{cc}
 \includegraphics[height=0.19\textwidth]{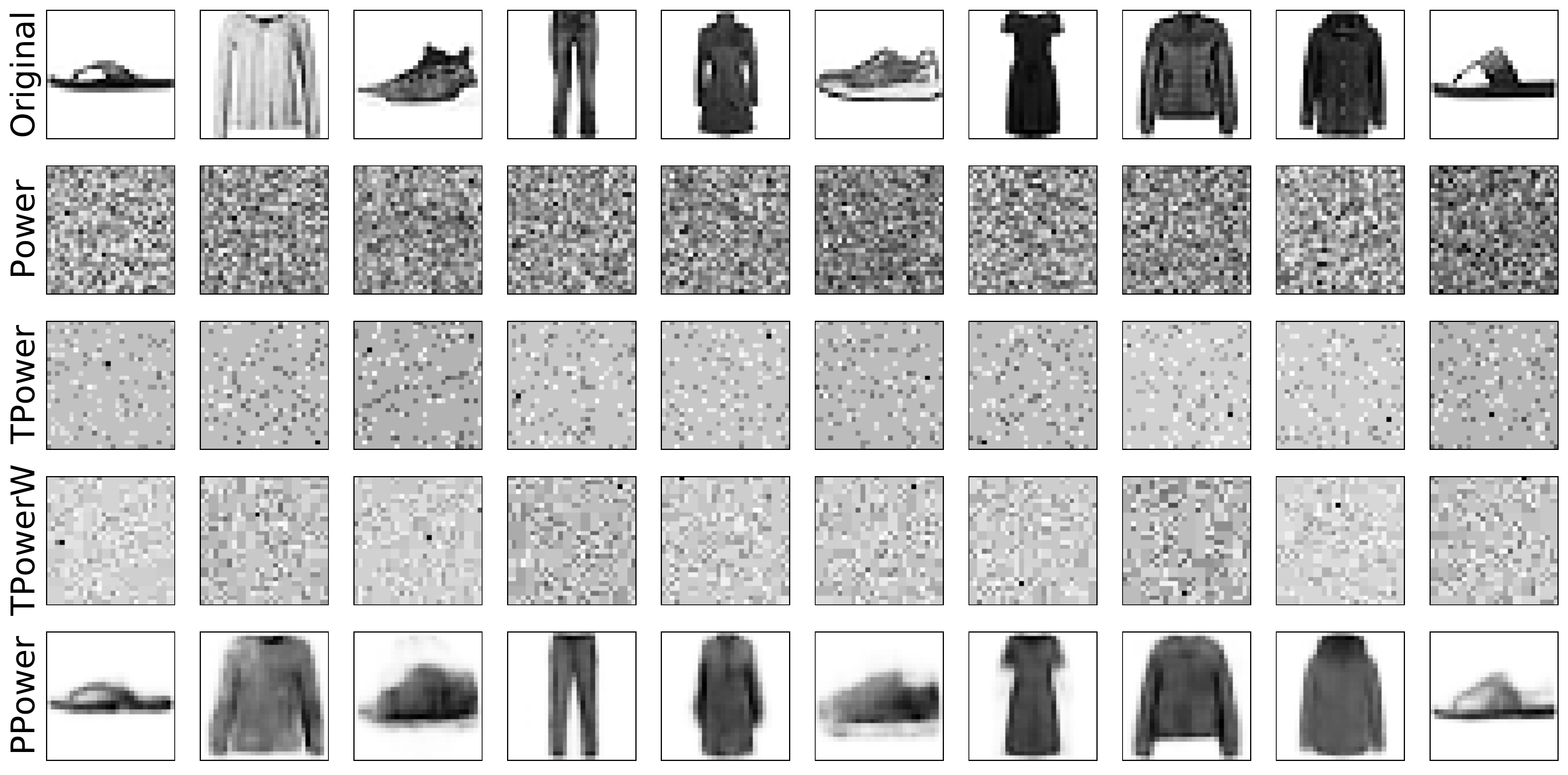} & \hspace{-0.5cm}
\includegraphics[height=0.19\textwidth]{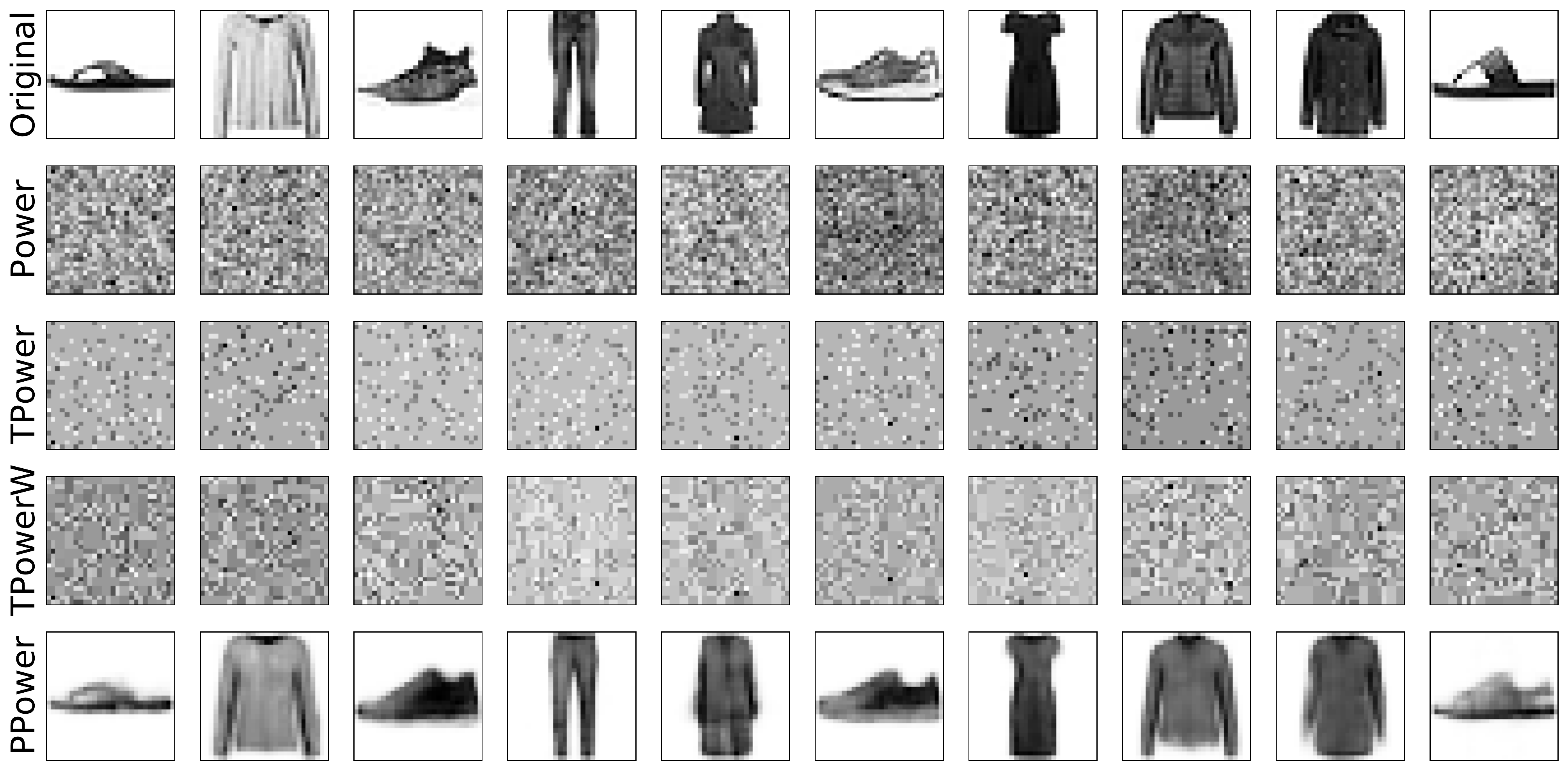} \\
{\small (a) $\beta = 1$ and $m = 200$} & {\small (b) $\beta = 2$ and $m = 100$}
\end{tabular}
\caption{Examples of reconstructed Fashion-MNIST images for the spiked covariance model.}
\label{fig:fashionMNIST_imgs_spikedCov}
\end{center}
\end{figure} 

 \begin{figure}
\begin{center}
\begin{tabular}{cc}
\includegraphics[height=0.35\textwidth]{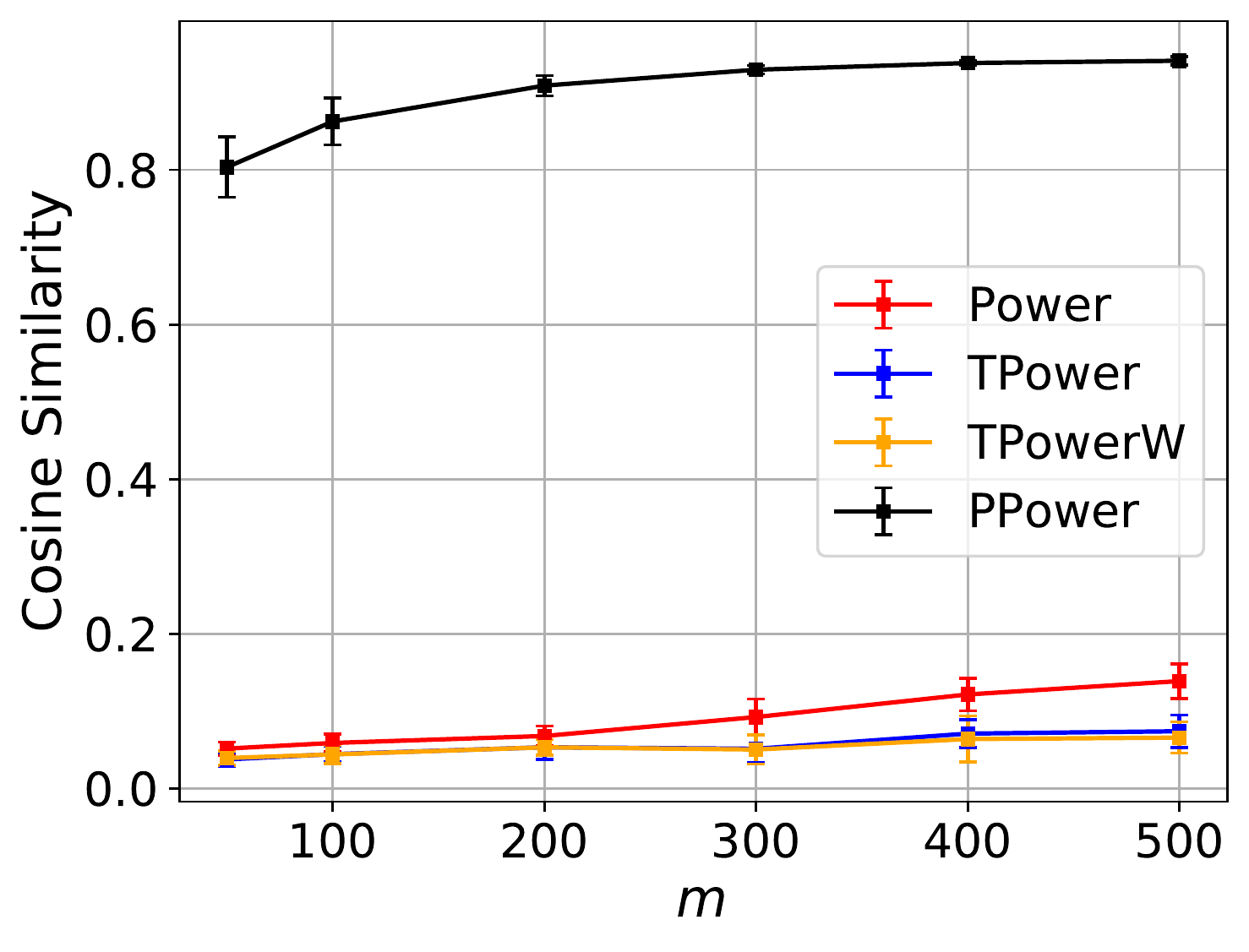} & \hspace{-0.5cm}
\includegraphics[height=0.35\textwidth]{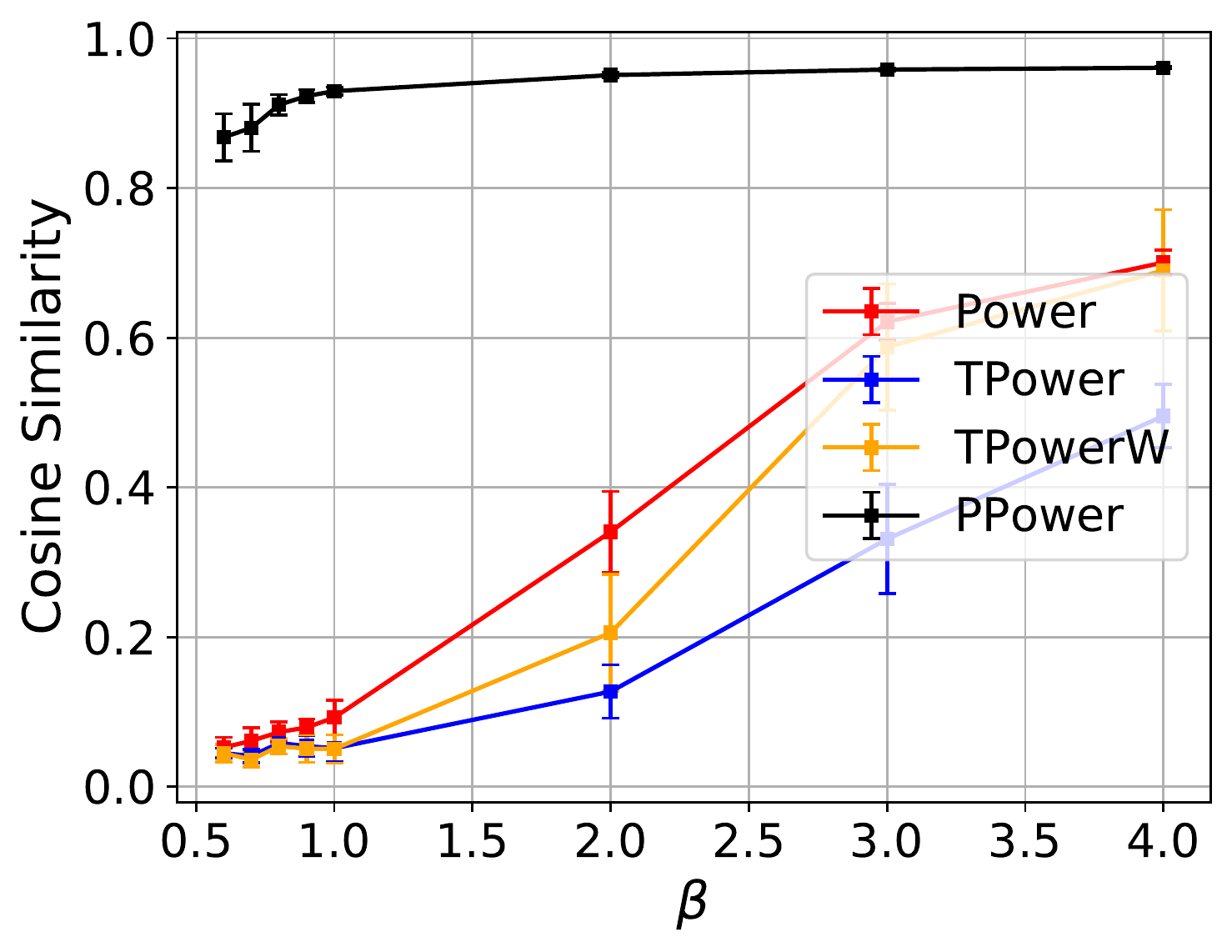} \\
{\small (a) Fixing $\beta = 1$ and varying $m$} & {\small (b) Fixing $m = 300$ and varying $\beta$}
\end{tabular}
\caption{Quantitative comparisons of the performance of~\texttt{Power},~\texttt{TPower} and~\texttt{PPower} according to the Cosine Similarity for the Fashion-MNIST dataset and the spiked covariance model.} \label{fig:quant_fashionMNIST_spikedCov}
\end{center}
\end{figure} 
    \begin{figure}
\begin{center}
\begin{tabular}{cc}
\includegraphics[height=0.19\textwidth]{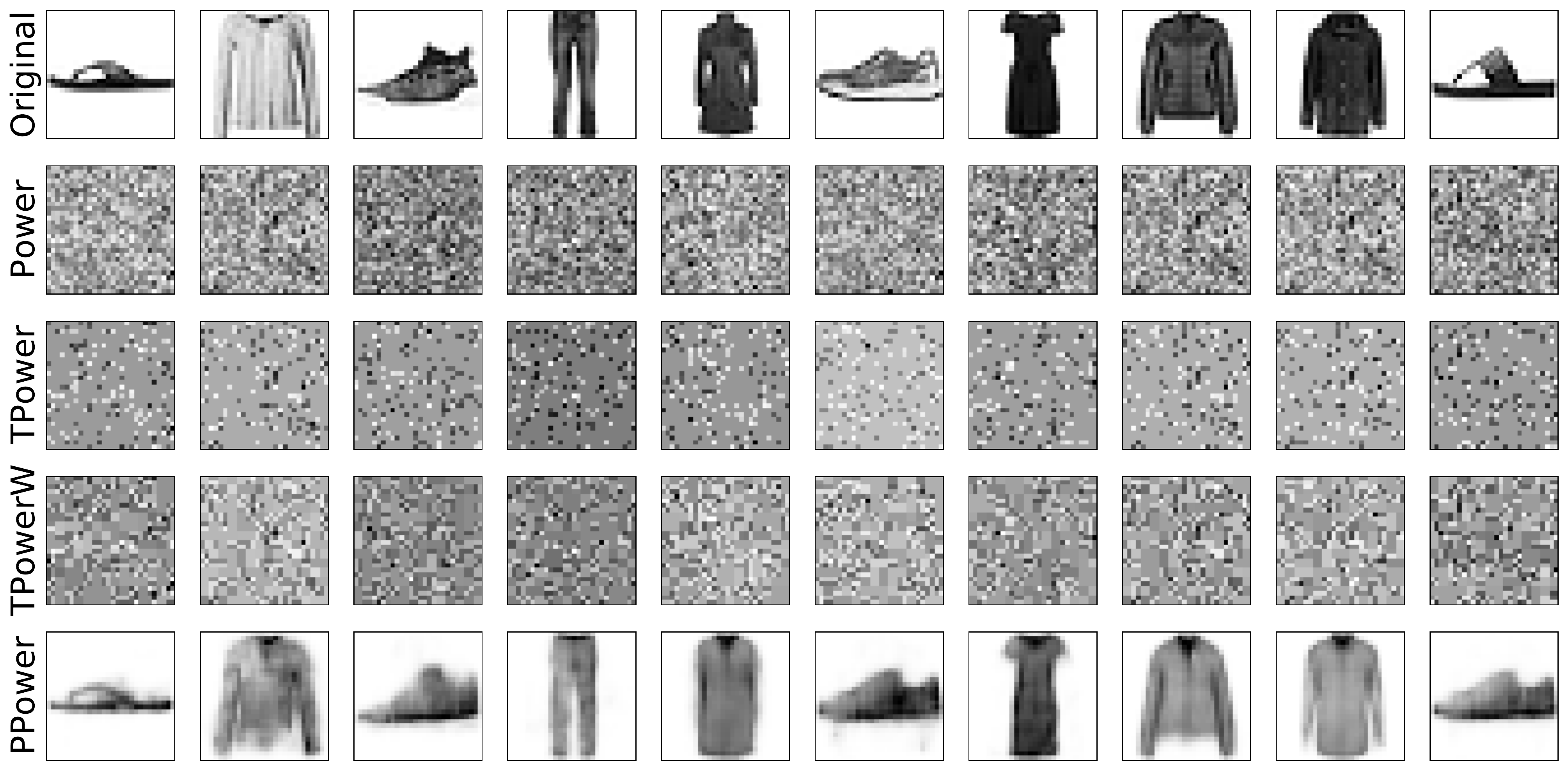} & \hspace{-0.5cm}
\includegraphics[height=0.19\textwidth]{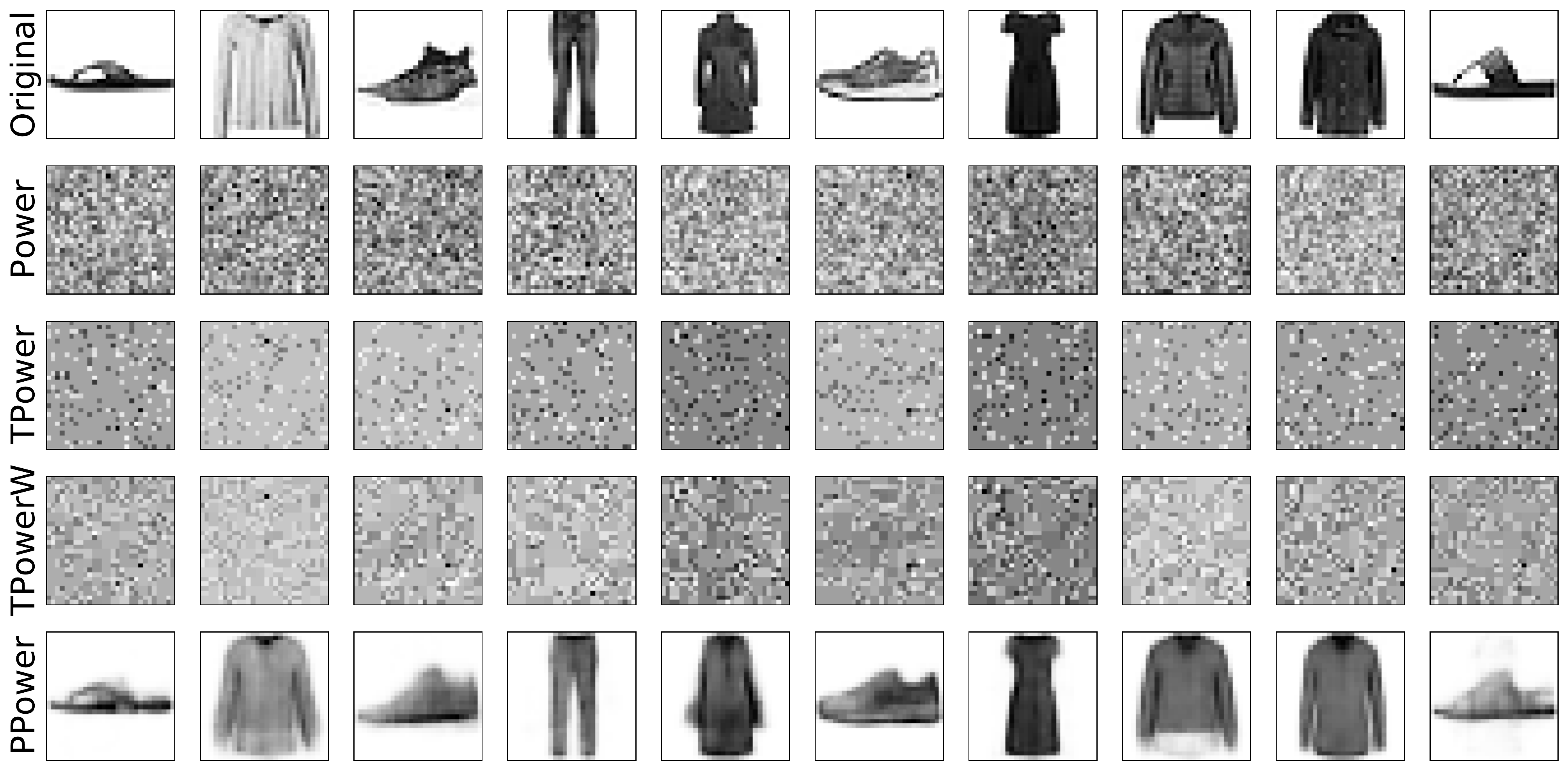} \\
{\small (a) $m = 200$} & {\small (b) $m = 400$}
\end{tabular}
\caption{Examples of reconstructed images of the Fashion-MNIST dataset for phase retrieval.} \label{fig:fashionMNIST_imgs_pr}
\end{center}
\end{figure} 

\begin{figure}
\begin{center}
\includegraphics[height=0.35\columnwidth]{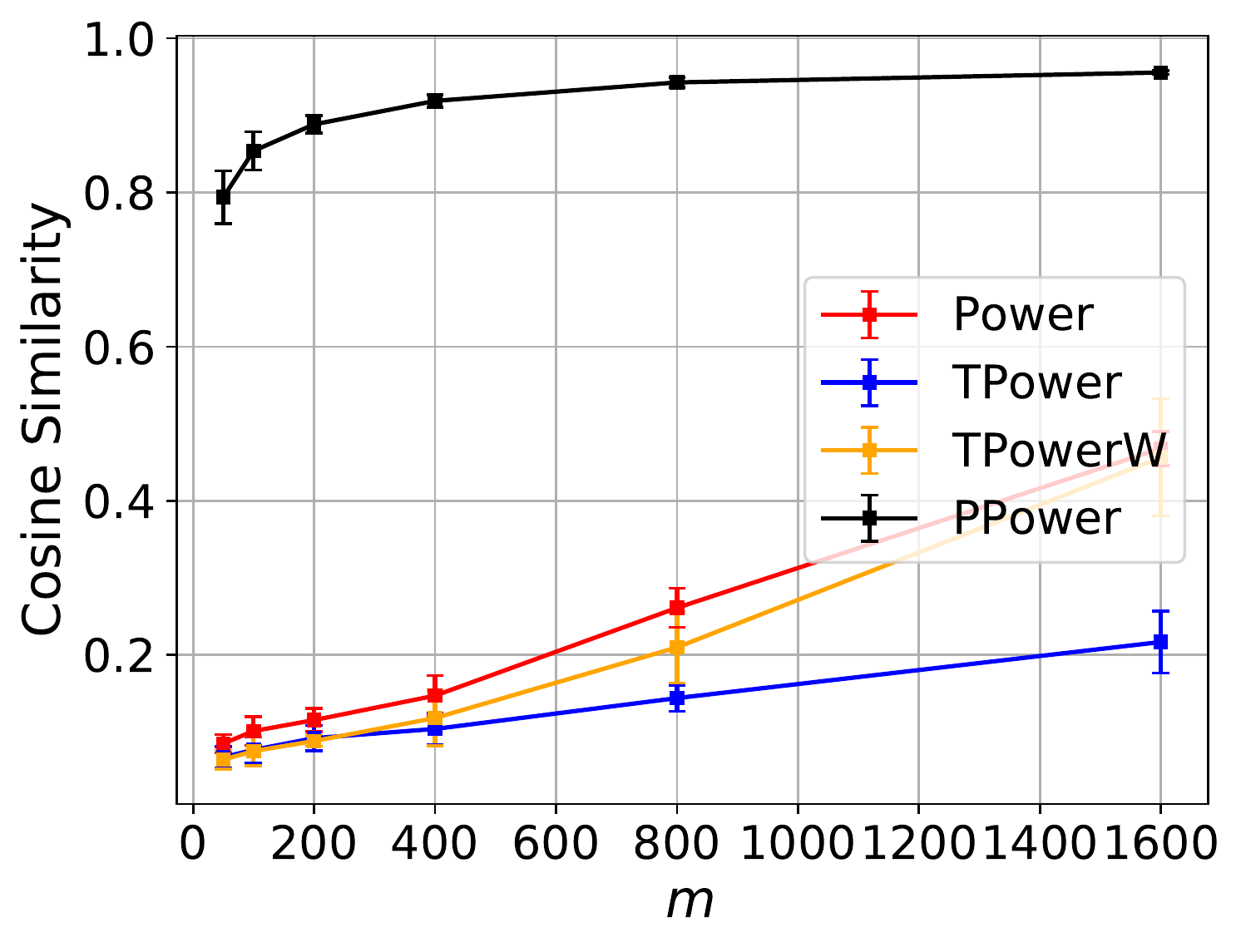}
\caption{Quantitative comparisons of the performance of~\texttt{Power},~\texttt{TPower} and~\texttt{PPower} according to the Cosine Similarity for the Fashion-MNIST dataset and the phase retrieval model.} \label{fig:quant_fashionMNIST_pr} 
\end{center}
\end{figure}

\section{Numerical Results for the CelebA Dataset}
\label{app:exp_CelebA}

The CelebA dataset consists of more than $200, 000$ face images of celebrities, where each input image is cropped to a $64 \times 64$ RGB image with $n = 64 \times 64 \times 3 = 12288$. The generative model $G$ is set to be a
pre-trained Deep Convolutional Generative Adversarial Networks (DCGAN) model with latent dimension $k = 100$. We use the DCGAN model trained by the authors of~\cite{bora2017compressed} directly. We select the best estimate among $2$ random restarts.
The Adam optimizer with $100$ steps and a learning rate of $0.1$
is used for the projection operator. 

 In each iteration of~\texttt{TPower} and~\texttt{TPowerW}, the calculated entries are truncated to zero except for the largest $q$ entries. For CelebA, $q$ is set to be $2000$.  The remaining parameters are the same as those for the MNIST dataset. Here we focus on the spiked covariance model.
The results are reported in Figures~\ref{fig:CelebA_imgs_spikedCov} and~\ref{fig:quant_CelebA_spikedCov}. 
From these figures, we can again observe the superiority of~\texttt{PPower} compare to the baselines.  
While \texttt{TPowerW} does improve over \texttt{TPower}, its performance is still very limited.

 \begin{figure}
\begin{center}
\begin{tabular}{cc}
 \includegraphics[height=0.25\textwidth]{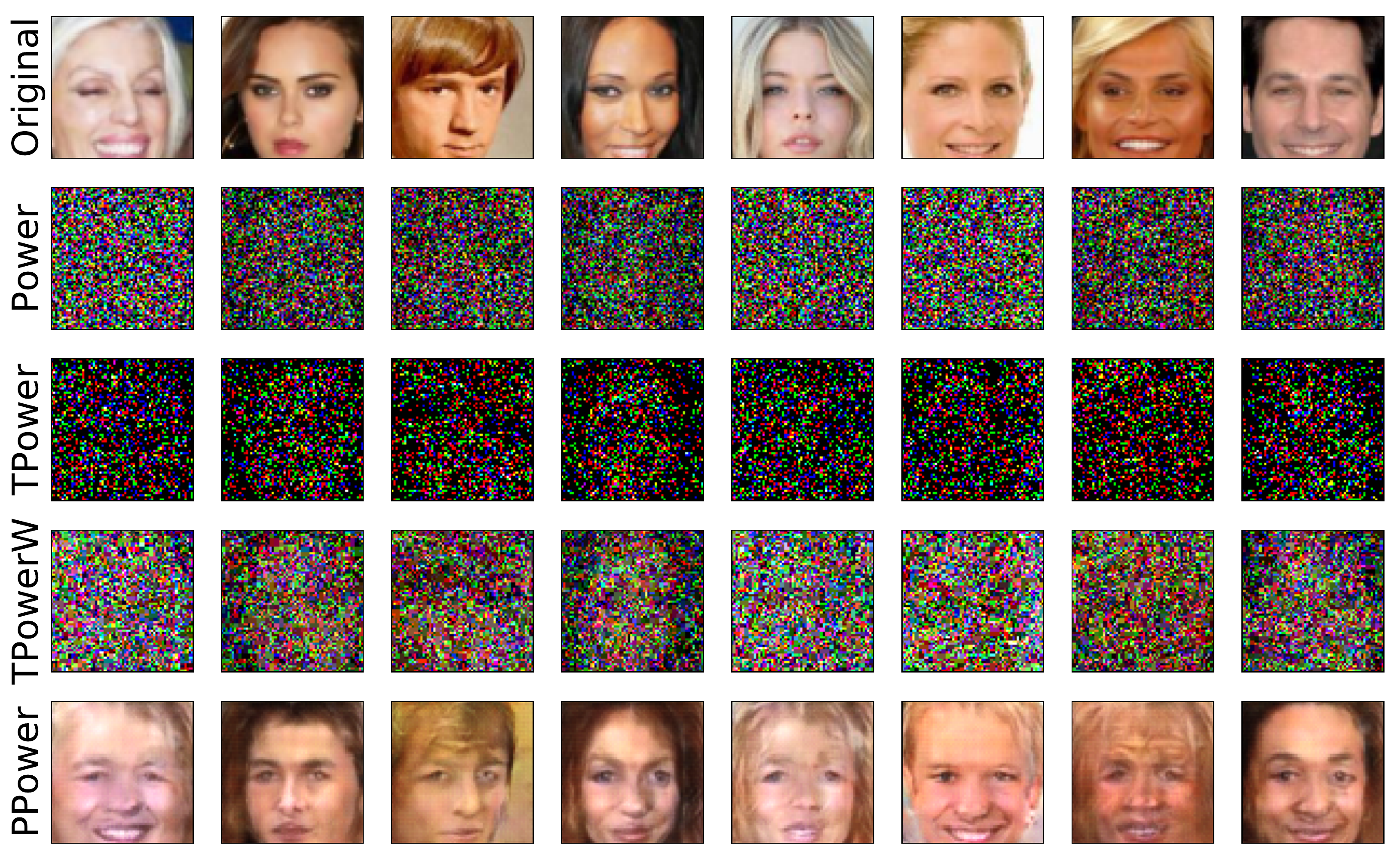} & \hspace{-0.5cm}
\includegraphics[height=0.25\textwidth]{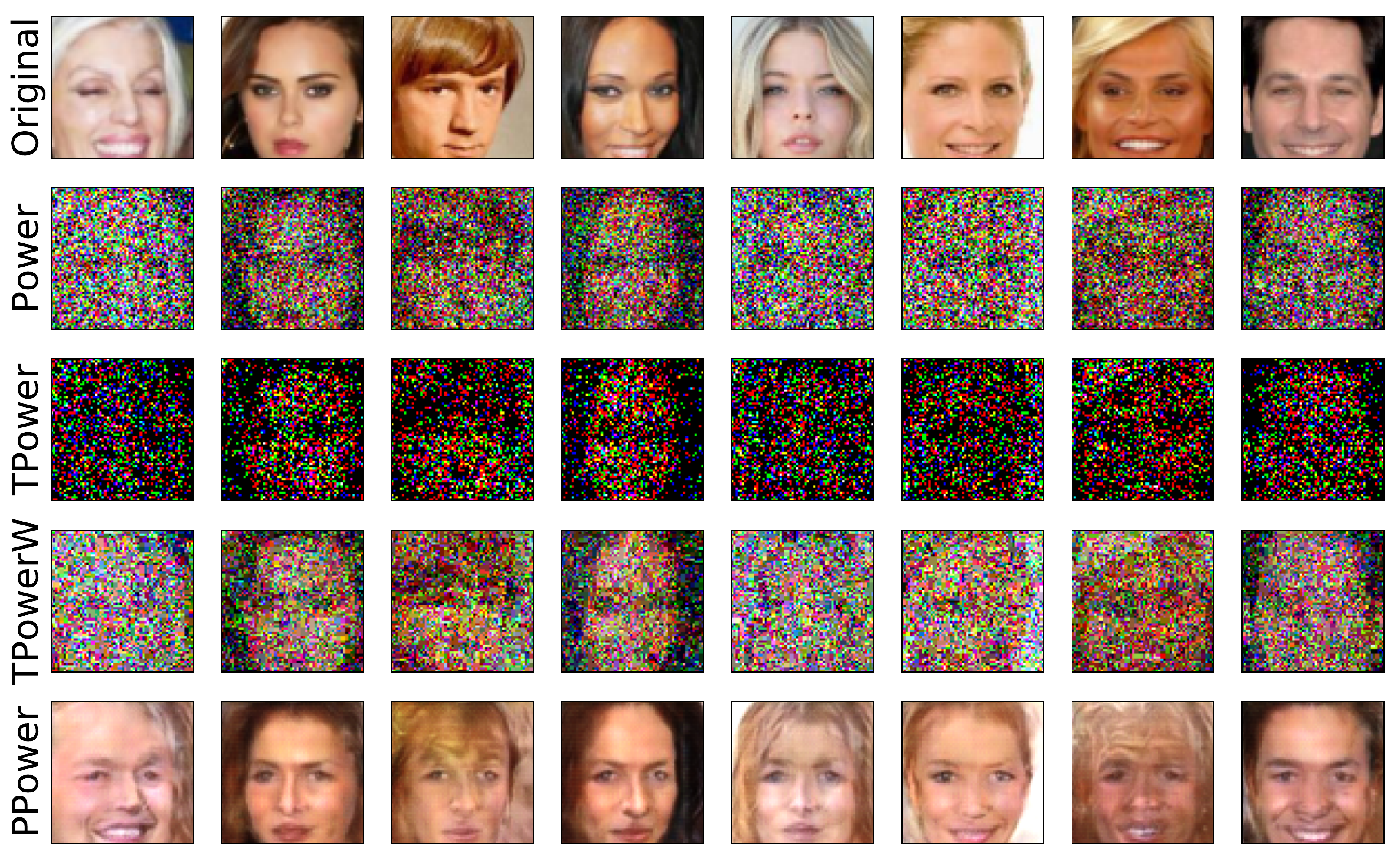} \\
{\small (a) $\beta = 1$ and $m = 6000$} & {\small (b) $\beta = 4$ and $m = 3000$}
\end{tabular}
\caption{Examples of reconstructed CelebA images for the spiked covariance model.}
\label{fig:CelebA_imgs_spikedCov}
\end{center}
\end{figure} 

 \begin{figure}
\begin{center}
\begin{tabular}{cc}
\includegraphics[height=0.35\textwidth]{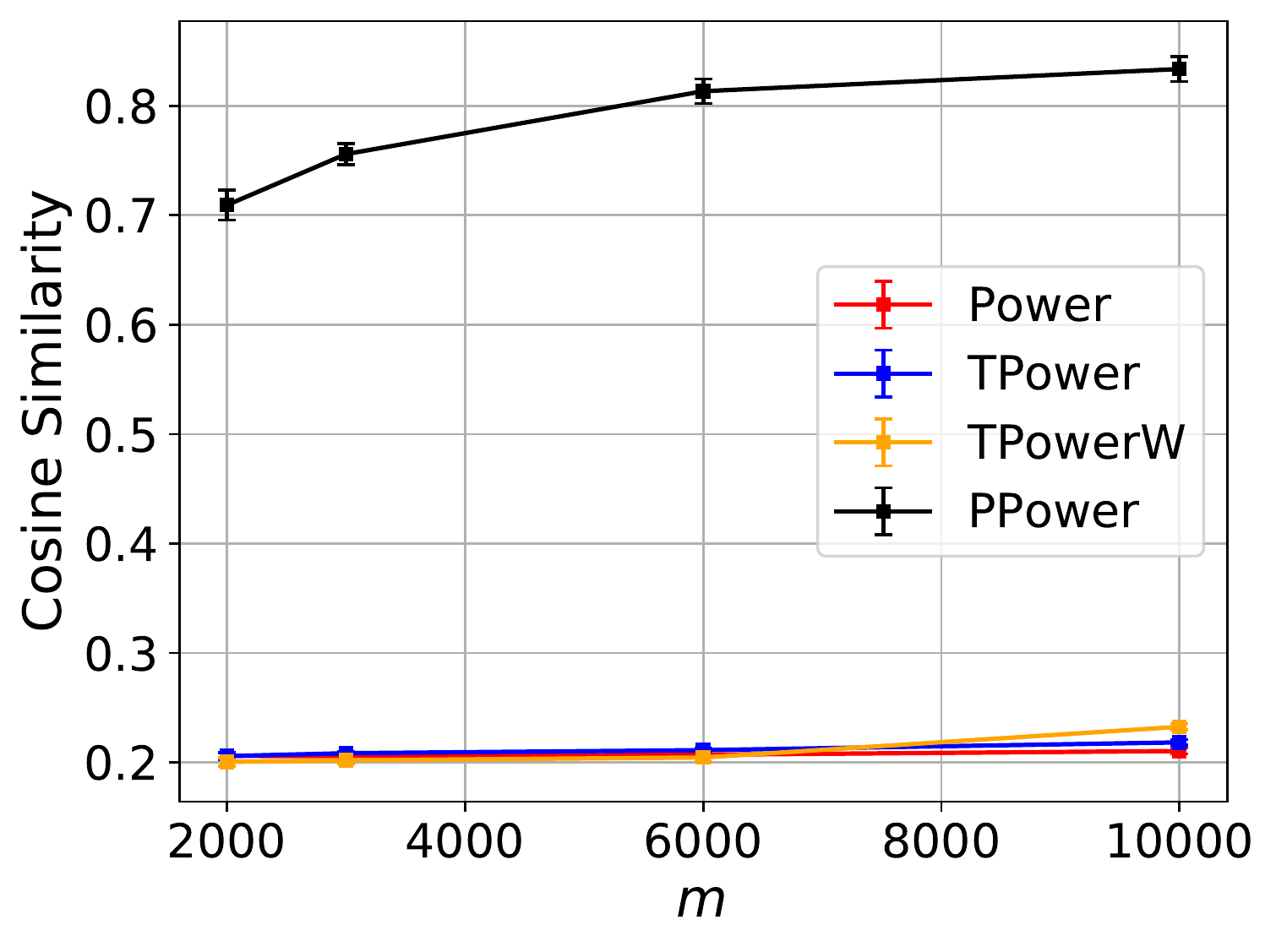} & \hspace{-0.5cm}
\includegraphics[height=0.35\textwidth]{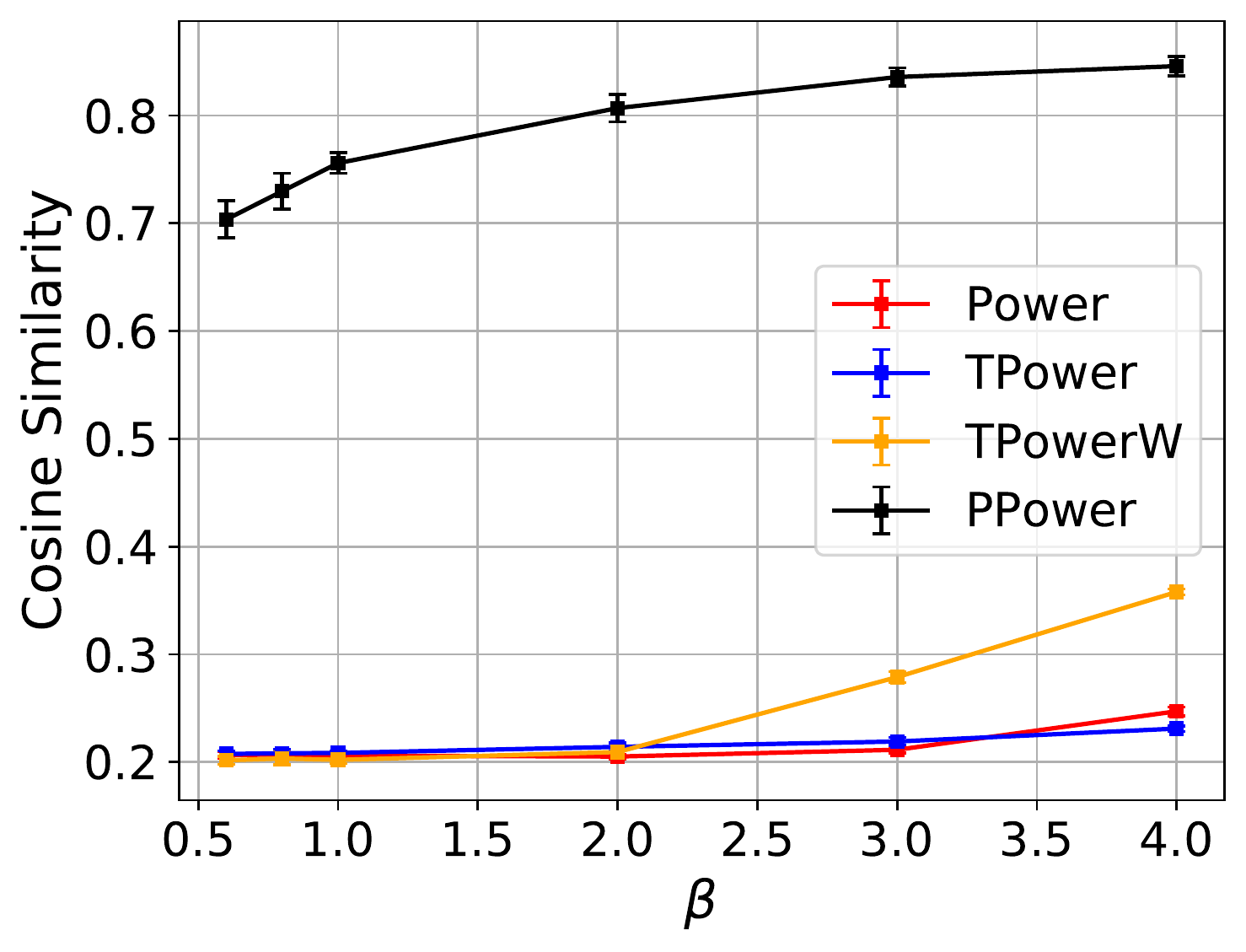} \\
{\small (a) Fixing $\beta = 1$ and varying $m$} & {\small (b) Fixing $m = 3000$ and varying $\beta$}
\end{tabular}
\caption{Quantitative comparisons of the performance of~\texttt{Power},~\texttt{TPower},~\texttt{TPowerW} and~\texttt{PPower} according to the Cosine Similarity for the CelebA dataset and spiked covariance model.} \label{fig:quant_CelebA_spikedCov}
\end{center}
\end{figure} 

\section{Lower Bound for the Reconstruction Error}
\label{app:lb_gpca}

In the following, we provide an algorithm-independent lower bound that serves as a counterpart to the upper bound given in Theorem \ref{thm:barbx_bxG}.  Note that here we consider the case that there is no representation error, which amounts to having $\|\bar{\bx} - \bx_G\|_2 = 0$ in Theorem \ref{thm:barbx_bxG}.  

\begin{theorem}\label{thm:lb_gpca}
  Given positive values of $r$ and $L$ such that $Lr$ is sufficiently large, and positive integer parameters $k$ and $m$, there exists an $L$-Lipschitz continuous generative model $G \,:\, B_2^k(r) \to \mathbb{R}^n$ (for suitably-chosen $n$), and a family of distributions $\{P_{\bV}^{(\bar{\bx},m)}\}_{\bar{\bx} \in \mathrm{Range}(G)}$ on $\bbR^{n \times n}$, such that any estimator $\hat{\bv}$ (depending on $\bV$) must satisfy 
 \begin{equation}\label{eq:lb_gpca_main}
  \sup_{\bar{\bx} \in \mathrm{Range}(G)} \bbE [\|\hat{\bv}\hat{\bv}^T -\bar{\bx}\bar{\bx}^T\|_\rmF] = \Omega\left(\sqrt{\frac{k \log (Lr)}{m}}\right),
 \end{equation}
 where expectation is taken over randomness of $\bV$ (whose distribution depends on $\bar{\bx}$); moreover, the distributions $\{P_{\bV}^{(\bar{\bx},m)}\}$ satisfy the following properties:
 \begin{itemize}
     \item There are $m$ independent samples $\bx_1,\ldots,\bx_m$ drawn independently from $\calN(\mathbf{0},\bar{\bV})$ with $\bar{\bV} := (\bar{\lambda}_1 -\bar{\lambda}_2)\bar{\bx}\bar{\bx}^T +\bar{\lambda}_2 \bI_n$ for some constants $0 < \bar{\lambda}_2 < \bar{\lambda}_1$, and $\bV$ is the corresponding non-centered sample covariance matrix: 
     \begin{equation}
        \bV = \frac{1}{m}\sum_{i=1}^m \bx_i \bx_i^T. \label{eq:lb_V}
     \end{equation}
     \item $\bar{\bV} = \bbE[\bV]$ satisfies Assumption~\ref{assump:barbV} with eigenvalues $(\bar{\lambda}_1,\bar{\lambda}_2,\dotsc,\bar{\lambda}_2)$, and $\bE := \bV - \bar{\bV}$ satisfies Assumption~\ref{assump:E_cond} with probability $1- e^{-\Omega(\log (|S_1|\cdot|S_2|))} - e^{-\Omega(n)}$  (i.e., the same probability scaling as Example~\ref{exam:spikedCov} in Section \ref{sec:examples}).
 \end{itemize}
 
 
\end{theorem}

Before proving the theorem, we present some useful lemmas. 

\subsection{Auxiliary Results for Theorem~\ref{thm:lb_gpca}}

We first state a standard lemma from \cite{yu1997assouad} based on Fano's inequality (see also \cite{Tsy09}), using generic notation.

\begin{lemma}{\em \hspace{1sp}\cite[Lemma~3]{yu1997assouad}}\label{lem:gen_fano}
 Let $N\ge 1$ be an integer and $\bs_1,\ldots,\bs_N \subset \Theta$ index a collection of probability measures $\bbP_{\bs_i}$ on a measurable space $(\calX,\calA)$. Let $\rmd$ be a pseudometric\footnote{A pseudometric is a generalization of a metric in which the distance between two distinct points can be zero.} on $\Theta$ and suppose that for all $i \ne j$,
 \begin{equation}
  \rmd(\bs_i,\bs_j) \ge \alpha_N
 \end{equation}
and 
\begin{equation}
 D(\bbP_{\bs_i} || \bbP_{\bs_j}) \le \beta_N,
\end{equation}
where $D(\bbP_1 ||\bbP_2)$ denotes the Kullback-Leibler (KL) divergence between two probability measures $\bbP_1$ and $\bbP_2$. Then, every $\calA$-measurable estimator $\hat{\bs}$ satisfies 
\begin{equation}
 \max_i \bbE_{\bs_i} \left[\rmd(\hat{\bs},\bs_i)\right] \ge \frac{\alpha_N}{2}\left(1-\frac{\beta_N +\log 2}{\log N}\right),
\end{equation}
where $\bbE_{\bs_i}$ is the average with respect to  $\bbP_{\bs_i}$.
\end{lemma}

In addition, we have the following lemma concerning the KL divergence between $m$-fold products of Gaussian probability measures.
\begin{lemma}{\em \hspace{1sp}\cite[Lemma~3.1.3]{vu2012minimax}}\label{lem:kl_gaussian}
 For $i=1,2$, let $\bw_i \in \calS^{n-1}$, $\lambda_1 > \lambda_2 > 0$, and
 \begin{equation}
  \mathbf{\Sigma}_i = (\lambda_1-\lambda_2) \bw_i \bw_i^T + \lambda_2 \bI_n.
 \end{equation}
 Let $\bbP_i^m$ be the $m$-fold product of $\calN(\mathbf{0},\mathbf{\Sigma}_i)$. Then,
 \begin{equation}
  D(\bbP_1^m||\bbP_2^m) = \frac{m}{2\alpha^2} \|\bw_1\bw_1^T -\bw_2\bw_2^T\|_{\rmF}^2,
 \end{equation}
where $\alpha^2 = \frac{\lambda_1\lambda_2}{(\lambda_1 -\lambda_2)^2}$.
\end{lemma}

\subsection{Proof of Theorem~\ref{thm:lb_gpca}}
\label{app:sub_lb_gpca}

To prove the lower bound, we follow the construction technique from~\cite[Appendix~C]{liu2020sample}, who proved the existence of a generative model $G\,:\, B_2^k(r) \to \calS^{n-1}$ (with $n$ being a suitably-chosen integer multiple of $k$) such that 
\begin{equation}\label{eq:rangeG}
 \mathrm{Range}(G) = \tilde{\calW}_k := \left\{\bw\in \calW_k\,:\, w_n \ge \frac{w_c}{\sqrt{(k-1)w_{\max}^2 + w_c^2}}\right\},
\end{equation}
where $w_{\max}$ and $w_c$ are positive constants, $w_n$ is the $n$-th entry of $\bw$, and $\calW_k$ denotes the set of $k$-group-sparse vectors in $\calS^{n-1}$, i.e., the $n$ coefficients are arranged into $k$ disjoint groups of size $\frac{n}{k}$, and exactly one coefficient in each group is non-zero.  In addition, we have that $G$ is $L$-Lipschitz continuous with the Lipschitz constant $L$ being~\cite{liu2020sample}
\begin{equation}\label{eq:LipschitzConst}
 L = \frac{2n w_{\max}}{\sqrt{k} r w_c}  = \frac{n}{k} \cdot \frac{2\sqrt{k} w_{\max}}{r w_c}. 
\end{equation}
For any $\lambda \in (0,1)$, let 
\begin{equation}
 S(\lambda) := \{\bw \in \calW_k\,:\, w_n \ge \lambda\}.  \label{eq:setS}
\end{equation} 
Then, from \eqref{eq:rangeG}--\eqref{eq:LipschitzConst}, when $(k-1)w_{\max}^2 = 3 w_c^2$, we obtain
\begin{equation}\label{eq:Lr_nk_relation}
 Lr = \Theta\left(\frac{n}{k}\right),
\end{equation}
and
\begin{equation}
 \mathrm{Range}(G) = \left\{\bw\in \calW_k\,:\, w_n \ge \frac{1}{2}\right\} = S\left(\frac{1}{2}\right). \label{eq:following}
\end{equation}
Then, adapting an idea from~\cite{liu2020sample}, we claim that for any $\epsilon \in \big(0,\frac{\sqrt{3}}{2}\big)$, there exists a subset $\calC_{\epsilon} \subseteq S\big(\frac{1}{2}\big)$ such that:
\begin{itemize}
    \item $\log |\calC_{\epsilon}| \ge ck \log\frac{n}{k}$ for some universal constant $c$;
    \item for all distinct pairs $\bw_1,\bw_2 \in \calC_{\epsilon}$, it holds that $\frac{\epsilon}{\sqrt{2}} < \|\bw_1-\bw_2\|_2 \le \sqrt{2}\epsilon$. 
\end{itemize}
For completeness, a proof of this claim is given at Appendix~\ref{app:claim_lb_gpca}. From Lemma~\ref{lem:simple_dist_eq}, we obtain for all distinct pairs $\bw_1,\bw_2 \in \calC_{\epsilon}$ that
\begin{equation}
 \frac{\epsilon^2}{2} \le \|\bw_1\bw_1^T-\bw_2\bw_2^T\|_\rmF^2 \le 4\epsilon^2. \label{eq:Frob_bounds}
\end{equation}
Fix $\bar{\lambda}_1 > \bar{\lambda}_2 > 0$, and for each $\bw \in \calC_{\epsilon}$, let
\begin{equation}\label{eq:Sigma_w}
 \mathbf{\Sigma}_{\bw} = (\bar{\lambda}_1-\bar{\lambda}_2) \bw \bw^T + \bar{\lambda}_2 \bI_n.
\end{equation}
Since $\|\bw\|_2 = 1$, it immediately follows that the eigenvalues of $\mathbf{\Sigma}_{\bw}$ are $(\bar{\lambda}_1,\bar{\lambda}_2,\dotsc,\bar{\lambda}_2)$, in accordance with Assumption~\ref{assump:barbV}.

Since Theorem \ref{thm:lb_gpca} concerns the worst-case $\bar{\bx}$, it suffices to prove hardness within an arbitrary restricted set.  Accordingly, we consider $\bar{\bx} = \bw$ for all possible choices of $\bw \in \calC_{\epsilon}$.  As stated in the theorem, when $\bar{\bx} = \bw$, the samples are drawn from the $m$-fold product of $\calN(\mathbf{0}, \mathbf{\Sigma}_{\bw})$ (i.e., $\bar{\bV} = \mathbf{\Sigma}_{\bw}$), which we denote by $\bbP^m_{\bw}$, and $\bV$ is given by \eqref{eq:lb_V}.  The fact that Assumption \ref{assump:E_cond} holds for any such $\bw$ is shown in a similar manner to Example~\ref{exam:spikedCov}; the details are given in Appendix~\ref{app:claim2_lb_gpca} for completeness.

From Lemma~\ref{lem:kl_gaussian} and \eqref{eq:Frob_bounds}, we have for all distinct pairs $\bw_1,\bw_2 \in \calC_{\epsilon}$ that
\begin{equation}
 D\left(\bbP_{\bw_1}^m ||\bbP_{\bw_2}^m \right) = \frac{m}{2\alpha^2} \|\bw_1\bw_1 -\bw_2\bw_2^T\|_\rmF^2 \le \frac{2m\epsilon^2}{\alpha^2},
\end{equation}
where $\alpha^2 := \frac{\bar{\lambda}_1 \bar{\lambda}_2}{(\bar{\lambda}_1 - \bar{\lambda}_2)^2}$. By the assumption that $\bar{\lambda}_1$ and $\bar{\lambda}_2$ are fixed positive constants with $\bar{\lambda}_2 < \bar{\lambda}_1$, we also have that $\alpha^2$ is a fixed positive constant. Then, from Lemma~\ref{lem:gen_fano} and \eqref{eq:Frob_bounds}, we obtain
\begin{equation}
 \max_{\bw \in \calC_{\epsilon}} \bbE_{\bw} \left[\|\hat{\bv}\hat{\bv}^T -\bw\bw^T\|_\rmF\right] \ge \frac{\epsilon}{2\sqrt{2}}\left(1-\frac{2m\epsilon^2/\alpha^2 +\log 2}{\log |\calC_\epsilon|}\right), \label{eq:Fano_init}
\end{equation}
where $\bbE_{\bw}$ is the average with respect to $\bbP_{\bw}^m$ (leaving the dependence on $m$ implicit).  

Recalling that $\log |\calC_{\epsilon}| \ge ck \log\frac{n}{k}$, we now observe that by choosing $\epsilon$ to be on the order $\alpha \sqrt{\frac{k\log\frac{n}{k}}{m}}$, we can ensure that
\begin{equation}
 \frac{2 m \epsilon^2/\alpha^2}{\log |\calC_\epsilon|} \le \frac{1}{4}.
\end{equation}
Combining this with the fact that $\log |\calC_\epsilon| \to \infty$ and in particular $ \log |\calC_\epsilon| \ge 4\log 2$ for sufficiently large $n$, it follows from \eqref{eq:Fano_init} that
\begin{equation}
 \max_{\bw \in \calC_{\epsilon}} \bbE_{\bw} \left[\|\hat{\bv}\hat{\bv}^T -\bw\bw^T\|_\rmF\right] \ge \frac{\epsilon}{4\sqrt{2}}. 
\end{equation}
Substituting the chosen scaling of $\epsilon$ and the behavior of the Lipschitz constant in \eqref{eq:Lr_nk_relation}, this yields
\begin{equation}\label{eq:maxBw}
    \max_{\bw \in \calC_{\epsilon}} \bbE_{\bw} \left[\|\hat{\bv}\hat{\bv}^T -\bw\bw^T\|_\rmF\right] \ge C \alpha \sqrt{\frac{k\log(Lr)}{m}}.
\end{equation}
for some positive constant $C$.  

The lower bound in \eqref{eq:maxBw} holds even for an algorithm that has access to the samples $\bx_1,\dotsc,\bx_m$, whereas Theorem \ref{thm:lb_gpca} concerns only having access to $\bV$.  Since $\bV$ is deterministic given $\bx_1,\dotsc,\bx_m$ (see \eqref{eq:lb_V}), the latter problem can only be more difficult.  Hence, by identifying $\bar{\bx}$ with $\bw$ in \eqref{eq:maxBw} and further upper bounding the maximum over $\calC_{\epsilon}$ by the maximum over all of ${\rm Range}(G)$, the proof of Theorem \ref{thm:lb_gpca} is complete. 




\subsection{Proof of the Claim Following \eqref{eq:following}}
\label{app:claim_lb_gpca}

Consider the set of length-$(n-1)$ binary-valued signals
 \begin{equation}
  \calU := \Big\{\br \in \{0,1\}^{n-1}\,:\, {\large \big[ \substack{ \br \\ 1 } \big]} \text{ is } k \text{-group-sparse} \Big\},
 \end{equation}
where we recall that $n$ is an integer multiple of $k$. For any fixed $\br \in \calU$, we have 
\begin{equation}
 \left|\left\{\br' \in \calU\,:\, \|\br-\br'\|_0\le \frac{k-1}{2}\right\}\right| \le \binom{k-1}{\frac{k-1}{2}} \left(\frac{n}{k}\right)^{(k-1)/2},
\end{equation}
where the upper bound follows from that $\br'$ can be distinct with $\br$ in at most $\frac{k-1}{2}$ blocks among the $k-1$ blocks, and there are at most $\frac{n}{k} $ choices in each block. We aim to construct a set $\Omega \subseteq \calU$ such that for all distinct $\br,\br' \in \Omega$, 
\begin{equation}\label{eq:eqOmega}
 \|\br-\br'\|_0 > \frac{k-1}{2}.
\end{equation}
Suppose that we construct $\Omega$ by picking elements in $\calU$ uniformly at random. When adding the $j$-th point to $\Omega$, the probability that this point violates~\eqref{eq:eqOmega} with respect to the previously added points is upper bounded by
\begin{equation}
 \frac{(j-1)\binom{k-1}{\frac{k-1}{2}} \left(\frac{n}{k}\right)^{(k-1)/2}}{(\frac{n}{k})^{k-1}}.
\end{equation}
Therefore, taking a union bound, we can bound the total probability that $\Omega$ fails to satisfy~\eqref{eq:eqOmega}, denoted $P_1$, by
\begin{equation}
 P_1 \le \sum_{j=1}^{|\Omega|}\frac{(j-1)\binom{k-1}{\frac{k-1}{2}} \left(\frac{n}{k}\right)^{(k-1)/2}}{(\frac{n}{k})^{k-1}} \le \frac{|\Omega|^2}{2} \frac{\binom{k-1}{\frac{k-1}{2}}}{\left(\frac{n}{k}\right)^{(k-1)/2}} \le \frac{|\Omega|^2}{2} \left(\frac{2ek}{n}\right)^{(k-1)/2}. 
\end{equation}
Then, under the condition 
\begin{equation}\label{eq:cardOmega}
 |\Omega| \le \left(\frac{n}{2ek}\right)^{(k-1)/4},
\end{equation}
it holds that $P_1 \le \frac{1}{2}$. This implies that we can choose a set $\Omega \subseteq \calU$ such that both~\eqref{eq:eqOmega} and~\eqref{eq:cardOmega} are satisfied; moreover, in the latter, we can ensure that equality holds up to insignificant rounding.

For any $\epsilon \in (0,\frac{\sqrt{3}}{2})$ (noting that the upper limit is the value such that $\sqrt{1-\epsilon^2} = \frac{1}{2}$), setting
\begin{equation}
 \calC_{\epsilon}:= \left\{\left(\frac{\epsilon}{\sqrt{k-1}}\br, \sqrt{1-\epsilon^2}\right)\,:\, \br \in \Omega\right\},
\end{equation}
we find that $\calC_{\epsilon} \subseteq S\big(\frac{1}{2}\big)$, where $S(\cdot)$ is defined in \eqref{eq:setS}. In addition, since we are considering binary-valued signals satisfying \eqref{eq:eqOmega}, we have for any $\br,\br' \in \Omega$ that
\begin{equation}
 \|\br-\br'\|_2^2 = \|\br-\br'\|_0, \quad  \frac{k-1}{2}<\|\br-\br'\|_0 \le 2(k-1).
\end{equation}
Combined with the fact that we ensured \eqref{eq:cardOmega} holds with equality (up to rounding), this completes the proof of the claim.

\subsection{Proof that Assumption \ref{assump:E_cond} Holds for $\bw \in \calC_{\epsilon}$}
\label{app:claim2_lb_gpca}
  
Fix $\bw \in \calC_{\epsilon}$, and suppose that the $m$ samples $\bx_1,\ldots,\bx_m$ are generated following the Gaussian distribution $\calN(\mathbf{0},\bSigma_{\bw})$ with $\bSigma_{\bw} = (\bar{\lambda}_1 -\bar{\lambda}_2)\bw\bw^T + \bar{\lambda}_2 \bI_n$. Let $\bV := \frac{1}{m}\sum_{i=1}^m \bx_i \bx_i^T$ be the non-centered sample covariance matrix. Clearly, we have $\bbE[\bV] = \bSigma_{\bw}$, and thus $\bE := \bV - \bSigma_{\bw}$ has mean zero.

We write $\bx_i = \mathbf{\Sigma}_{\bw}^{1/2} \bt_i$, where $\bt_i \sim\calN(\mathbf{0},\bI_n)$. As per Assumption~\ref{assump:E_cond}, fix two finite signal sets $S_1$ and $S_2$. For any $\bs_1 \in S_1$ and $\bs_2 \in S_2$, letting $\tilde{\bs}_1 = \mathbf{\Sigma}_{\bw}^{1/2} \bs_1$ and $\tilde{\bs}_2 = \mathbf{\Sigma}_{\bw}^{1/2} \bs_2$, we have 
\begin{align}
 \bs_1^T\bE\bs_2 &= \bs_1^T(\bV -\bSigma_{\bw})\bs_2 \\
 & = \frac{1}{m}\sum_{i=1}^m (\bx_i^T\bs_1)(\bx_i^T\bs_2) - \bs_1^T\bSigma_{\bw}\bs_2\\
 & = \frac{1}{m}\sum_{i=1}^m  \left((\bt_i^T\tilde{\bs}_1)(\bt_i^T\tilde{\bs}_2) - \left(\tilde{\bs}_1^T\tilde{\bs}_2\right)\right) .\label{eq:claim2_twoTerms}
 \end{align}

We upper bound~\eqref{eq:claim2_twoTerms} using a concentration argument. We observe that $\bbE[(\bt_i^T\tilde{\bs}_1)(\bt_i^T\tilde{\bs}_2)] = \tilde{\bs}_1^T\tilde{\bs}_2$. In addition, we have that for $j \in \{1,2\}$, $\bt_i^T\tilde{\bs}_j$ is sub-Gaussian, and the sub-Gasusian norm is $C\|\tilde{\bs}_j\|_2$ for some constant $C$. Applying Lemma~\ref{lem:prod_subGs}, we deduce that $(\bt_i^T\tilde{\bs}_1)(\bt_i^T\tilde{\bs}_2)$ is sub-exponential, with the sub-exponential norm being upper bounded by $C^2 \|\tilde{\bs}_1\|_2 \|\tilde{\bs}_2\|_2$. From Lemma~\ref{lem:large_dev}, it follows that for any $t>0$ satisfying $m = \Omega(t)$, the following holds with probability $1-e^{-\Omega(t)}$ (recall that $C$ may vary from line to line):
\begin{equation}
 \frac{1}{m}\sum_{i=1}^m  \left((\bt_i^T\tilde{\bs}_1)(\bt_i^T\tilde{\bs}_2) - \left(\tilde{\bs}_1^T\tilde{\bs}_2\right)\right) \le C \|\tilde{\bs}_1\|_2  \|\tilde{\bs}_2\|_2 \cdot \frac{\sqrt{t}}{\sqrt{m}} \le  C\bar{\lambda}_1 \|\bs_1\|_2\|\bs_2\|_2\cdot \frac{\sqrt{t}}{\sqrt{m}},\label{eq:main_bd_claim2}
\end{equation}
where we note that the assumption $m = \Omega(t)$ ensures that the first term is dominant in the minimum in~\eqref{eq:subexp2}, and the last inequality uses $\|\tilde{\bs}_j\|_2 \le \|\mathbf{\Sigma}_{\bw}^{1/2}\|_{2\to 2} \|\bs_j\|_2 = \sqrt{\bar{\lambda}_1}\|\bs_j\|_2$.   Taking a union bound over all $\bs_1 \in S_1$ and $\bs_2 \in S_2$ in \eqref{eq:main_bd_claim2}, and setting $t = \log (|S_1|\cdot|S_2|)$, we obtain with probability $1-e^{-\Omega(\log (|S_1| \cdot |S_2|))}$ that~\eqref{eq:bE_cond} holds. 

Finally, similarly to Example~\ref{exam:spikedCov}, we also have with probability $1-e^{-\Omega(n)}$ that
\begin{equation}\label{eq:ub_bE_claim2}
 \|\bE\|_{2\to 2} = O\left(\frac{n}{m}\right).
\end{equation}

\bibliographystyle{IEEEtran}
\bibliography{writeups}

\end{document}

%% file: preamble.tex
\usepackage[mathscr]{eucal}
\usepackage{epsfig,epsf,psfrag}
\usepackage{amssymb,amsmath,amsfonts,latexsym}
\usepackage{amsmath,graphicx,bm,xcolor,url}
\usepackage[caption=false]{subfig} 
\usepackage{fixltx2e}
\usepackage{array}
\usepackage{verbatim}
\usepackage{bm}
\usepackage{algpseudocode}
\usepackage{algorithm}
\usepackage{verbatim}
\usepackage{textcomp}
\usepackage{mathrsfs}
\usepackage{epstopdf}
\usepackage{relsize}
\usepackage{cleveref} 
\usepackage{subfig}
 \usepackage{amsthm}

\catcode`~=11 \def\UrlSpecials{\do\~{\kern -.15em\lower .7ex\hbox{~}\kern .04em}} \catcode`~=13 

\allowdisplaybreaks[3]

\newcommand{\calA}{\mathcal{A}}

\newcommand{\calC}{\mathcal{C}}
\newcommand{\calD}{\mathcal{D}}

\newcommand{\calN}{\mathcal{N}}

\newcommand{\calP}{\mathcal{P}}

\newcommand{\calS}{\mathcal{S}}

\newcommand{\calU}{\mathcal{U}}

\newcommand{\calW}{\mathcal{W}}
\newcommand{\calX}{\mathcal{X}}

\newcommand{\ba}{\mathbf{a}}
\newcommand{\bA}{\mathbf{A}}
\newcommand{\bb}{\mathbf{b}}

\newcommand{\bc}{\mathbf{c}}

\newcommand{\bD}{\mathbf{D}}

\newcommand{\bE}{\mathbf{E}}

\newcommand{\bH}{\mathbf{H}}

\newcommand{\bI}{\mathbf{I}}

\newcommand{\br}{\mathbf{r}}

\newcommand{\bs}{\mathbf{s}}

\newcommand{\bt}{\mathbf{t}}

\newcommand{\bu}{\mathbf{u}}
\newcommand{\bU}{\mathbf{U}}
\newcommand{\bv}{\mathbf{v}}
\newcommand{\bV}{\mathbf{V}}
\newcommand{\bw}{\mathbf{w}}
\newcommand{\bW}{\mathbf{W}}
\newcommand{\bx}{\mathbf{x}}
\newcommand{\bX}{\mathbf{X}}
\newcommand{\by}{\mathbf{y}}

\newcommand{\bz}{\mathbf{z}}

\newcommand{\rmd}{\mathrm{d}}

\newcommand{\rmF}{\mathrm{F}}

\newcommand{\bbE}{\mathbb{E}}

\newcommand{\bbN}{\mathbb{N}}

\newcommand{\bbP}{\mathbb{P}}

\newcommand{\bbR}{\mathbb{R}}

\DeclareMathAlphabet{\mathbsf}{OT1}{cmss}{bx}{n}
\DeclareMathAlphabet{\mathssf}{OT1}{cmss}{m}{sl}

\DeclareSymbolFont{bsfletters}{OT1}{cmss}{bx}{n}  
\DeclareSymbolFont{ssfletters}{OT1}{cmss}{m}{n}
\DeclareMathSymbol{\bsfGamma}{0}{bsfletters}{'000}
\DeclareMathSymbol{\ssfGamma}{0}{ssfletters}{'000}
\DeclareMathSymbol{\bsfDelta}{0}{bsfletters}{'001}
\DeclareMathSymbol{\ssfDelta}{0}{ssfletters}{'001}
\DeclareMathSymbol{\bsfTheta}{0}{bsfletters}{'002}
\DeclareMathSymbol{\ssfTheta}{0}{ssfletters}{'002}
\DeclareMathSymbol{\bsfLambda}{0}{bsfletters}{'003}
\DeclareMathSymbol{\ssfLambda}{0}{ssfletters}{'003}
\DeclareMathSymbol{\bsfXi}{0}{bsfletters}{'004}
\DeclareMathSymbol{\ssfXi}{0}{ssfletters}{'004}
\DeclareMathSymbol{\bsfPi}{0}{bsfletters}{'005}
\DeclareMathSymbol{\ssfPi}{0}{ssfletters}{'005}
\DeclareMathSymbol{\bsfSigma}{0}{bsfletters}{'006}
\DeclareMathSymbol{\ssfSigma}{0}{ssfletters}{'006}
\DeclareMathSymbol{\bsfUpsilon}{0}{bsfletters}{'007}
\DeclareMathSymbol{\ssfUpsilon}{0}{ssfletters}{'007}
\DeclareMathSymbol{\bsfPhi}{0}{bsfletters}{'010}
\DeclareMathSymbol{\ssfPhi}{0}{ssfletters}{'010}
\DeclareMathSymbol{\bsfPsi}{0}{bsfletters}{'011}
\DeclareMathSymbol{\ssfPsi}{0}{ssfletters}{'011}
\DeclareMathSymbol{\bsfOmega}{0}{bsfletters}{'012}
\DeclareMathSymbol{\ssfOmega}{0}{ssfletters}{'012}

\newcommand{\balpha}{\bm{\alpha}}

\newcommand{\bSigma	}{\bm{\Sigma}}

\newtheorem{theorem}{Theorem} 
\newtheorem{lemma}{Lemma}

\newtheorem{assumption}{Assumption}

\newtheorem{definition}{Definition} 
\newtheorem{example}{Example} 
\newtheorem{remark}{Remark}

\newcommand{\qednew}{\nobreak \ifvmode \relax \else
      \ifdim\lastskip<1.5em \hskip-\lastskip
      \hskip1.5em plus0em minus0.5em \fi \nobreak
      \vrule height0.75em width0.5em depth0.25em\fi}



%% file: gpca_arXiv_v10.bbl
\begin{thebibliography}{10}
\providecommand{\url}[1]{#1}
\csname url@samestyle\endcsname
\providecommand{\newblock}{\relax}
\providecommand{\bibinfo}[2]{#2}
\providecommand{\BIBentrySTDinterwordspacing}{\spaceskip=0pt\relax}
\providecommand{\BIBentryALTinterwordstretchfactor}{4}
\providecommand{\BIBentryALTinterwordspacing}{\spaceskip=\fontdimen2\font plus
\BIBentryALTinterwordstretchfactor\fontdimen3\font minus
  \fontdimen4\font\relax}
\providecommand{\BIBforeignlanguage}[2]{{%
\expandafter\ifx\csname l@#1\endcsname\relax
\typeout{** WARNING: IEEEtran.bst: No hyphenation pattern has been}%
\typeout{** loaded for the language `#1'. Using the pattern for}%
\typeout{** the default language instead.}%
\else
\language=\csname l@#1\endcsname
\fi
#2}}
\providecommand{\BIBdecl}{\relax}
\BIBdecl

\bibitem{jolliffe1986}
I.~T. Jolliffe, \emph{Principal component analysis}.\hskip 1em plus 0.5em minus
  0.4em\relax Springer-Verlag, 1986.

\bibitem{hancock1996face}
P.~J. Hancock, A.~M. Burton, and V.~Bruce, ``Face processing: {H}uman
  perception and principal components analysis,'' \emph{Mem. Cogn.}, vol.~24,
  no.~1, pp. 26--40, 1996.

\bibitem{alter2000singular}
O.~Alter, P.~O. Brown, and D.~Botstein, ``Singular value decomposition for
  genome-wide expression data processing and modeling,'' \emph{PNAS}, vol.~97,
  no.~18, pp. 10\,101--10\,106, 2000.

\bibitem{ding2004k}
C.~Ding and X.~He, ``{$K$}-means clustering via principal component analysis,''
  in \emph{Int. Conf. Mach. Learn. (ICML)}, 2004, p.~29.

\bibitem{liu2019informativeness}
Z.~Liu and V.~Tan, ``The informativeness of $k$-means for learning mixture
  models,'' \emph{IEEE Trans. Inf. Theory}, vol.~65, no.~11, pp. 7460--7479,
  2019.

\bibitem{anderson1962introduction}
T.~W. Anderson, ``An introduction to multivariate statistical analysis,'' Wiley
  New York, Tech. Rep., 1962.

\bibitem{nadler2008finite}
B.~Nadler, ``Finite sample approximation results for principal component
  analysis: {A} matrix perturbation approach,'' \emph{Ann. Stat.}, vol.~36,
  no.~6, pp. 2791--2817, 2008.

\bibitem{johnstone2009consistency}
I.~M. Johnstone and A.~Y. Lu, ``On consistency and sparsity for principal
  components analysis in high dimensions,'' \emph{J. Am. Stat. Assoc.}, vol.
  104, no. 486, pp. 682--693, 2009.

\bibitem{jung2009pca}
S.~Jung and J.~S. Marron, ``{PCA} consistency in high dimension, low sample
  size context,'' \emph{Ann. Stat.}, vol.~37, no.~6B, pp. 4104--4130, 2009.

\bibitem{birnbaum2013minimax}
A.~Birnbaum, I.~M. Johnstone, B.~Nadler, and D.~Paul, ``Minimax bounds for
  sparse {PCA} with noisy high-dimensional data,'' \emph{Ann. Stat.}, vol.~41,
  no.~3, p. 1055, 2013.

\bibitem{zou2006sparse}
H.~Zou, T.~Hastie, and R.~Tibshirani, ``Sparse principal component analysis,''
  \emph{J. Comput. Graph. Stat.}, vol.~15, no.~2, pp. 265--286, 2006.

\bibitem{Fos19}
D.~Foster, \emph{Generative Deep Learning : Teaching Machines to Paint, Write,
  Compose and Play}.\hskip 1em plus 0.5em minus 0.4em\relax O'Reilly Media,
  Inc, USA, 2019.

\bibitem{bora2017compressed}
A.~Bora, A.~Jalal, E.~Price, and A.~G. Dimakis, ``Compressed sensing using
  generative models,'' in \emph{Int. Conf. Mach. Learn. (ICML)}, 2017, pp.
  537--546.

\bibitem{van2018compressed}
D.~Van~Veen, A.~Jalal, M.~Soltanolkotabi \emph{et~al.}, ``Compressed sensing
  with deep image prior and learned regularization,''
  \emph{https://arxiv.org/1806.06438}, 2018.

\bibitem{dhar2018modeling}
M.~Dhar, A.~Grover, and S.~Ermon, ``Modeling sparse deviations for compressed
  sensing using generative models,'' in \emph{Int. Conf. Mach. Learn.
  (ICML)}.\hskip 1em plus 0.5em minus 0.4em\relax PMLR, 2018, pp. 1214--1223.

\bibitem{heckel2019deep}
R.~Heckel and P.~Hand, ``Deep decoder: Concise image representations from
  untrained non-convolutional networks,'' in \emph{Int. Conf. Learn. Repr.
  (ICLR)}, 2019.

\bibitem{liu2020information}
Z.~Liu and J.~Scarlett, ``Information-theoretic lower bounds for compressive
  sensing with generative models,'' \emph{IEEE J. Sel. Areas Inf. Theory},
  vol.~1, no.~1, pp. 292--303, 2020.

\bibitem{kamath2020power}
A.~Kamath, S.~Karmalkar, and E.~Price, ``On the power of compressed sensing
  with generative models,'' in \emph{Int. Conf. Mach. Learn. (ICML)}, 2020, pp.
  5101--5109.

\bibitem{jalal2020robust}
A.~Jalal, L.~Liu, A.~G. Dimakis, and C.~Caramanis, ``Robust compressed sensing
  using generative models,'' \emph{Conf. Neur. Inf. Proc. Sys. (NeurIPS)},
  2020.

\bibitem{liu2020generalized}
Z.~Liu and J.~Scarlett, ``The generalized {L}asso with nonlinear observations
  and generative priors,'' in \emph{Conf. Neur. Inf. Proc. Sys. (NeurIPS)},
  vol.~33, 2020.

\bibitem{ongie2020deep}
G.~Ongie, A.~Jalal, C.~A. Metzler \emph{et~al.}, ``Deep learning techniques for
  inverse problems in imaging,'' \emph{IEEE J. Sel. Areas Inf. Theory}, vol.~1,
  no.~1, pp. 39--56, 2020.

\bibitem{whang2020compressed}
J.~Whang, Q.~Lei, and A.~G. Dimakis, ``Compressed sensing with invertible
  generative models and dependent noise,'' \emph{https://arxiv.org/2003.08089},
  2020.

\bibitem{jalal2021instance}
A.~Jalal, S.~Karmalkar, A.~Dimakis, and E.~Price, ``Instance-optimal compressed
  sensing via posterior sampling,'' in \emph{Int. Conf. Mach. Learn. (ICML)},
  2021.

\bibitem{nguyen2021provable}
T.~V. Nguyen, G.~Jagatap, and C.~Hegde, ``Provable compressed sensing with
  generative priors via {L}angevin dynamics,''
  \emph{https://arxiv.org/2102.12643}, 2021.

\bibitem{liu2020sample}
Z.~Liu, S.~Gomes, A.~Tiwari, and J.~Scarlett, ``Sample complexity bounds for
  $1$-bit compressive sensing and binary stable embeddings with generative
  priors,'' in \emph{Int. Conf. Mach. Learn. (ICML)}, 2020.

\bibitem{liu2021robust}
Z.~Liu, S.~Ghosh, J.~Han, and J.~Scarlett, ``Robust $1$-bit compressive sensing
  with partial {G}aussian circulant matrices and generative priors,''
  \emph{https://arxiv.org/2108.03570}, 2021.

\bibitem{vu2012minimax}
V.~Vu and J.~Lei, ``Minimax rates of estimation for sparse {PCA} in high
  dimensions,'' in \emph{Int. Conf. Artif. Intell. Stat. (AISTATS)}.\hskip 1em
  plus 0.5em minus 0.4em\relax PMLR, 2012, pp. 1278--1286.

\bibitem{d2007direct}
A.~d'Aspremont, L.~El~Ghaoui, M.~I. Jordan, and G.~R. Lanckriet, ``A direct
  formulation for sparse {PCA} using semidefinite programming,'' \emph{SIAM
  Rev.}, vol.~49, no.~3, pp. 434--448, 2007.

\bibitem{vu2013fantope}
V.~Q. Vu, J.~Cho, J.~Lei, and K.~Rohe, ``Fantope projection and selection: A
  near-optimal convex relaxation of sparse {PCA},'' \emph{Conf. Neur. Inf.
  Proc. Sys. (NeurIPS)}, vol.~26, 2013.

\bibitem{chang2016convex}
X.~Chang, F.~Nie, Y.~Yang \emph{et~al.}, ``Convex sparse {PCA} for unsupervised
  feature learning,'' \emph{ACM T. Knowl. Discov. D.}, vol.~11, no.~1, pp.
  1--16, 2016.

\bibitem{moghaddam2006spectral}
B.~Moghaddam, Y.~Weiss, and S.~Avidan, ``Spectral bounds for sparse {PCA}:
  Exact and greedy algorithms,'' \emph{Conf. Neur. Inf. Proc. Sys. (NeurIPS)},
  vol.~18, p. 915, 2006.

\bibitem{d2008optimal}
A.~d'Aspremont, F.~Bach, and L.~El~Ghaoui, ``Optimal solutions for sparse
  principal component analysis.'' \emph{J. Mach. Learn. Res.}, vol.~9, no.~7,
  2008.

\bibitem{jolliffe2003modified}
I.~T. Jolliffe, N.~T. Trendafilov, and M.~Uddin, ``A modified principal
  component technique based on the {L}asso,'' \emph{J. Comput. Graph. Stat.},
  vol.~12, no.~3, pp. 531--547, 2003.

\bibitem{shen2008sparse}
H.~Shen and J.~Z. Huang, ``Sparse principal component analysis via regularized
  low rank matrix approximation,'' \emph{J. Multivar. Anal.}, vol.~99, no.~6,
  pp. 1015--1034, 2008.

\bibitem{journee2010generalized}
M.~Journ{\'e}e, Y.~Nesterov, P.~Richt{\'a}rik, and R.~Sepulchre, ``Generalized
  power method for sparse principal component analysis.'' \emph{J. Mach. Learn.
  Res.}, vol.~11, no.~2, 2010.

\bibitem{hein2010inverse}
M.~Hein and T.~B{\"u}hler, ``An inverse power method for nonlinear
  eigenproblems with applications in $1$-spectral clustering and sparse
  {PCA},'' in \emph{Conf. Neur. Inf. Proc. Sys. (NeurIPS)}, 2010, pp. 847--855.

\bibitem{kuleshov2013fast}
V.~Kuleshov, ``Fast algorithms for sparse principal component analysis based on
  {R}ayleigh quotient iteration,'' in \emph{Int. Conf. Mach. Learn.
  (ICML)}.\hskip 1em plus 0.5em minus 0.4em\relax PMLR, 2013, pp. 1418--1425.

\bibitem{yuan2013truncated}
X.-T. Yuan and T.~Zhang, ``Truncated power method for sparse eigenvalue
  problems,'' \emph{J. Mach. Learn. Res.}, vol.~14, no.~4, 2013.

\bibitem{asteris2011sparse}
M.~Asteris, D.~S. Papailiopoulos, and G.~N. Karystinos, ``Sparse principal
  component of a rank-deficient matrix,'' in \emph{Int. Symp. Inf. Theory
  (ISIT)}.\hskip 1em plus 0.5em minus 0.4em\relax IEEE, 2011, pp. 673--677.

\bibitem{papailiopoulos2013sparse}
D.~Papailiopoulos, A.~Dimakis, and S.~Korokythakis, ``Sparse {PCA} through
  low-rank approximations,'' in \emph{Int. Conf. Mach. Learn. (ICML)}.\hskip
  1em plus 0.5em minus 0.4em\relax PMLR, 2013, pp. 747--755.

\bibitem{mackey2008deflation}
L.~W. Mackey, ``Deflation methods for sparse {PCA},'' in \emph{Conf. Neur. Inf.
  Proc. Sys. (NeurIPS)}, vol.~21, 2008, pp. 1017--1024.

\bibitem{ma2013sparse}
Z.~Ma, ``Sparse principal component analysis and iterative thresholding,''
  \emph{Ann. Stat.}, vol.~41, no.~2, pp. 772--801, 2013.

\bibitem{cai2013sparse}
T.~T. Cai, Z.~Ma, and Y.~Wu, ``Sparse {PCA}: Optimal rates and adaptive
  estimation,'' \emph{Ann. Stat.}, vol.~41, no.~6, pp. 3074--3110, 2013.

\bibitem{wang2014tighten}
Z.~Wang, H.~Lu, and H.~Liu, ``Tighten after relax: Minimax-optimal sparse {PCA}
  in polynomial time,'' \emph{Conf. Neur. Inf. Proc. Sys. (NeurIPS)}, vol.
  2014, p. 3383, 2014.

\bibitem{hand2018phase}
P.~Hand, O.~Leong, and V.~Voroninski, ``Phase retrieval under a generative
  prior,'' in \emph{Conf. Neur. Inf. Proc. Sys. (NeurIPS)}, 2018, pp.
  9154--9164.

\bibitem{hyder2019alternating}
R.~Hyder, V.~Shah, C.~Hegde, and M.~S. Asif, ``Alternating phase projected
  gradient descent with generative priors for solving compressive phase
  retrieval,'' in \emph{Int. Conf. Acoust. Sp. Sig. Proc. (ICASSP)}.\hskip 1em
  plus 0.5em minus 0.4em\relax IEEE, 2019, pp. 7705--7709.

\bibitem{jagatap2019algorithmic}
G.~Jagatap and C.~Hegde, ``Algorithmic guarantees for inverse imaging with
  untrained network priors,'' in \emph{Conf. Neur. Inf. Proc. Sys. (NeurIPS)},
  vol.~32, 2019.

\bibitem{wei2019statistical}
X.~Wei, Z.~Yang, and Z.~Wang, ``On the statistical rate of nonlinear recovery
  in generative models with heavy-tailed data,'' in \emph{Int. Conf. Mach.
  Learn. (ICML)}, 2019, pp. 6697--6706.

\bibitem{shamshad2020compressed}
F.~Shamshad and A.~Ahmed, ``Compressed sensing-based robust phase retrieval via
  deep generative priors,'' \emph{IEEE Sens. J.}, vol.~21, no.~2, pp.
  2286--2298, 2020.

\bibitem{aubin2020exact}
B.~Aubin, B.~Loureiro, A.~Baker \emph{et~al.}, ``Exact asymptotics for phase
  retrieval and compressed sensing with random generative priors,'' in
  \emph{Math. Sci. Mach. Learn. (MSML)}.\hskip 1em plus 0.5em minus 0.4em\relax
  PMLR, 2020, pp. 55--73.

\bibitem{liu2021towards}
Z.~Liu, S.~Ghosh, and J.~Scarlett, ``Towards sample-optimal compressive phase
  retrieval with sparse and generative priors,''
  \emph{https://arxiv.org/2106.15358}, 2021.

\bibitem{shah2018solving}
V.~Shah and C.~Hegde, ``Solving linear inverse problems using {GAN} priors: An
  algorithm with provable guarantees,'' in \emph{Int. Conf. Acoust. Sp. Sig.
  Proc. (ICASSP)}.\hskip 1em plus 0.5em minus 0.4em\relax IEEE, 2018, pp.
  4609--4613.

\bibitem{netrapalli2015phase}
P.~Netrapalli, P.~Jain, and S.~Sanghavi, ``Phase retrieval using alternating
  minimization,'' \emph{IEEE Trans. Sig. Proc.}, vol.~63, no.~18, pp.
  4814--4826, 2015.

\bibitem{candes2015phase}
E.~J. Cand{\`e}s, X.~Li, and M.~Soltanolkotabi, ``Phase retrieval via
  {W}irtinger flow: Theory and algorithms,'' \emph{IEEE Trans. Inf. Theory},
  vol.~61, no.~4, pp. 1985--2007, 2015.

\bibitem{aubin2019spiked}
B.~Aubin, B.~Loureiro, A.~Maillard \emph{et~al.}, ``The spiked matrix model
  with generative priors,'' \emph{Conf. Neur. Inf. Proc. Sys. (NeurIPS)},
  vol.~32, pp. 8366--8377, 2019.

\bibitem{cocola2020nonasymptotic}
J.~Cocola, P.~Hand, and V.~Voroninski, ``Nonasymptotic guarantees for spiked
  matrix recovery with generative priors,'' in \emph{Conf. Neur. Inf. Proc.
  Sys. (NeurIPS)}, vol.~33, 2020.

\bibitem{deshpande2014cone}
Y.~Deshpande, A.~Montanari, and E.~Richard, ``Cone-constrained principal
  component analysis,'' in \emph{Conf. Neur. Inf. Proc. Sys. (NeurIPS)}, 2014,
  pp. 2717--2725.

\bibitem{peng2020solving}
P.~Peng, S.~Jalali, and X.~Yuan, ``Solving inverse problems via
  auto-encoders,'' \emph{IEEE J. Sel. Areas Inf. Theory}, vol.~1, no.~1, pp.
  312--323, 2020.

\bibitem{raj2019gan}
A.~Raj, Y.~Li, and Y.~Bresler, ``{GAN}-based projector for faster recovery with
  convergence guarantees in linear inverse problems,'' in \emph{IEEE/CVF Int.
  Conf. Comp. Vision (ICCV)}, 2019.

\bibitem{boumal2016nonconvex}
N.~Boumal, ``Nonconvex phase synchronization,'' \emph{SIAM J. Optim.}, vol.~26,
  no.~4, pp. 2355--2377, 2016.

\bibitem{liu2017estimation}
H.~Liu, M.-C. Yue, and A.~Man-Cho~So, ``On the estimation performance and
  convergence rate of the generalized power method for phase synchronization,''
  \emph{SIAM J. Optim.}, vol.~27, no.~4, pp. 2426--2446, 2017.

\bibitem{deshpande2015finding}
Y.~Deshpande and A.~Montanari, ``Finding hidden cliques of size $\sqrt{N/e}$ in
  nearly linear time,'' \emph{Found. Comput. Math.}, vol.~15, no.~4, pp.
  1069--1128, 2015.

\bibitem{chen2018projected}
Y.~Chen and E.~J. Cand{\`e}s, ``The projected power method: An efficient
  algorithm for joint alignment from pairwise differences,'' \emph{Comm. Pure
  Appl. Math.}, vol.~71, no.~8, pp. 1648--1714, 2018.

\bibitem{yi2020non}
Y.~Yi and M.~Neykov, ``Non-sparse {PCA} in high dimensions via cone projected
  power iteration,'' \emph{https://arxiv.org/2005.07587}, 2020.

\bibitem{perry2018optimality}
A.~Perry, A.~S. Wein, A.~S. Bandeira, and A.~Moitra, ``Optimality and
  sub-optimality of {PCA} {I}: Spiked random matrix models,'' \emph{Ann.
  Stat.}, vol.~46, no.~5, pp. 2416--2451, 2018.

\bibitem{vershynin2010introduction}
R.~Vershynin, ``Introduction to the non-asymptotic analysis of random
  matrices,'' \emph{https://arxiv.org/abs/1011.3027}, 2010.

\bibitem{deshpande2016sparse}
Y.~Deshpande and A.~Montanari, ``Sparse {PCA} via covariance thresholding,''
  \emph{J. Mach. Learn. Res.}, vol.~17, no.~1, pp. 4913--4953, 2016.

\bibitem{chung2019weak}
H.~W. Chung and J.~O. Lee, ``Weak detection of signal in the spiked {W}igner
  model,'' in \emph{Int. Conf. Mach. Learn. (ICML)}.\hskip 1em plus 0.5em minus
  0.4em\relax PMLR, 2019, pp. 1233--1241.

\bibitem{zhang2017nonconvex}
H.~Zhang, Y.~Zhou, Y.~Liang, and Y.~Chi, ``A nonconvex approach for phase
  retrieval: Reshaped {W}irtinger flow and incremental algorithms,'' \emph{J.
  Mach. Learn. Res.}, vol.~18, 2017.

\bibitem{zass2007nonnegative}
R.~Zass and A.~Shashua, ``Nonnegative sparse {PCA},'' in \emph{Conf. Neur. Inf.
  Proc. Sys. (NeurIPS)}.\hskip 1em plus 0.5em minus 0.4em\relax Citeseer, 2007,
  pp. 1561--1568.

\bibitem{sigg2008expectation}
C.~D. Sigg and J.~M. Buhmann, ``Expectation-maximization for sparse and
  non-negative {PCA},'' in \emph{Int. Conf. Mach. Learn. (ICML)}, 2008, pp.
  960--967.

\bibitem{asteris2014nonnegative}
M.~Asteris, D.~Papailiopoulos, and A.~Dimakis, ``Nonnegative sparse {PCA} with
  provable guarantees,'' in \emph{Int. Conf. Mach. Learn. (ICML)}.\hskip 1em
  plus 0.5em minus 0.4em\relax PMLR, 2014, pp. 1728--1736.

\bibitem{wang2016statistical}
T.~Wang, Q.~Berthet, and R.~J. Samworth, ``Statistical and computational
  trade-offs in estimation of sparse principal components,'' \emph{Ann. Stat.},
  vol.~44, no.~5, pp. 1896--1930, 2016.

\bibitem{lecun1998gradient}
Y.~LeCun, L.~Bottou, Y.~Bengio, and P.~Haffner, ``Gradient-based learning
  applied to document recognition,'' \emph{Proc. IEEE}, vol.~86, no.~11, pp.
  2278--2324, 1998.

\bibitem{xiao2017fashion}
H.~Xiao, K.~Rasul, and R.~Vollgraf, ``Fashion-{MNIST}: {A} novel image dataset
  for benchmarking machine learning algorithms,''
  \emph{https://arxiv.org/1708.07747}, 2017.

\bibitem{liu2015deep}
Z.~Liu, P.~Luo, X.~Wang, and X.~Tang, ``Deep learning face attributes in the
  wild,'' in \emph{Int. Conf. Comput. Vis. (ICCV)}, 2015, pp. 3730--3738.

\bibitem{vershynin2018high}
R.~Vershynin, \emph{High-dimensional probability: An introduction with
  applications in data science}.\hskip 1em plus 0.5em minus 0.4em\relax
  Cambridge university press, 2018, vol.~47.

\bibitem{yu1997assouad}
B.~Yu, ``Assouad, {F}ano, and {L}e {C}am,'' in \emph{Festschrift for Lucien Le
  Cam}.\hskip 1em plus 0.5em minus 0.4em\relax Springer, 1997, pp. 423--435.

\bibitem{Tsy09}
A.~B. Tsybakov, \emph{Introduction to nonparametric estimation}.\hskip 1em plus
  0.5em minus 0.4em\relax Springer, New York, 2009.

\end{thebibliography}
